\documentclass[journal]{IEEEtran}
\usepackage{cite}
\usepackage[cmex10]{amsmath}

\usepackage{comment}

\usepackage{graphicx}
\usepackage{amsmath, amssymb}
\usepackage{amsfonts}
\usepackage{ascmac}
\usepackage{subcaption}
\usepackage{bm}
\usepackage{url}
\usepackage{authblk}
\usepackage{amsthm}
\usepackage{empheq}

\usepackage[
  colorlinks=true,
  linkcolor=blue,
  urlcolor=blue,
  citecolor=blue
]{hyperref}

\newcommand{\ali}[1]{\begin{align} #1 \end{align}}
\newcommand{\nn}{\nonumber}

\newcommand{\mat}[1]{\begin{bmatrix} #1 \end{bmatrix}}

\newcommand{\pa}{\mathop{}\!\partial}

\newtheorem{theorem}{Theorem}
\newtheorem{lemma}{Lemma}

\newtheorem{proposition}{Proposition}
\usepackage{mdframed}
\newmdtheoremenv{theorembox}{Theorem}

\begin{document}
\title{\Huge Closed-Form Steepest Descent Direction toward Flat Minima: Reducing Upper Bounds on the Loss Hessian Eigenspectrum in Neural Networks}

\author{
Yuto Omae$^1$, Kazuki Sakai$^2$, Yohei Kakimoto$^1$, 
Makoto Sasaki$^1$, Yusuke Sakai$^3$, Hirotaka Takahashi$^3$
\thanks{1: Nihon University (Japan)}
\thanks{2: National Institute of Technology, Nagaoka College (Japan)}
\thanks{3: Tokyo City University (Japan)}
}

\markboth{Preprint}%
{Y. Omae \MakeLowercase{\textit{et al.}}}
\maketitle

\begin{abstract}
One influential theory for explaining the generalization ability of neural networks (NNs) is the flatness hypothesis, which suggests that the flatness of the loss landscape is related to generalization performance.
In general, the flatness of the loss landscape is quantified by the eigenvalues of the Hessian matrix of the Taylor-expanded loss function. 
Several training algorithms have been proposed to reduce the eigenvalues of the loss Hessian. However, most existing studies focus on the design of training algorithms and do not clarify how the training data distribution and the internal parameters of NNs contribute to directions leading to flatter minima.
One direct way to achieve this is to analytically characterize the directions in which the eigenvalues decrease; however, deriving such directions is generally difficult. 
On the other hand, recent studies have reported the Wolkowicz-Styan (WS) upper bound, a theorem that analytically describes an upper bound on the largest eigenvalue of the cross-entropy (CE) loss Hessian in a three-layer hierarchical NN.
However, that study was limited to deriving the upper bound and did not derive its gradient. 
Therefore, in this study, we analytically derive the gradient of the WS upper bound and use its closed-form expression to characterize directions leading to flatter minima.
To examine whether this direction facilitates convergence to flatter minima, we propose a regularization method that updates the network parameters along the steepest descent direction of the WS upper bound. 
Results from several numerical experiments show that this regularization narrows the range of the Hessian eigenvalue spectrum, avoids both sharp minima and saddle points, and promotes convergence to flatter minima.
Therefore, we name this method Hessian Spectral Range (HSR) Regularization. 
Comparisons with existing regularization methods show that HSR Regularization outperforms Hessian Regularization and achieves solutions that are as flat as those obtained by Sharpness-Aware Minimization (SAM).
The applicability of the proposed method is limited because it only works for the combination of the CE loss and a three-layer hierarchical NN. 
However, to the best of the authors' knowledge, no previous study has reported a closed-form gradient that promotes convergence to flatter minima without relying on numerical approximations. 
Therefore, this study contributes to the theoretical development of NNs.
\end{abstract}


\section{Introduction}
Neural networks (NNs) are widely utilized in a broad range of tasks and have achieved state-of-the-art performance in numerous domains~\cite{chaiDeep2021}\cite{mehrishReview2023}\cite{arkhangelskayaDeep2023}.
On the other hand, the theoretical understanding of their generalization capabilities is still under development.
As a prominent theory concerning the generalization of NNs, the flatness hypothesis is widely recognized~\cite{hochreiterFlat1997}.
According to the flatness hypothesis, if the loss function exhibits a sharp landscape in the vicinity of a solution obtained through training, the generalization error tends to be large; conversely, if the loss function has a flat shape, the generalization error is expected to be small~\cite{liuHessian2023}\cite{aroraUnderstanding2022}.
To quantify the sharpness of the loss, the eigenspectrum of the Hessian matrix has been widely employed as a representative metric.
This is because when the loss function is Taylor-expanded around a critical point, its local curvature is characterized by the Hessian matrix appearing in the quadratic term~\cite{liuHessian2023}\cite{lyuUnderstanding2023}.
Previous studies have proposed several optimization algorithms aimed at reaching flat minima by mitigating sharpness.
Representative examples of such approaches include Hessian regularization~\cite{liuHessian2023} and sharpness-aware minimization (SAM)~\cite{foretSharpnessAware2021}, both of which have demonstrated improvements in test performance across various tasks.
However, these methods focus primarily on algorithmic design, and the factors that dictate the direction toward flat minima are not yet fully understood.
To theoretically comprehend the structural mechanisms that form flat minima, it is essential to analytically describe the direction leading toward them.

Therefore, the objective of this study is to derive a closed-form solution for the direction toward flat minima.
In this paper, we define this direction as the steepest descent direction of sharpness.
In other words, the research objective is to derive the parameter gradient of the maximum eigenvalue. However, because it is generally difficult to analytically express the maximum eigenvalue, obtaining its gradient analytically is inherently challenging.
Nevertheless, a recent study~\cite{omaeWolkowiczStyan2026} derived a closed-form solution for the upper bound of the maximum eigenvalue under the condition of the cross-entropy (CE) loss in three-layer hierarchical NNs.
This theorem describes the upper bound of the maximum eigenvalue based on the traces of the Hessian and the squared Hessian, which is referred to as the ``Wolkowicz-Styan (WS) upper bound.''
Although this function is expected to possess an analytical derivative, the aforementioned study did not extend to its derivation.
Therefore, this study attempts to describe the steepest descent direction of sharpness as a closed-form function by deriving the parameter gradient of the WS upper bound.
This enables us to investigate how the training data distribution and the internal network parameters influence the direction toward flat minima.
To the best of the authors' knowledge, no prior study has reported the derivation of the direction toward flat minima in a closed form.
This work provides a new foundation for the analytical understanding of flat minima.

In this paper, we name the optimization approach that moves in the steepest descent direction of the WS upper bound ``Hessian Spectral Range (HSR) Regularization.''
HSR regularization has the effects of decreasing the maximum eigenvalue and increasing the minimum eigenvalue of the Hessian matrix.
That is, this method has the effect of narrowing the range of the eigenspectrum, which is expected to prevent the model from falling into both sharp minima and saddle points.
In this work, we verify whether HSR regularization can achieve flatness comparable to existing methods, such as Hessian regularization and SAM.
At present, HSR regularization can only be applied to three-layer hierarchical NNs, which poses a significant limitation in practical applications.
However, because it can achieve flatness comparable to existing methods without relying on numerical approximations, the direction toward flat minima derived in this study can be regarded as a closed-form function of considerable value.
The above constitutes the academic contribution of this study to the NN domain.

\section{Related works}
\subsection{Flat minima}
In 1997, Hochreiter et al.~\cite{hochreiterFlat1997} argued that as a requirement for NNs with high generalization performance, it is important not only that the error is low but also that the errors in its vicinity are low, meaning that the loss landscape is flat.
Regarding this hypothesis, some studies have questioned its validity due to issues such as invariance under reparameterization~\cite{dinhSharp2017}.
On the other hand, numerous practical applications have reported that reaching flat minima improves generalization performance~\cite{chenWhen2022}\cite{huangSharpness2025}\cite{liVisualizing2018}.
Furthermore, it has been reported that several empirical techniques considered effective for improving the generalization performance of NNs may be related to the reduction of sharpness.
For instance, the verification of sharpness reduction effects achieved by batch normalization~\cite{keskarLargeBatch2017}\cite{ghorbaniInvestigation2019}\cite{lyuUnderstanding2023}, stochastic gradient descent~\cite{wuImplicit2023}\cite{weiHow2019}, and skip connections~\cite{liVisualizing2018} constitutes an intriguing area of research.
As can be seen from these examples, the pursuit of flat minima is considered an important perspective for constructing deep learning models with high generalization capabilities.

\subsection{Numerical Approaches to Eigenspectrum Analysis}
This sharpness is evaluated by the quadratic term when the loss function is Taylor-expanded around a critical point\cite{liuHessian2023}.
The reason for this is that the eigenvalues of the Hessian matrix represent the curvature\cite{aroraUnderstanding2022}.
While the eigenspectrum consists of multiple eigenvalues, the maximum eigenvalue in particular is utilized as a crucial indicator representing the curvature of the loss landscape\cite{lyuUnderstanding2023}.
The eigenspectrum can be obtained by solving the characteristic equation of the Hessian.
Letting $D$ denote the parameter size of the NN, the Hessian becomes a matrix of size $D \times D$.
When $D\ge 5$, the eigenvalues of the Hessian cannot be obtained analytically because a characteristic equation of degree five or higher does not possess a closed-form solution.
However, in modern deep learning, which is currently the mainstream, the parameter size $D$ of networks is exceedingly large.
For instance, in the implementation using PyTorch~\cite{Torchvision}, $D \sim 1.38\times10^8$ for VGG16~\cite{simonyanVery2014} and $D \sim 1.17\times10^7$ for ResNet18~\cite{kaimingDeep2016}.
To determine the eigenspectrum of such a massive Hessian matrix, numerical approximations are employed.
As prominent approaches for this purpose, Hutchinson's method\cite{hutchinsonStochastic1989} and the Lanczos method~\cite{lanczosIteration1950} are well known.
Hutchinson's method is a technique for estimating the Hessian trace, whereas the Lanczos method is used to estimate the eigenspectrum; by leveraging these methods, it is possible to numerically evaluate the sharpness.
In fact, several studies have proposed methods to compute the eigenspectrum of deep learning models using the Lanczos method~\cite{ghorbaniInvestigation2019}\cite{yaoPyHessian2020}, as well as techniques to calculate the Hessian trace via Hutchinson's method~\cite{liuHessian2023}\cite{dongHAWQV22020}.

\subsection{Analytical Approaches to Eigenspectrum Analysis}
On the other hand, the numerical approximation approach suffers from an inherent limitation in that it cannot clarify what causes the loss landscape to become sharp.
To achieve this, it is necessary to express the eigenvalues analytically.
As a pioneering study on the analytical computation of the Hessian, Bishop~\cite{bishopExact1992} proposed an extended backpropagation algorithm that precisely calculates all components of the Hessian matrix for feedforward networks with arbitrary topologies.
However, this study provides a foundational method for computing the Hessian and does not delve into the analytical representation of the eigenvalues themselves.
To express eigenvalues analytically, it is necessary to impose certain structural constraints on the network, making it difficult to achieve this for general networks with arbitrary layer structures.
For this reason, analytical studies have been conducted using simplified network architectures.
For instance, Singh et al.~\cite{singhAnalytic2021} derived a closed-form expression for the rank of the Hessian in networks utilizing linear activations.
Additionally, Wu et al.~\cite{wuDissecting2022} proposed a separation conjecture that approximates the layer-wise Hessian using Kronecker products, analytically explaining common structures such as the low-rank properties of the Hessian and the overlapping of eigenspaces among different models.
Furthermore, Singh et al.~\cite{singhCracking2025} obtained closed-form representations of the eigenvalues of the loss Hessian in linear networks with identity or ReLU activations.
While obtaining closed-form representations of eigenvalues is inherently difficult for non-linear activations, it is possible to derive upper bounds for them.
For example, Omae et al.~\cite{omaeWolkowiczStyan2026} derived an upper bound for the maximum eigenvalue under the CE loss in a three-layer hierarchical NN.
A key advantage of this approach is that it allows the activation function of the hidden layer to be chosen arbitrarily.

\subsection{Optimization Algorithms for Sharpness Reduction}
Based on various reports indicating that a sharper critical point of the loss function leads to a larger generalization error, several methods aimed at mitigating sharpness have been devised.
For instance, Yue et al.~\cite{yueSALR2024} proposed the Sharpness-Aware Learning Rate Scheduler, which dynamically adjusts the learning rate according to the sharpness of the loss landscape to promote convergence to flat minima.
Liu et al.~\cite{liuHessian2023} proposed Hessian regularization, which suppresses curvature by regularizing the trace of the Hessian matrix.
Sankar et al.~\cite{sankarDeeper2021} proposed Layerwise Hessian Trace Regularization to reduce the Hessian trace of each layer.
Luo et al.~\cite{luoExplicit2025} proposed an eigenvalue regularization method that explicitly suppresses the maximum eigenvalue of the Hessian.
In addition, Sharpness-Aware Minimization (SAM) is widely recognized as a representative approach for suppressing large eigenvalues of the Hessian~\cite{foretSharpnessAware2021}.
SAM is an optimization method that avoids sharp local minima by finding a perturbation that maximizes the loss in the vicinity of the parameters and subsequently minimizing this worst-case loss.
As theoretical advancements of SAM, modified versions have been proposed, such as a variant that overcomes the vulnerability to weight parameter scaling~\cite{kwonASAM2021} and another that functions effectively even with imbalanced data~\cite{zhouImbSAM2023a}.
Research dedicated to theoretically elucidating the effectiveness of SAM is also progressing~\cite{andriushchenkoUnderstanding2022}.
Beyond theoretical studies, it has been reported that both SAM and Hessian regularization have the effect of enhancing test performance in practical application tasks~\cite{chenWhen2022}\cite{huangSharpness2025}\cite{zhangNoise2024}.

\subsection{Originality of This Study}
Most of the aforementioned methods adopt an approach that numerically searches for the direction that reduces the sharpness of the loss function and subsequently updates the parameters in that direction. In other words, they can be regarded as techniques for estimating the direction toward flat minima through numerical computation. These methods are applicable to large-scale neural networks and are highly effective for the practical purpose of improving generalization performance.
On the other hand, it is not straightforward to understand how the numerically obtained direction relates to the distribution of training data or the internal parameter structure of the network. Therefore, an inherent limitation exists in terms of clarifying the underlying mechanisms behind the direction toward flat minima.

In this study, we adopt an analytical approach to address this issue. Specifically, we analytically derive the direction toward flat minima based on the eigenvalue statistics of the Hessian matrix. While analytical methods are generally applicable only to small-scale models or under specific assumptions, and thus may be inferior to numerical methods in terms of applicability, they offer a distinct advantage. Specifically, they allow the relationship between sharpness reduction and network structure to be described mathematically, which is expected to facilitate a theoretical understanding of neural networks.
Therefore, the numerical and analytical approaches are not in competition with each other, but rather serve different purposes. The former is useful as a practical optimization method for large-scale models, whereas the latter is valuable as a theoretical framework for understanding the behavior of neural networks. The originality of this study lies in analytically expressing the direction toward flat minima, thereby enabling its theoretical interpretation.

\begin{figure}[t] 
    \centering
    \includegraphics[scale=0.3]{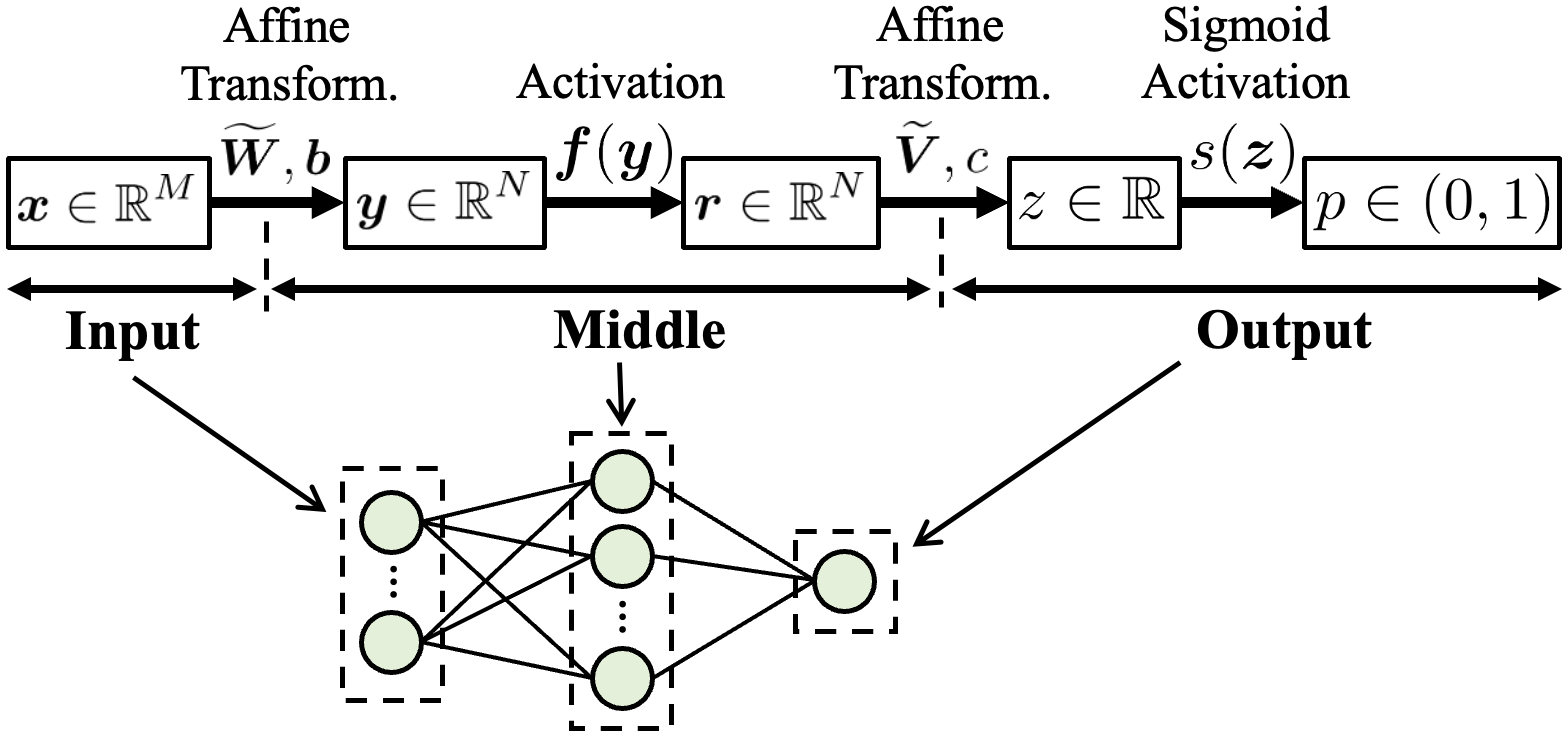}
    \caption{Three-layer hierarchical NN analyzed in this study.}
    \label{figg_nnimg}
\end{figure}

\section{Overview of the Previous Model}
In this study, we adopt the denominator layout for the arrangement of derivatives within matrices and vectors.
For further details, please refer to Appendix~\ref{secc_layout}.
Additionally, this work is a continuation of Omae et al.~\cite{omaeWolkowiczStyan2026}, and the definitions of all variables and functions are identical to those in the previous study.
Therefore, we provide a brief overview here.

\subsection{Model Assumptions and Loss Function}
The target model considered in this study is a three-layer hierarchical NN for binary classification.
Given an input $\bm{x}$, the estimated probability $p$ is computed as
\ali{
\bm{y} = \bm{W} \bm{h}(\bm{x}), \ 
\bm{r} = \bm{f}(\bm{y}), \
z = \bm{V} \bm{h}(\bm{r}), \
p &= s(z), \nn
}
where $s$ denotes the sigmoid function, which outputs the estimated probability.
Specifically, an input is classified into class 1 if $p \ge 0.5$, and into class 0 if $p < 0.5$.
Here, $\bm{x} \in \mathbb{R}^M$, $\bm{y} \in \mathbb{R}^N$, and $\bm{r} \in \mathbb{R}^N$, where $M$ and $N$ denote the dimensionalities of the input and the hidden layer, respectively.
Note that $\bm{f}(\bm{y})$ represents the activation function of the hidden layer; for further details, refer to Eq.~\eqref{eqq_bm_f}.
We also note that various expressions for the activation functions used in this paper are summarized in Appendix~\ref{secc_act}.

The function $\bm{h}$ prepends a 1 to the 0-th dimension of the input vector.
Specifically, $\bm{h}(\bm{x}) = [1 \quad \bm{x}^\top]^\top  \in \mathbb{R}^{M+1}$ and $\bm{h}(\bm{r}) = [1 \quad \bm{r}^\top]^\top  \in \mathbb{R}^{N+1}$.
The matrices $\bm{W}$ and $\bm{V}$ represent the affine mapping parameters, defined as $\bm{W} = [\bm{b} \quad \widetilde{\bm{W}}] \in \mathbb{R}^{N \times (M+1)}$ and $\bm{V} = [c \quad \widetilde{\bm{V}}] \in \mathbb{R}^{1 \times (N+1)}$.
Here, $\widetilde{\bm{W}} \in \mathbb{R}^{N \times M}$ and $\widetilde{\bm{V}} \in \mathbb{R}^{1 \times N}$ are the weight parameters, whereas $\bm{b} \in \mathbb{R}^N$ and $c \in \mathbb{R}$ correspond to the bias parameters.
The network architecture using this notation is illustrated in Fig.~\ref{figg_nnimg}.
As can be seen, the object of analysis in this study is a fundamental NN with a single hidden layer.

Letting $\bm{w}$ and $\bm{v}$ denote the vectors obtained by vertically concatenating the column vectors of $\bm{W}$ and $\bm{V}$, respectively, the vector $\bm{\theta}$ formed by vertically stacking these vectors constitutes the full parameter set of the NN. That is,
\ali{
\bm{\theta} = \mat{\bm{w} \\ \bm{v}} \in \mathbb{R}^{D}, \ 
\bm{w} \in \mathbb{R}^{(M+1)N}, \ 
\bm{v} \in \mathbb{R}^{N+1}. \nn
}
Note that in this paper, the set of real matrices with $D$ rows and 1 column, denoted as $\mathbb{R}^{D \times 1}$, is abbreviated as $\mathbb{R}^{D}$.
In addition, the total number of dimensions is given by $D = MN+2N+1$.
For the specific arrangement of the elements in $\bm{w}$ and $\bm{v}$, please refer to Eq.~(3) in~\cite{omaeWolkowiczStyan2026}.

Since this model is an NN designed for binary classification, the loss function is given by the binary cross-entropy.
Specifically, letting $p_i$ denote the estimated probability that the $i$-th data sample belongs to class 1, and $q_i \in \{0, 1\}$ denote its ground-truth label, the loss function is expressed as
 \ali{
L(\bm{\theta}) = - \sum_{i=1}^{I} \big( q_i \log p_i(\bm{\theta})  + (1-q_i) \log(1- p_i(\bm{\theta})) \big), \nn
}
where $I$ represents the total number of training data samples.

\subsection{Wolkowicz-Styan Upper Bound}

While the previous study~\cite{omaeWolkowiczStyan2026} denoted the Hessian of the loss $L$ with respect to $\bm{\theta}$ as $\bm{H}_L(\bm{\theta}, \bm{\theta})$, this paper denotes it as $\bm{H}_L(\bm{\theta})$ to save space.
Furthermore, let $\lambda_1 \geq \lambda_2 \geq \cdots \geq \lambda_{D}$ represent the eigenvalues of $\bm{H}_L(\bm{\theta})$.
That is, $\lambda_1$ and $\lambda_{D}$ correspond to the maximum and minimum eigenvalues, respectively.
By applying Eq.~(2.3) of Wolkowicz and Styan in \cite{wolkowiczBounds1980}, the upper bound of the maximum eigenvalue $\lambda_1$ of $\bm{H}_L(\bm{\theta})$ can be expressed as
\ali{
&\lambda_1 \leq \lambda_\mathrm{sup}(\bm{\theta}) = \mu(\bm{\theta}) + \sqrt{D-1} \sigma(\bm{\theta}), \label{eq_main_theorem} \\
&\mu(\bm{\theta}) = \frac{1}{D}\mathrm{tr} ( \bm{H}_L(\bm{\theta}) ), \sigma(\bm{\theta})^2 = \frac{1}{D}\mathrm{tr} ( \bm{H}_L(\bm{\theta})^2 ) 
- \mu(\bm{\theta})^2. \label{eq_main_theorem_musig}
}
Here, $\mu(\bm{\theta})$ and $\sigma(\bm{\theta})^2$ represent the mean and variance of the eigenspectrum, respectively.
The term $\lambda_\mathrm{sup}(\bm{\theta})$ denotes the WS upper bound, which is a function comprising three arguments: $D$, $\mathrm{tr} ( \bm{H}_L(\bm{\theta}))$, and $\mathrm{tr} ( \bm{H}_L(\bm{\theta})^2)$.
Omae et al.~\cite{omaeWolkowiczStyan2026} obtained the WS upper bound of the NN as a closed-form function by deriving $\mathrm{tr} ( \bm{H}_L(\bm{\theta}) )$ and $\mathrm{tr} ( \bm{H}_L(\bm{\theta})^2 )$ in closed form under the CE loss in a three-layer hierarchical NN.
According to their work~\cite{omaeWolkowiczStyan2026}, $\mathrm{tr} ( \bm{H}_L(\bm{\theta}))$ is given by
\ali{
&\mathrm{tr}(\bm{H}_L(\bm{\theta})) = \sum_{i=1}^I s^\prime(z_i)(1+\bm{f}(\bm{y}_i)^\top\bm{f}(\bm{y}_i)) \nn \\  
&+\sum_{i=1}^I (1+\bm{x}_i^\top \bm{x}_i) \Big( s^\prime(z_i) \| \bm{F}^{\prime}(\bm{y}_i) \widetilde{\bm{V}}^\top\|^2 + \delta_i \widetilde{\bm{V}} \bm{f}^{\prime\prime}(\bm{y}_i) \Big). \label{eqq_tr_h1_origin}
}
The terms $s^\prime(z_i)$ and $\delta_i$ correspond to $s^\prime(z)$ and $\delta$ for the $i$-th data sample, respectively, which are given by Eqs. (52) and (53) in~\cite{omaeWolkowiczStyan2026} as
\ali{
s^\prime(z) = p(1-p), \ \delta = p - q. \nn
}
Similarly, $\mathrm{tr} ( \bm{H}_L(\bm{\theta})^2 )$ is given by
\ali{
\mathrm{tr} ( \bm{H}_L(\bm{\theta})^2 ) 
&= \sum_{i=1}^I \sum_{j=1}^I (1+ \bm{x}_i^\top\bm{x}_j)^2 \phi_{ij} \nn \\
&+ 2\sum_{i=1}^I \sum_{j=1}^I (1+ \bm{x}_i^\top\bm{x}_j) \psi_{ij} \nn \\
&+\sum_{i=1}^I \sum_{j=1}^I (1+ \bm{f}(\bm{y}_i)^\top\bm{f}(\bm{y}_j))^2 \omega_{ij}. \label{eqq_tr_h2_origin}
}
This expression was derived by Omae et al.~\cite{omaeWolkowiczStyan2026}. 
Here, the component $\phi_{ij}$ is defined as
\ali{
\phi_{ij} &= \bm{o}_{i}^\top \bm{\Phi}_{ij} \bm{o}_{j} \in \mathbb{R}, \ \bm{\Phi}_{ij} \in \mathbb{R}^{2 \times 2},  \label{eq_phi_ij} \\
(\bm{\Phi}_{ij})_{11} &= 
(\widetilde{\bm{V}} \bm{F}^\prime(\bm{y}_j) \bm{F}^\prime(\bm{y}_i)\widetilde{\bm{V}}^\top)^2, \label{eqq_Phi_11} \\
(\bm{\Phi}_{ij})_{12} &= \widetilde{\bm{V}}  \bm{F}^\prime(\bm{y}_i) \mathrm{diag}(\widetilde{\bm{V}}^\top) \bm{F}^{\prime\prime}(\bm{y}_j) \bm{F}^\prime(\bm{y}_i)\widetilde{\bm{V}}^\top, \label{eqq_Phi_12} \\
(\bm{\Phi}_{ij})_{21} &=  \widetilde{\bm{V}} \bm{F}^\prime(\bm{y}_j) \mathrm{diag}(\widetilde{\bm{V}}^\top)\bm{F}^{\prime\prime}(\bm{y}_i) \bm{F}^\prime(\bm{y}_j)\widetilde{\bm{V}}^\top, \label{eqq_Phi_21} \\
 (\bm{\Phi}_{ij})_{22} &= \widetilde{\bm{V}} \bm{F}^{\prime\prime}(\bm{y}_i) \bm{F}^{\prime\prime}(\bm{y}_j) \widetilde{\bm{V}}^\top. \label{eqq_Phi_22}
}
For the detailed definitions of $\bm{F}^\prime(\bm{y})$ and $\bm{F}^{\prime\prime}(\bm{y})$, refer to Eq.~\eqref{eqq_bm_matF_k}.
Furthermore, $\bm{o}$ is defined as
\ali{
\bm{o} =
\mat{
s^\prime(z) &
\delta
}^\top, \label{eqq_o_def}
}
where $\bm{o}_i$ and $\bm{o}_j$ correspond to $\bm{o}$ for the $i$-th and $j$-th data samples, respectively.
$\psi_{ij}$ is defined as
\ali{
\psi_{ij} &= \bm{o}_{i}^\top \bm{\Psi}_{ij} \bm{o}_{j} \in \mathbb{R}, \ 
\bm{\Psi}_{ij} \in \mathbb{R}^{2 \times 2}, \label{eq_psi_ij}\\
(\bm{\Psi}_{ij})_{11} &= (1+\bm{f}(\bm{y}_i)^\top \bm{f}(\bm{y}_j)) \widetilde{\bm{V}} \bm{F}^\prime(\bm{y}_j) \bm{F}^\prime(\bm{y}_i) \widetilde{\bm{V}}^\top, \label{eqq_Psi_11} \\
(\bm{\Psi}_{ij})_{12} &= \bm{f}(\bm{y}_i)^\top \bm{F}^\prime(\bm{y}_j) \bm{F}^\prime(\bm{y}_i) \widetilde{\bm{V}}^\top, \label{eqq_Psi_12} \\
(\bm{\Psi}_{ij})_{21} &= \widetilde{\bm{V}} \bm{F}^\prime(\bm{y}_j) \bm{F}^\prime(\bm{y}_i) \bm{f}(\bm{y}_j), \label{eqq_Psi_21} \\
 (\bm{\Psi}_{ij})_{22} &= \bm{f}^\prime(\bm{y}_i)^\top \bm{f}^\prime(\bm{y}_j). \label{eqq_Psi_22}
}
For $\bm{f}^\prime(\bm{y})$, refer to Eq.~\eqref{eqq_bm_vecF_k}.
$\omega_{ij}$ is defined as
\ali{
\omega_{ij} &= \bm{o}_{i}^\top \bm{\Omega}_{ij} \bm{o}_{j} \in (0, 1/16], \ 
\bm{\Omega}_{ij} = \mat{1 & 0\\0 & 0}. \label{eq_ome_ij}
}

\section{Steepest Descent Direction of the WS Upper Bound}
\subsection{Main Theorem}
In our previous study~\cite{omaeWolkowiczStyan2026}, we derived the upper bound of the maximum eigenvalue, $\lambda_\mathrm{sup}(\bm{\theta})$, as a closed-form function, as shown in Eq.~\eqref{eq_main_theorem}.
As a continuation of that work, this study derives the steepest descent direction of the WS upper bound as a closed-form function.
Since this direction corresponds to the negative gradient of the WS upper bound, it is given as follows.
\begin{theorem}[]\label{}
\ali{
&- \frac{\pa \lambda_\mathrm{sup}(\bm{\theta})}{\pa \bm{\theta}} = 
- \frac{\pa \mu(\bm{\theta})}{\pa \bm{\theta}} -
\sqrt{D-1} \frac{\pa \sigma(\bm{\theta})}{\pa \bm{\theta}}, \label{eqq_main_theo_direction}\\
&\frac{\pa \mu(\bm{\theta})}{\pa \bm{\theta}} = \frac{1}{D} \frac{\pa \mathrm{tr} (\bm{H}_L(\bm{\theta}))}{\pa \bm{\theta}} \label{eqq_main_theo_mu}, \\
&\frac{\pa \sigma(\bm{\theta})}{\pa \bm{\theta}} 
=  
\frac{1}{2 \sigma(\bm{\theta}) D} \frac{\pa \mathrm{tr} (\bm{H}_L(\bm{\theta})^2)}{\pa \bm{\theta}} 
- \frac{\mu(\bm{\theta})}{\sigma(\bm{\theta})} \frac{\pa \mu(\bm{\theta})}{\pa \bm{\theta}} \label{eqq_main_theo_sigma}. 
}
\end{theorem}
\begin{proof} See Appendix \ref{secc_h2_tr_grad}. \end{proof}
By clarifying the gradients of the Hessian trace and the squared Hessian trace, the closed-form expression for the steepest descent direction of the WS upper bound can be obtained.
Furthermore, since $\pa / \pa \bm{\theta}$ comprises $\pa / \pa \bm{w}$ and $\pa / \pa \bm{v}$ as its components, it is sufficient to clarify $\pa \mathrm{tr}(\bm{H}_L(\bm{\theta})) / \pa \bm{w}$, $\pa \mathrm{tr}(\bm{H}_L(\bm{\theta})) / \pa \bm{v}$, $\pa \mathrm{tr}(\bm{H}_L(\bm{\theta})^2) / \pa \bm{w}$, and $\pa \mathrm{tr}(\bm{H}_L(\bm{\theta})^2) / \pa \bm{v}$.

Observing the main theorem reveals that the steepest descent direction of $\lambda_\mathrm{sup}(\bm{\theta})$ is composed of the steepest descent directions of the mean $\mu(\bm{\theta})$ and the standard deviation $\sigma(\bm{\theta})$ of the eigenspectrum.
That is, moving in the direction of $-\pa \lambda_\mathrm{sup}(\bm{\theta}) / \pa \bm{\theta}$ represents an operation that reduces the WS upper bound $\lambda_\mathrm{sup}(\bm{\theta})$ by shifting down the entire eigenspectrum and suppressing its variance.
The operation that contributes most to reducing both the mean and the standard deviation is bringing the maximum eigenvalue $\lambda_1$ closer to the mean.
Therefore, moving in the steepest descent direction of the WS upper bound can be expected to have the effect of decreasing the maximum eigenvalue $\lambda_1$.
Consequently, $-\pa \lambda_\mathrm{sup}(\bm{\theta}) / \pa \bm{\theta}$ can be regarded as a direction toward flat minima.

\begin{figure}[t] 
    \centering
    \includegraphics[scale=0.7]{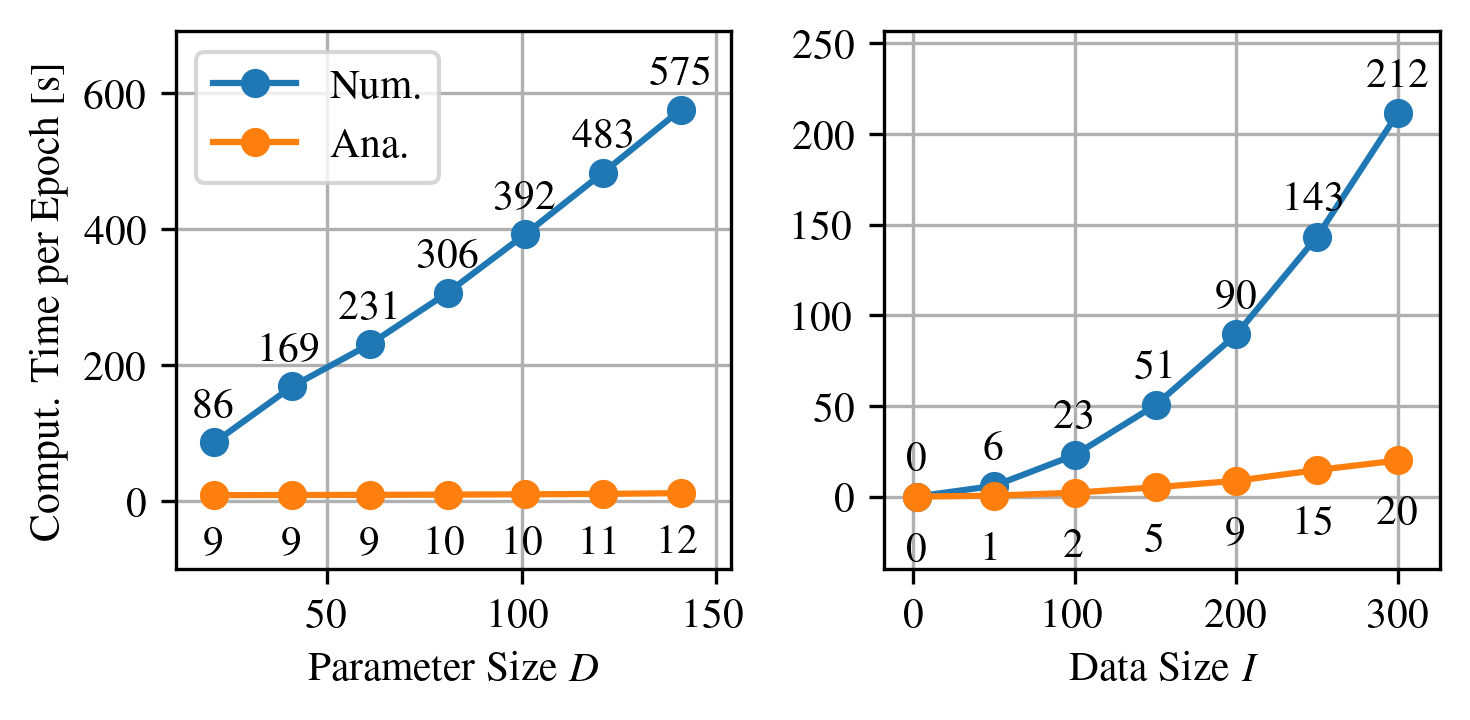}
    \caption{
    Computation time of $\pa \lambda_\mathrm{sup}(\bm{\theta}) / \pa \bm{\theta}$.
    ``Num.'' denotes the numerical solution, and ``Ana.'' denotes the analytical solution.
    Left: Variation with respect to the dimensionality $D$, with the training data size fixed at $I=200$.
    Right: Variation with respect to the data size $I$, with the dimensionality fixed at $D=21$.
    All computations were executed via serial processing on an Apple M2 CPU (clock frequency: 3.49 GHz).
    The three-point finite difference method was used for numerical differentiation.
    The activation function of the hidden layer is sigmoid.
    }
    \label{figg_compt_time}
\end{figure}

Although it is possible to solve $\pa \lambda_\mathrm{sup}(\bm{\theta}) / \pa \bm{\theta}$ using numerical differentiation, this approach does not allow us to analyze how the direction toward flat minima is characterized.
Therefore, this study aims to clarify the analytical solution, which is established in Theorems \ref{theoo_h1_w}, \ref{theoo_h1_v}, \ref{theoo_h2_w}, and \ref{theoo_h2_v}.
From another perspective, the analytical solution offers a distinct advantage in terms of computation time.
Fig.~\ref{figg_compt_time} shows the computation times of the numerical and analytical solutions required to derive $\pa \lambda_\mathrm{sup}(\bm{\theta}) / \pa \bm{\theta}$.
Fig.~\ref{figg_compt_time} (left) represents the case where the parameter size $D$ of the NN is increased by increasing the dimensionality of the hidden layer $N$ while fixing the input dimensionality at $M=2$.
It can be observed that the computation time increases linearly with respect to $D$.
The WS upper bound is a function based on the traces of the Hessian and the squared Hessian, and the computation time of the trace is affected by the number of diagonal components $D$.
Therefore, it is considered that the computation time increases linearly with $D$.
Fig.~\ref{figg_compt_time} (right) illustrates the case where the training data size $I$ is increased.
Although the details are described later, $\pa \lambda_\mathrm{sup}(\bm{\theta}) / \pa \bm{\theta}$ includes a double sum with respect to $I$.
Thus, the computation time is expected to be proportional to $I^2$.
Although an increasing trend in computation time is observed with respect to both $D$ and $I$, the analytical solution completes the computation in less time than the numerical solution in either case.

\begin{figure}[t] 
    \centering
    \includegraphics[scale=0.63]{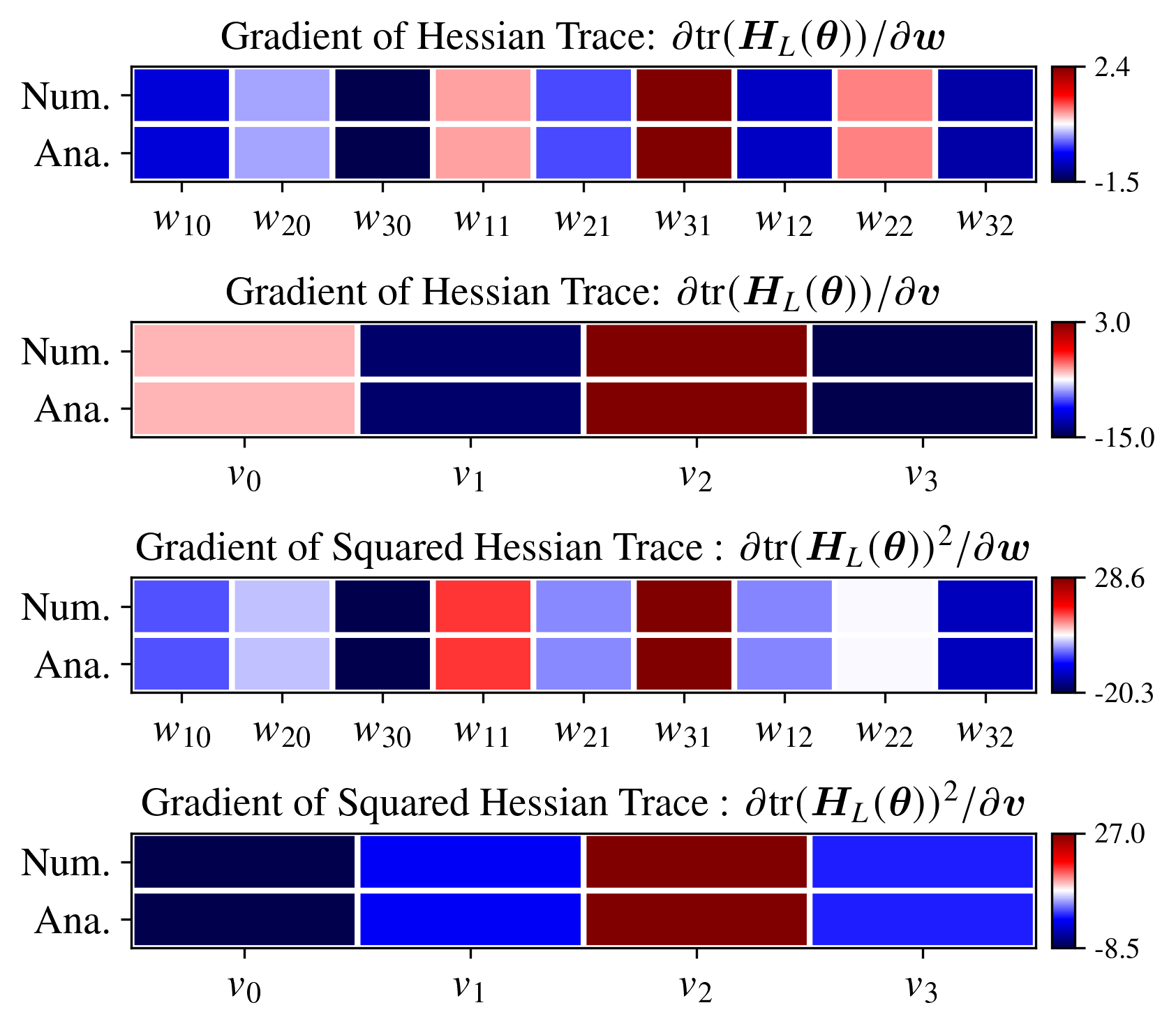}
    \caption{
    Comparison between numerical and analytical solutions. 
    ``Num.'' denotes the numerical solution, and ``Ana.'' denotes the analytical solution.
    The three-point finite difference method was used for numerical differentiation.
    The activation function of the hidden layer is sigmoid.
     }
    \label{figg_error}
\end{figure}

\subsection{Gradients of the Hessian Trace}
Here, we address the derivation of $\pa \mathrm{tr}(\bm{H}_L(\bm{\theta})) / \pa \bm{w}$ and $\pa \mathrm{tr}(\bm{H}_L(\bm{\theta})) / \pa \bm{v}$, which are the components of the steepest descent direction of the WS upper bound.
The gradient of $\mathrm{tr} (\bm{H}_L(\bm{\theta}))$ with respect to the affine parameters $\bm{w}$ from the input layer to the hidden layer is given as follows.
\begin{theorem}\label{theoo_h1_w}
\ali{
\frac{\pa \mathrm{tr}(\bm{H}_L(\bm{\theta}))}{\pa \bm{w}} 
&= \sum_{i=1}^I \bigg( \bm{h}(\bm{x}_i) \otimes \sum_{a \in \mathfrak{A}} \bm{\mathcal{W}}_{i}^a \bigg).\label{eqq_h1_w_div_theorem}
}
\end{theorem}
\begin{proof}
See Appendix \ref{secc_H1_grad_w}.
\end{proof}
Where
\ali{
a \in \mathfrak{A} := \{ \mathrm{I}, \mathrm{II}, \mathrm{III} \} \nn
}
and
\ali{
\bm{\mathcal{W}}_{i}^\mathrm{I} &= (1+\bm{x}_i^\top \bm{x}_i) \Big(s^{\prime\prime}(z_i) \bm{F}^{\prime}(\bm{y}_i) \widetilde{\bm{V}}^\top \widetilde{\bm{V}} \bm{F}^{\prime}(\bm{y}_i) \nn \\
&+ 2 s^\prime(z_i) \mathrm{diag}(\widetilde{\bm{V}}^\top ) \bm{F}^{\prime\prime}(\bm{y}_i) \Big)
\bm{F}^{\prime}(\bm{y}_i) \widetilde{\bm{V}}^\top, \label{eqq_W_i_I}\\
\bm{\mathcal{W}}_{i}^\mathrm{II} &= (1+\bm{x}_i^\top \bm{x}_i) \nn \\
&\times \Big(s^\prime(z_i) \bm{F}^\prime (\bm{y}_i) \widetilde{\bm{V}}^\top \widetilde{\bm{V}} \bm{f}^{\prime\prime}(\bm{y}_i) + \delta_i \bm{F}^{\prime\prime\prime}(\bm{y}_i) \widetilde{\bm{V}}^\top \Big), \label{eqq_W_i_II} \\
\bm{\mathcal{W}}_{i}^\mathrm{III} &= s^{\prime\prime}(z_i) (1+\bm{f}(\bm{y}_i)^\top \bm{f}(\bm{y}_i) )\bm{F}^{\prime}(\bm{y}_i) \widetilde{\bm{V}}^\top \nn \\
&+ 2 s^\prime(z_i) \bm{F}^\prime(\bm{y}_i)  \bm{f}(\bm{y}_i). \label{eqq_W_i_III}
}
For $\bm{f}^{\prime\prime}(\bm{y})$ and $\bm{F}^{\prime\prime\prime}(\bm{y})$, refer to Eqs.~\eqref{eqq_bm_vecF_k} and \eqref{eqq_bm_matF_k}.
The term $s^{\prime\prime}(z)$ represents the second derivative of the sigmoid function, which is given by
\ali{
s^{\prime\prime}(z) = (1-2p) (1-p) p. \nn
}
The dimensions of these variables are given by
\ali{
\frac{\pa \mathrm{tr}(\bm{H}_L(\bm{\theta}))}{\pa \bm{w}}  \in \mathbb{R}^{(M+1)N}, \ 
\bm{\mathcal{W}}_{i}^\mathrm{I}, \bm{\mathcal{W}}_{i}^\mathrm{II}, \bm{\mathcal{W}}_{i}^\mathrm{III} 
\in \mathbb{R}^{N}, \forall i \in \mathbb{N}_{\le I}.  \nn
}

The gradient of the Hessian trace with respect to the affine parameters $\bm{v}$ from the hidden layer to the output layer is given as follows.
\begin{theorem}\label{theoo_h1_v}
\ali{
\frac{\pa \mathrm{tr}(\bm{H}_L(\bm{\theta}))}{\pa \bm{v}}
&= \sum_{i=1}^I \sum_{a \in \mathfrak{A}} \bm{\mathcal{V}}_{i}^a.
\label{eqq_h1_v_div_theorem}
}
\end{theorem}
\begin{proof}
See Appendix \ref{secc_H1_grad_v}.
\end{proof}
Where
\ali{
\bm{\mathcal{V}}_{i}^\mathrm{I} &= 
(1+\bm{x}_i^\top \bm{x}_i) \Big(s^{\prime\prime}(z_i) \bm{h}(\bm{f}(\bm{y}_i)) \widetilde{\bm{V}} \bm{F}^{\prime}(\bm{y}_i) \nn \\ 
&+  2 s^\prime(z_i) \bm{F}^\prime_0(\bm{y}_i) \Big)  \bm{F}^{\prime}(\bm{y}_i) \widetilde{\bm{V}}^\top,   \label{eqq_V_i_I}\\
\bm{\mathcal{V}}_{i}^\mathrm{II} &= 
 (1+\bm{x}_i^\top \bm{x}_i) \nn \\ 
 & \times \big( s^\prime(z_i) \bm{h}(\bm{f}(\bm{y}_i)) \widetilde{\bm{V}} \bm{f}^{\prime\prime}(\bm{y}_i)
+ \delta_i \bm{f}^{\prime\prime}_0(\bm{y}_i) \big),  \label{eqq_V_i_II}\\
\bm{\mathcal{V}}_{i}^\mathrm{III} &= 
s^{\prime\prime}(z_i) (1+\bm{f}(\bm{y}_i)^\top \bm{f}(\bm{y}_i) ) \bm{h}(\bm{f}(\bm{y}_i)) \label{eqq_V_i_III}.
}
For $\bm{F}^\prime_0(\bm{y})$ and $\bm{f}^{\prime\prime}_0(\bm{y})$, refer to Eq.~\eqref{eqq_bm_matvecF_k_0}.
The dimensions of these variables are given by
\ali{
\frac{\pa \mathrm{tr}(\bm{H}_L(\bm{\theta}))}{\pa \bm{v}}, 
\bm{\mathcal{V}}_{i}^\mathrm{I}, 
\bm{\mathcal{V}}_{i}^\mathrm{II}, 
\bm{\mathcal{V}}_{i}^\mathrm{III} 
\in \mathbb{R}^{N+1}, \forall i \in \mathbb{N}_{\le I}. \nn
}
The top two panels of Fig.~\ref{figg_error} compare the numerical and analytical solutions of $\pa \mathrm{tr}(\bm{H}_L(\bm{\theta}))/\pa \bm{w}$ and $\pa \mathrm{tr}(\bm{H}_L(\bm{\theta}))/\pa \bm{v}$.
The difference norms between the numerical and analytical solutions are both on the order of $\sim 10^{-9}$, confirming that they are strongly consistent.
This trend remains identical for parameters generated with different random seeds.
These results demonstrate that the analytical solutions derived in this study are correct.

\subsection{Gradients of the Squared Hessian Trace}
Here, we address $\pa\mathrm{tr} (\bm{H}_L(\bm{\theta})^2) / \pa \bm{w}$ and $\pa\mathrm{tr} (\bm{H}_L(\bm{\theta})^2) / \pa \bm{v}$, which are the components of the steepest descent direction of the WS upper bound.
The gradient of the squared Hessian trace with respect to the affine parameters $\bm{w}$ from the input layer to the hidden layer is given as follows.
\begin{theorem}\label{theoo_h2_w}
\ali{
&\frac{\pa \mathrm{tr} (\bm{H}_L(\bm{\theta})^2)}{\pa \bm{w}} =
2\sum_{i=1}^I \sum_{j=1}^I 
\bigg( \bm{h}(\bm{x}_i) \otimes  \sum_{a \in \mathfrak{B}}
\bm{\mathcal{W}}_{ij}^a \bigg). \label{eqq_H2_w_div_fin}
}
\end{theorem}
\begin{proof}
See Appendix \ref{secc_H2_grad_w}.
\end{proof}
Where
\ali{
&\bm{\mathcal{W}}_{ij}^a = \bm{\mathcal{A}}_{ij}^a + \bm{\mathcal{B}}_{ij}^a, \ 
a \in \mathfrak{B} := \{\Phi, \Psi, \Omega\}, \nn \\
&\bm{\mathcal{A}}_{ij}^\Phi =
(1+ \bm{x}_i^\top\bm{x}_j)^2 \bm{G}^\Phi_{ij} (\bm{o}_i \otimes \bm{o}_j), \label{eqq_Aij_Phi} \\
&\bm{\mathcal{A}}_{ij}^\Psi =
2(1+ \bm{x}_i^\top\bm{x}_j) \bm{G}^\Psi_{ij} (\bm{o}_i \otimes \bm{o}_j), \label{eqq_Aij_Psi} \\
&\bm{\mathcal{A}}_{ij}^\Omega = (1+ \bm{f}(\bm{y}_i)^\top \bm{f}(\bm{y}_j))^2 s^{\prime\prime}(z_i) s^\prime(z_j) \bm{F}^{\prime}(\bm{y}_i) \widetilde{\bm{V}}^\top, \label{eqq_Aij_Ome}\\
&\bm{\mathcal{B}}_{ij}^\Phi =
(1+ \bm{x}_i^\top\bm{x}_j)^2 \bm{F}^\prime (\bm{y}_i) \widetilde{\bm{V}}^\top \bm{s}^{\prime\prime/\prime}(z_i)^\top \bm{\Phi}_{ij} \bm{o}_j, \label{eqq_Bij_Phi} \\
&\bm{\mathcal{B}}_{ij}^\Psi =
2(1+ \bm{x}_i^\top\bm{x}_j) \bm{F}^\prime (\bm{y}_i) \widetilde{\bm{V}}^\top \bm{s}^{\prime\prime/\prime}(z_i)^\top \bm{\Psi}_{ij} \bm{o}_j, \label{eqq_Bij_Psi} \\
&\bm{\mathcal{B}}_{ij}^\Omega = 2 s^\prime(z_i) s^\prime(z_j) (1+ \bm{f}(\bm{y}_i)^\top \bm{f}(\bm{y}_j)) \bm{F}^\prime(\bm{y}_i) \bm{f}(\bm{y}_j), \label{eqq_Bij_Ome}\\
&\bm{s}^{\prime\prime/\prime}(z) = \mat{s^{\prime\prime}(z) & s^{\prime}(z)}^\top. \label{eqq_spp_sp}
}
Here, $\bm{G}^\Phi_{ij}$ is given by
\ali{
\bm{G}^\Phi_{ij} &= \mat{
\bm{\mathfrak{a}}^\Phi_{ij} &
\bm{\mathfrak{b}}^\Phi_{ij} & 
\bm{\mathfrak{c}}^\Phi_{ij} & 
\bm{\mathfrak{d}}^\Phi_{ij}  
}\in \mathbb{R}^{N \times 4}, \label{eqq_G_phi_ij}\\
 \bm{\mathfrak{a}}^\Phi_{ij} &= 2 (\bm{\Phi}_{ij})_{11}^{1/2} \bm{F}^{\prime\prime}(\bm{y}_{i}) \bm{F}^\prime(\bm{y}_{j}) \widetilde{\bm{V}}^{\odot 2 \top}, \label{eqq_a_phi_ij}\\
 \bm{\mathfrak{b}}^\Phi_{ij} &= 2
\bm{F}^\prime(\bm{y}_{i})
\bm{F}^{\prime\prime}(\bm{y}_{i}) 
\bm{F}^{\prime\prime}(\bm{y}_{j}) 
\widetilde{\bm{V}}^{\odot 3 \top}, \label{eqq_b_phi_ij}\\
 \bm{\mathfrak{c}}^\Phi_{ij} &= \bm{F}^{\prime\prime\prime}(\bm{y}_{i})
\bm{F}^\prime(\bm{y}_{j})^{2}
\widetilde{\bm{V}}^{\odot 3 \top}, \label{eqq_c_phi_ij}\\
 \bm{\mathfrak{d}}^\Phi_{ij}&= \bm{F}^{\prime\prime\prime}(\bm{y}_{i})
\bm{F}^{\prime\prime}(\bm{y}_{j})
\widetilde{\bm{V}}^{\odot 2 \top}, \label{eqq_d_phi_ij}
}
and $\bm{G}^\Psi_{ij}$ is given by
\ali{
&\bm{G}^\Psi_{ij} = \mat{
\bm{\mathfrak{a}}^\Psi_{ij} &
\bm{\mathfrak{b}}^\Psi_{ij} & 
\bm{\mathfrak{c}}^\Psi_{ij} & 
\bm{\mathfrak{d}}^\Psi_{ij}  
}\in \mathbb{R}^{N \times 4} , \label{eqq_G_psi_ij}\\
&\bm{\mathfrak{a}}^\Psi_{ij} =
\Big(\bm{F}^\prime(\bm{y}_i) \bm{f}(\bm{y}_j) \widetilde{\bm{V}} \bm{F}^\prime(\bm{y}_i) \nn \\
&+
\mathrm{diag}(\widetilde{\bm{V}}^\top) \bm{F}^{\prime\prime}(\bm{y}_i)(1+\bm{f}(\bm{y}_i)^\top \bm{f}(\bm{y}_j)) \Big) \bm{F}^\prime(\bm{y}_j)\widetilde{\bm{V}}^\top, \label{eqq_a_psi_ij}\\
&\bm{\mathfrak{b}}^\Psi_{ij} = 
(\bm{F}^\prime(\bm{y}_i)^2 + \bm{F}^{\prime\prime}(\bm{y}_i) \bm{F}(\bm{y}_i) ) \bm{F}^{\prime}(\bm{y}_j) \widetilde{\bm{V}}^\top, \label{eqq_b_psi_ij}\\
&\bm{\mathfrak{c}}^\Psi_{ij} = 
\bm{F}^{\prime\prime}(\bm{y}_i) \bm{F}(\bm{y}_j) \bm{F}^\prime(\bm{y}_j) \widetilde{\bm{V}}^\top, \label{eqq_c_psi_ij}\\
&\bm{\mathfrak{d}}^\Psi_{ij} = \bm{F}^{\prime\prime}(\bm{y}_i) \bm{f}^\prime(\bm{y}_j). \label{eqq_d_psi_ij}
}
The dimensions of these variables are given by
\ali{
&\frac{\pa \mathrm{tr} (\bm{H}_L(\bm{\theta})^2)}{\pa \bm{w}} \in \mathbb{R}^{(M+1)N}, 
\bm{\mathcal{A}}_{ij}^a, 
\bm{\mathcal{B}}_{ij}^a, 
\in \mathbb{R}^{N}, 
\forall a \in \mathfrak{B}, \nn \\
&\bm{\mathfrak{a}}^{b}_{ij}, 
\bm{\mathfrak{b}}^{b}_{ij},  
\bm{\mathfrak{c}}^{b}_{ij}, 
\bm{\mathfrak{d}}^{b}_{ij}
\in \mathbb{R}^{N}, 
\forall b \in \mathfrak{B}\setminus\{\Omega\}, 
\forall i, j \in \mathbb{N}_{\le I}. \nn
}

The gradient of the squared Hessian trace with respect to the affine parameters $\bm{v}$ from the hidden layer to the output layer is given as follows.
\begin{theorem}[]\label{theoo_h2_v}
\ali{
&\frac{\pa \mathrm{tr} (\bm{H}_L(\bm{\theta})^2)}{\pa \bm{v}}
=
2\sum_{i=1}^I \sum_{j=1}^I \sum_{a \in \mathfrak{B}}
\bm{\mathcal{V}}_{ij}^a.
\label{eqq_H2_v_div_fin}
}
\end{theorem}
\begin{proof}
See Appendix \ref{secc_H2_grad_v}.
\end{proof}
Where
\ali{
\bm{\mathcal{V}}_{ij}^a &= \bm{\mathcal{C}}_{ij}^a + \bm{\mathcal{D}}_{ij}^a, \ a \in \mathfrak{B}, \nn \\
\bm{\mathcal{C}}_{ij}^\Phi &= 
(1+ \bm{x}_i^\top\bm{x}_j)^2 \bm{K}_{ij}^\Phi (\bm{o}_i \otimes \bm{o}_j), \label{eqq_Cij_Phi} \\
\bm{\mathcal{C}}_{ij}^\Psi &= 
2(1+ \bm{x}_i^\top\bm{x}_j) \bm{K}_{ij}^\Psi (\bm{o}_i \otimes \bm{o}_j), \label{eqq_Cij_Psi}  \\
\bm{\mathcal{C}}_{ij}^\Omega &= s^{\prime\prime}(z_i)s^\prime(z_j) (1+ \bm{f}(\bm{y}_i)^\top\bm{f}(\bm{y}_j))^2 \bm{h}(\bm{f}(\bm{y}_i)), \label{eqq_Cij_Omega} \\
\bm{\mathcal{D}}_{ij}^\Phi &= 
(1+ \bm{x}_i^\top\bm{x}_j)^2 \big( 
\bm{h}(\bm{f}(\bm{y}_i)) \bm{s}^{\prime\prime/\prime}(z_i)^\top \bm{\Phi}_{ij} \bm{o}_j \big),  \label{eqq_Dij_Phi} \\
\bm{\mathcal{D}}_{ij}^\Psi &= 
2(1+ \bm{x}_i^\top\bm{x}_j) \big( 
\bm{h}(\bm{f}(\bm{y}_i)) \bm{s}^{\prime\prime/\prime}(z_i)^\top \bm{\Psi}_{ij} \bm{o}_j \big), \label{eqq_Dij_Psi} \\
\bm{\mathcal{D}}_{ij}^\Omega &= \bm{0}_{N+1}. \label{eqq_Dij_Omega} 
}
Here, $\bm{0}_{N+1}$ denotes an $(N+1)$-dimensional zero vector.
$\bm{K}_{ij}^\Phi$ is given by
\ali{
\bm{K}_{ij}^\Phi &= \mat{\bm{\mathfrak{e}}^{\Phi}_{ij} & \bm{\mathfrak{f}}^{\Phi}_{ij} & \bm{0}_{N+1} &\bm{\mathfrak{g}}^{\Phi}_{ij}} \in \mathbb{R}^{(N+1) \times 4} \label{eqq_K_Phi_ij}, \\
\bm{\mathfrak{e}}^{\Phi}_{ij} &= 2 (\bm{\Phi}_{ij})_{11}^{1/2}  \bm{V}^\top \odot 
\bm{f}^\prime_0(\bm{y}_j) \odot \bm{f}^\prime_0(\bm{y}_i), \label{eqq_e_phi_ij} \\
 \bm{\mathfrak{f}}^{\Phi}_{ij} &= 3 \bm{V}^{\odot 2\top} \odot \bm{f}^\prime_0(\bm{y}_i)^{\odot 2} \odot \bm{f}^{\prime\prime}_0(\bm{y}_j), \label{eqq_f_phi_ij}\\
   \bm{\mathfrak{g}}^{\Phi}_{ij} &= \bm{V}^{\top} \odot \bm{f}^{\prime\prime}_0(\bm{y}_i) \odot \bm{f}^{\prime\prime}_0(\bm{y}_j), \label{eqq_g_phi_ij}
}
and $\bm{K}_{ij}^\Psi$ is given by
\ali{
\bm{K}_{ij}^\Psi &= 
\mat{\bm{\mathfrak{e}}^{\Psi}_{ij} & \bm{\mathfrak{f}}^{\Psi}_{ij} & \bm{0}_{N+1} &\bm{0}_{N+1}} \in \mathbb{R}^{(N+1) \times 4}, \label{eqq_K_Psi_ij}\\
  \bm{\mathfrak{e}}^{\Psi}_{ij} &= (1+\bm{f}(\bm{y}_i)^\top \bm{f}(\bm{y}_j)) 
\bm{F}^\prime_0(\bm{y}_i) \bm{F}^\prime(\bm{y}_j) 
\widetilde{\bm{V}}^\top, \label{eqq_e_psi_ij} \\
 \bm{\mathfrak{f}}^{\Psi}_{ij} &= \bm{F}^\prime_0(\bm{y}_i) \bm{F}^\prime(\bm{y}_j) \bm{f}(\bm{y}_i).\label{eqq_f_psi_ij}
}
The dimensions of these variables are given by
\ali{
&\frac{\pa \mathrm{tr}(\bm{H}_L(\bm{\theta})^2)}{\pa \bm{v}}, 
\bm{\mathcal{C}}_{ij}^a, 
\bm{\mathcal{D}}_{ij}^a, 
\bm{\mathfrak{e}}^{b}_{ij}, 
\bm{\mathfrak{f}}^{b}_{ij}, 
\bm{\mathfrak{g}}^{\Phi}_{ij} 
\in \mathbb{R}^{(N+1) \times 1}, \nn \\
&\forall a \in \mathfrak{B}, 
\forall b \in \mathfrak{B}\setminus\{\Omega\}, 
\forall i, j \in \mathbb{N}_{\le I}. \nn
}
The bottom two panels of Fig.~\ref{figg_error} compare the numerical and analytical solutions of $\pa \mathrm{tr}(\bm{H}_L(\bm{\theta})^2)/\pa \bm{w}$ and $\pa \mathrm{tr}(\bm{H}_L(\bm{\theta})^2)/\pa \bm{v}$.
The difference norms between the numerical and analytical solutions are both on the order of $\sim 10^{-9}$, confirming that they are strongly consistent.
This trend remains identical for parameters generated with different random seeds.
These results demonstrate that the analytical solutions derived in this study are correct.

\begin{figure}[t] 
    \centering
    \includegraphics[scale=0.73]{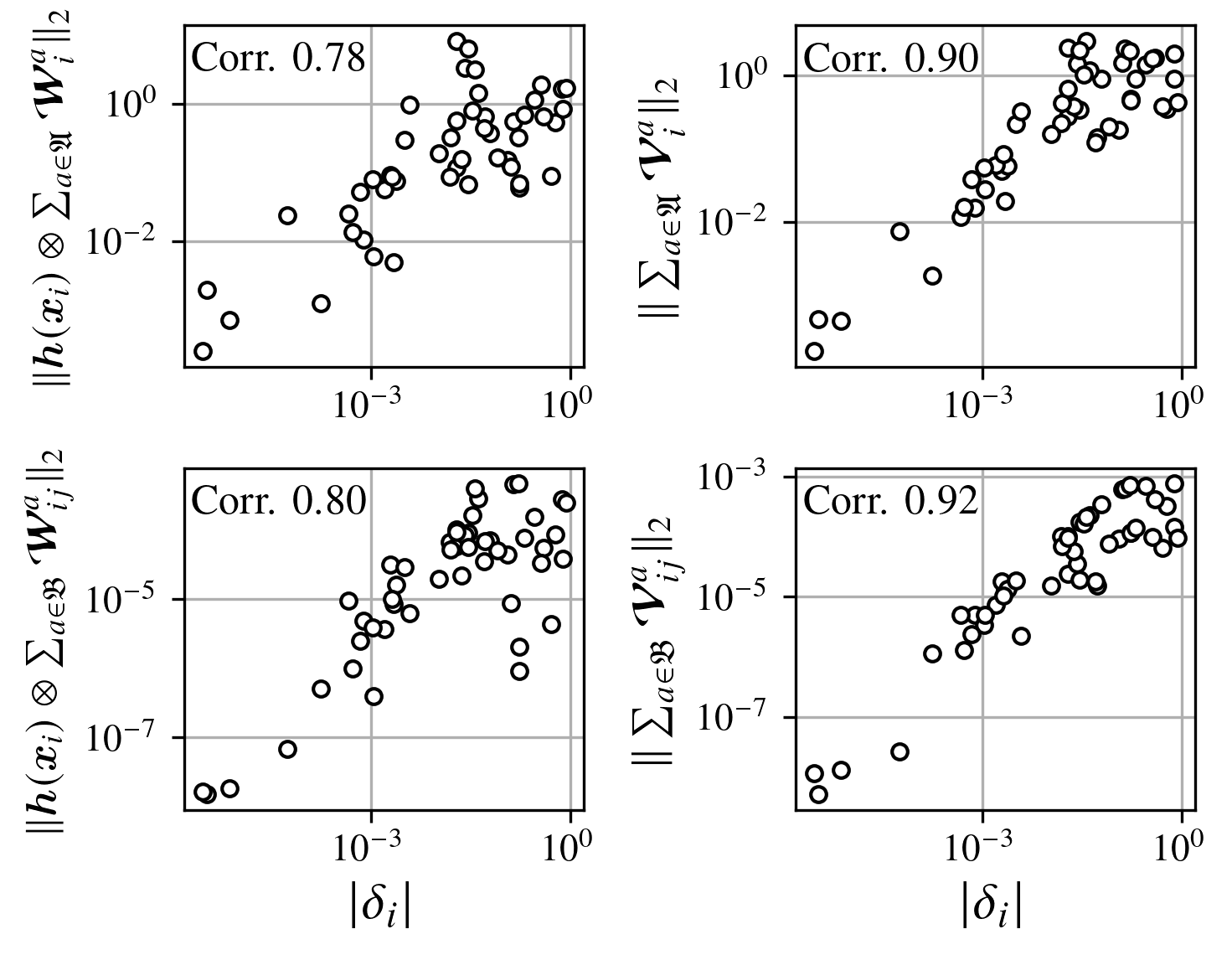}
    \caption{
    Relationship between $|\delta_i|$ and gradient norms. 
    A single NN trained in the experiment in Section~\ref{sec_expe} (described later) was used. 
    Since $I=50$, there are 50 data points. 
    The bottom two panels show the results with $j$ fixed.
    The activation function of the hidden layer is sigmoid.
    }
    \label{fig_delta_vs_norm}
\end{figure}

\section{Discussion}
\subsection{Influence of Individual Data Samples on the Gradient Norm}
All four gradients derived in this study possess the following property.
\begin{proposition}
\ali{
&p_i \rightarrow q_i \nn\\
&\Rightarrow \bigg( \bm{h}(\bm{x}_i) \otimes \sum_{a \in \mathfrak{A}} \bm{\mathcal{W}}_{i}^a  \rightarrow \bm{0}_{(M+1)N} \bigg) 
\land
\bigg( \sum_{a \in \mathfrak{A}} \bm{\mathcal{V}}_{i}^a \rightarrow \bm{0}_{N+1} \bigg) \nn \\
&\land
\bigg( \bm{h}(\bm{x}_i) \otimes  \sum_{a \in \mathfrak{B}}
\bm{\mathcal{W}}_{ij}^a
 \rightarrow \bm{0}_{(M+1)N} \bigg)  \land
\bigg( \sum_{a \in \mathfrak{B}}
\bm{\mathcal{V}}_{ij}^a
 \rightarrow \bm{0}_{N+1} \bigg). \label{eqq_data_i_conv}
}
\end{proposition}
\begin{proof}
See Appendix \ref{secc_data_i_conv}.
\end{proof}
From Eq.~(53) of the predecessor study~\cite{omaeWolkowiczStyan2026}, $|\delta|$ represents the degree of proximity between $p$ and $q$.
Therefore, this proposition asserts that when the NN's estimate $p$ for a given training data sample is sufficiently close to the ground-truth label $q$, resulting in a sufficiently small $|\delta|$, that specific data sample does not affect the gradient of the WS upper bound.
To verify whether this phenomenon actually occurs, we extracted a single NN trained in the experiments described later in Section~\ref{sec_expe} and plotted a scatter plot illustrating the relationship between $|\delta|$ and the gradient norm of the WS upper bound.
This result is shown in Fig.~\ref{fig_delta_vs_norm}.
As can be seen from the figure, the closer $|\delta|$ is to 0, that is, the closer $p$ is to $q$, the smaller the gradient norm becomes.
From this observation, it can be concluded that data samples that can be correctly estimated with high confidence do not influence the direction toward a flat minimum.
As a practical advantage, it is conceivable that training data samples whose $|\delta|$ is close to 0 can be removed when computing the direction to achieve a flat minimum.

\begin{figure}[t] 
    \centering
    \includegraphics[scale=0.7]{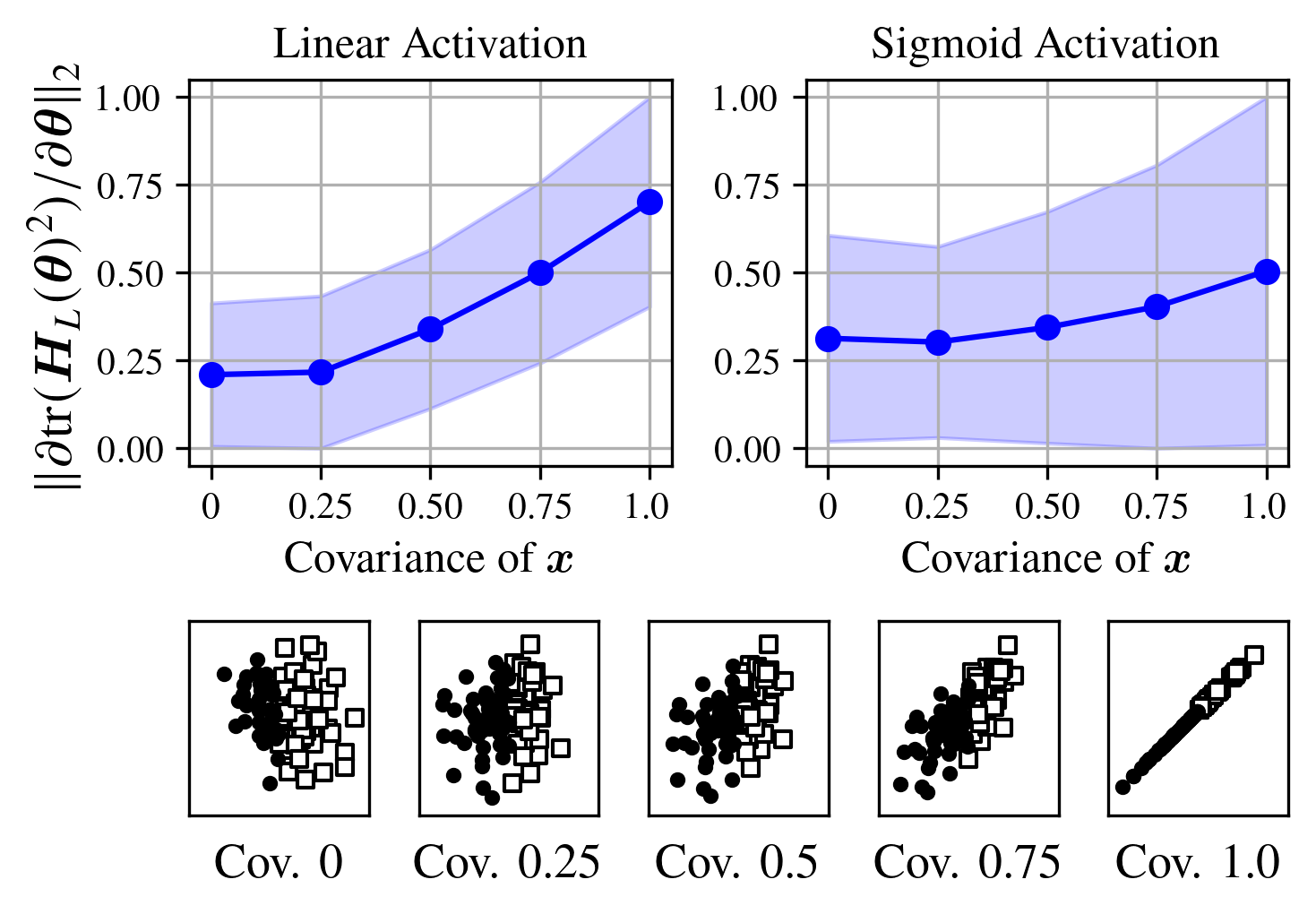}
    \caption{
Relationship between the covariance of the bivariate Gaussian distribution generating the input data $\bm{x}$ and the gradient of the squared Hessian trace.
The left panel represents the case of linear activation, and the right panel represents the case of sigmoid activation.
The values are normalized to a range of 0 to 1 to clearly illustrate the scale of variation.
Solid lines represent the median, and shaded areas represent the IQR.
    }
    \label{fig_ip_and_norm}
\end{figure}

\subsection{Inner Product of Input Data and the Norm of the Steepest Descent Direction}
By observing $\pa \mathrm{tr} (\bm{H}_L(\bm{\theta})^2) / \pa \bm{w}$ and $\pa \mathrm{tr} (\bm{H}_L(\bm{\theta})^2) / \pa \bm{v}$, one can confirm the presence of the inner product between input data samples, $\bm{x}_i^\top \bm{x}_j$, within terms such as $\bm{\mathcal{A}}_{ij}^\Phi$ and $\bm{\mathcal{B}}_{ij}^\Phi$ in Eqs.~\eqref{eqq_Aij_Phi} and \eqref{eqq_Bij_Phi}, as well as $\bm{\mathcal{C}}_{ij}^\Phi$ and $\bm{\mathcal{D}}_{ij}^\Phi$ in Eqs.~\eqref{eqq_Cij_Phi} and \eqref{eqq_Dij_Phi}.
Therefore, it is considered that the directional similarity among training data samples affects the scaling factor of the gradient of the squared Hessian trace.
To verify this, we observed the variation in the gradient norm of the squared Hessian trace with respect to changes in the covariance of the Gaussian distribution used to generate the training data $\bm{x}$.
The training data samples were centered at the origin of a two-dimensional space, where a sample was assigned to class 0 if $x_1 < 0$ and to class 1 if $x_1 \geq 0$.
The structure of the NN and the size of the training dataset were configured as $(M, N, I) = (2, 3, 100)$.
The obtained results are shown in Fig.~\ref{fig_ip_and_norm}.
The left panel represents the case with linear activation, while the right panel represents the case with sigmoid activation.
As can be seen from the figure, training data with a higher inner product leads to a larger gradient norm of the squared Hessian trace.
In light of Eq.~\eqref{eq_main_theorem_musig}, this implies that pairs of similar training data samples exert a strong influence in the direction that reduces the standard deviation of the eigenspectrum.
However, the scaling effect due to the inner product is large for linear activation and small for sigmoid activation.
In the case of saturating activation functions, the influence exerted by the inner product of the training data appears to diminish.

\begin{figure}[t] 
    \centering
    \includegraphics[scale=0.75]{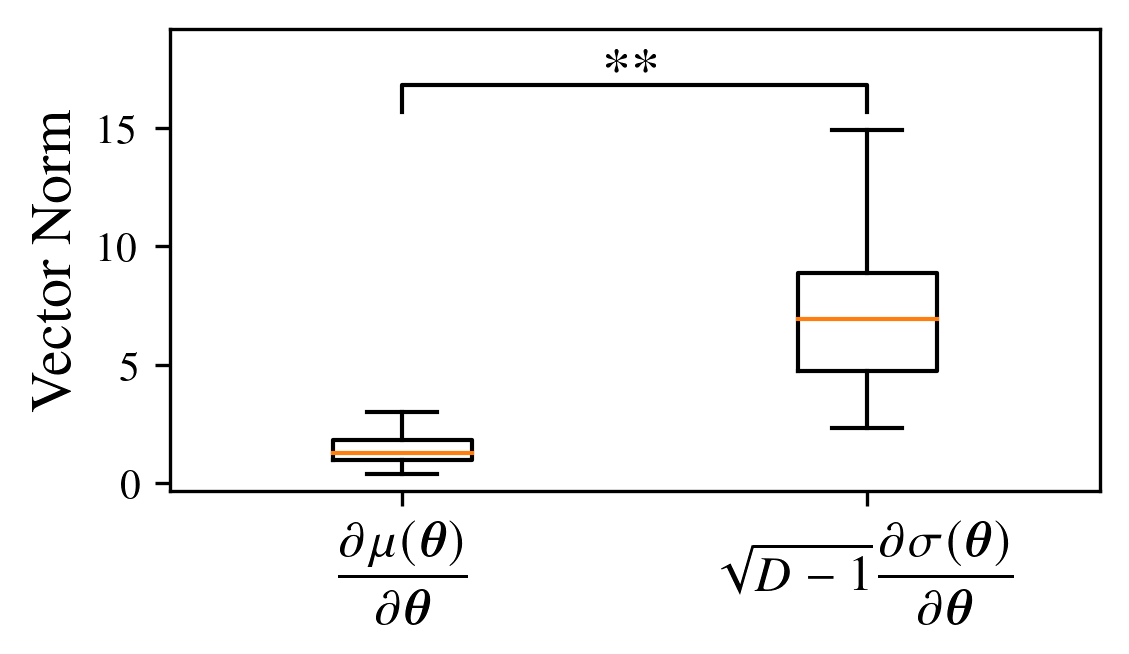}
    \caption{
    Comparison of gradient norms between the mean and standard deviation terms of $\pa \lambda_\mathrm{sup}(\bm{\theta}) / \pa \bm{\theta}$.
Calculated using 100 randomly generated NN patterns under the condition of $(I, M, N) = (50, 2, 3)$.
Asterisks indicate p-values from the two-sided Wilcoxon signed-rank test (**: 0.1\% level, *: 1\% level).
Boxes represent the IQR, and whiskers indicate the $\pm 1.5 \times \mathrm{IQR}$ range.
The activation function of the hidden layer is sigmoid.
    }
    \label{fig_musig_norm}
\end{figure}

\subsection{Dominant Terms for the Direction Toward Flat Minima}
From Eq.~\eqref{eqq_main_theo_direction}, the direction toward flat minima, $- \pa \lambda_\mathrm{sup}(\bm{\theta}) / \pa \bm{\theta}$, can be decomposed into terms that reduce the mean and standard deviation of the eigenspectrum, which are given by
\ali{
- \frac{\pa \mu(\bm{\theta})}{\pa \bm{\theta}}, \ 
- \sqrt{D-1} \frac{\pa \sigma(\bm{\theta})}{\pa \bm{\theta}}. \nn
}
From Eqs.~\eqref{eqq_h1_w_div_theorem} and \eqref{eqq_h1_v_div_theorem}, $\pa \mu(\bm{\theta}) / \pa \bm{\theta}$ is expressed by a single summation over the training dataset size $I$, whereas from Eqs.~\eqref{eqq_H2_w_div_fin} and \eqref{eqq_H2_v_div_fin}, $\pa \sigma(\bm{\theta}) / \pa \bm{\theta}$ is expressed by a double summation over the training dataset size $I$.
Furthermore, the architecture is structured such that $\pa \sigma(\bm{\theta}) / \pa \bm{\theta}$ is scaled by the NN parameter size $D$.
Therefore, it is expected that the direction toward flat minima, $- \pa \lambda_\mathrm{sup}(\bm{\theta}) / \pa \bm{\theta}$, is strongly influenced by the direction that reduces the standard deviation of the eigenspectrum rather than its mean.
To verify whether this hypothesis is correct, we calculated the norms of $\pa \mu(\bm{\theta}) / \pa \bm{\theta}$ and $\sqrt{D-1} \pa \sigma(\bm{\theta})/\pa \bm{\theta}$ using multiple randomly generated initial parameters.
The results are shown in Fig.~\ref{fig_musig_norm}.
As can be seen from the figure, as expected, it can be determined that the norm of $\sqrt{D-1} \pa \sigma(\bm{\theta})/\pa \bm{\theta}$ is larger than that of $\pa \mu(\bm{\theta}) / \pa \bm{\theta}$.

\begin{figure}[t] 
    \centering
    \includegraphics[scale=0.7]{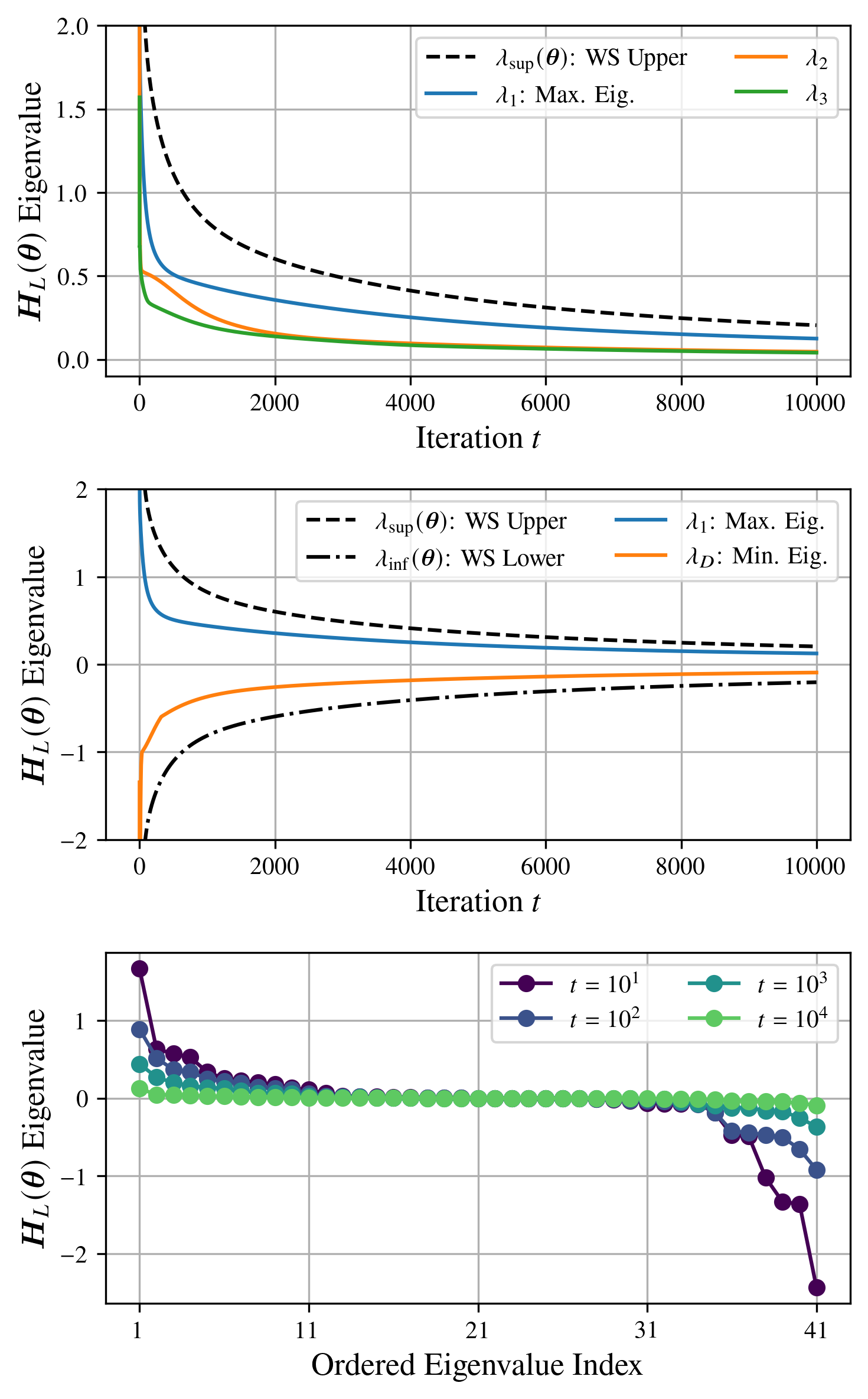}
    \caption{
    Top: Training dynamics of the WS upper bound and top eigenvalues.
    Middle: Training dynamics of the WS upper and lower bounds, maximum eigenvalue, and minimum eigenvalue.
    Bottom: Details of the eigenspectrum during training.
    Here, $(\gamma_1, \gamma_2) = (0, 0.01)$, which represents the case where only the WS upper bound is reduced without decreasing the loss.
    The activation function of the hidden layer is sigmoid.
}
    \label{fig_ws_ul}
\end{figure}

\subsection{HSR Regularization}
Although moving the NN parameters in the direction of $- \pa \lambda_\mathrm{sup}(\bm{\theta}) / \pa \bm{\theta}$ is expected to decrease the WS upper bound and consequently reduce the eigenvalues, it is practically necessary to simultaneously reduce the loss $L(\bm{\theta})$.
Therefore, in this study, we refer to $\lambda_\mathrm{sup}(\bm{\theta})$ as the HSR regularization term, and we propose a method termed ``HSR Regularization'' that minimizes the weighted sum of $L(\bm{\theta})$ and $\lambda_\mathrm{sup}(\bm{\theta})$.
That is, the update equation is given by
\ali{
\bm{\theta}^{(t+1)} = \bm{\theta}^{(t)} - \bigg(
\gamma_1 \left. \frac{\pa L(\bm{\theta})}{\pa \bm{\theta}} \right|_{\bm{\theta} = \bm{\theta}^{(t)}} + 
\gamma_2 \left. \frac{\pa \lambda_\mathrm{sup}(\bm{\theta})}{\pa \bm{\theta}} \right|_{\bm{\theta} = \bm{\theta}^{(t)}}
\bigg), \nn
}
where $\gamma_1$ is the weight for the loss reduction and $\gamma_2$ is the weight for the reduction of the WS upper bound.

To verify whether moving in the steepest descent direction of the WS upper bound decreases both the WS upper bound and the eigenvalues, we executed the gradient descent method under the conditions of $\gamma_1 = 0$ and $\gamma_{2} = 0.01$.
The results are shown in the upper panel of Fig.~\ref{fig_ws_ul}.
As can be seen from the figure, both the WS upper bound and the eigenvalues decrease as the training progresses.
Therefore, it can be said that HSR regularization has the effect of lowering the second-order derivatives of the loss function, and it is considered to have the effect of facilitating the attainment of flat minima.

\subsection{Effect of Elevating the Wolkowicz-Styan Lower Bound}\label{secc_inf_lambda_up}
Wolkowicz et al.~\cite{wolkowiczBounds1980} derived not only the upper bound for the maximum eigenvalue but also a closed-form expression for the lower bound of the minimum eigenvalue.
By applying Eq.~(2.2) in \cite{wolkowiczBounds1980}, the lower bound $\lambda_\mathrm{inf}(\bm{\theta})$ for the minimum eigenvalue $\lambda_D$ is given by
\ali{
\lambda_D \geq \lambda_\mathrm{inf}(\bm{\theta}) &= \mu(\bm{\theta}) - \sqrt{D-1} \sigma(\bm{\theta}).\label{eqq_lambda_inf}
}
Moreover, Fig.~\ref{fig_musig_norm} indicates that the direction toward flat minima is more easily determined by reducing $\sigma(\bm{\theta})$ rather than $\mu(\bm{\theta})$.
As an extreme example of this, let us consider a case where $\mu(\bm{\theta})$ remains unchanged while only $\sigma(\bm{\theta})$ decreases.
Under this condition, the following holds true.
\begin{proposition}
\ali{
&\Big( \sigma(\bm{\theta}^{(t+1)}) < \sigma(\bm{\theta}^{(t)}) \Big) \land 
\Big( \mu(\bm{\theta}^{(t+1)}) = \mu(\bm{\theta}^{(t)}) \Big)
\Rightarrow \nn \\
&\Big( \lambda_\mathrm{sup}(\bm{\theta}^{(t+1)}) < \lambda_\mathrm{sup}(\bm{\theta}^{(t)}) \Big)
\land
\Big( \lambda_\mathrm{inf}(\bm{\theta}^{(t+1)}) > \lambda_\mathrm{inf}(\bm{\theta}^{(t)}) \Big).
\label{eqq_inf_sup_rel}
}
\end{proposition}
\begin{proof}
See Appendix \ref{secc_pr_sup_inf}.
\end{proof}
This proposition asserts that training to decrease the WS upper bound to reduce the maximum eigenvalue inherently increases the WS lower bound.
To verify whether this phenomenon actually occurs, we performed the gradient descent method with $\gamma_1 = 0$ and $\gamma_{2} = 0.01$ to minimize only the WS upper bound.
The results are shown in the middle panel of Fig.~\ref{fig_ws_ul}.
As can be seen from the figure, as the training progresses, the WS upper bound decreases while the WS lower bound increases.
Concurrently, it can be observed that the maximum eigenvalue approaches 0 from the positive direction, and the minimum eigenvalue approaches 0 from the negative direction.
The lower panel of Fig.~\ref{fig_ws_ul} illustrates the details of the eigenspectrum, confirming that the range of the eigenvalues narrows as the training progresses.
Generally, when executing an operation to reduce the maximum eigenvalue, there is a potential risk of the eigenvalues becoming negative.
However, because HSR regularization also has the effect of increasing the WS lower bound, this risk is mitigated.
Since HSR regularization can be interpreted as having the effect of narrowing the range of the eigenspectrum, it is considered to facilitate the attainment of flat minima.
From the perspective of the flatness hypothesis, this is a highly desirable property.

\begin{figure}[t] 
    \centering
    \includegraphics[scale=0.65]{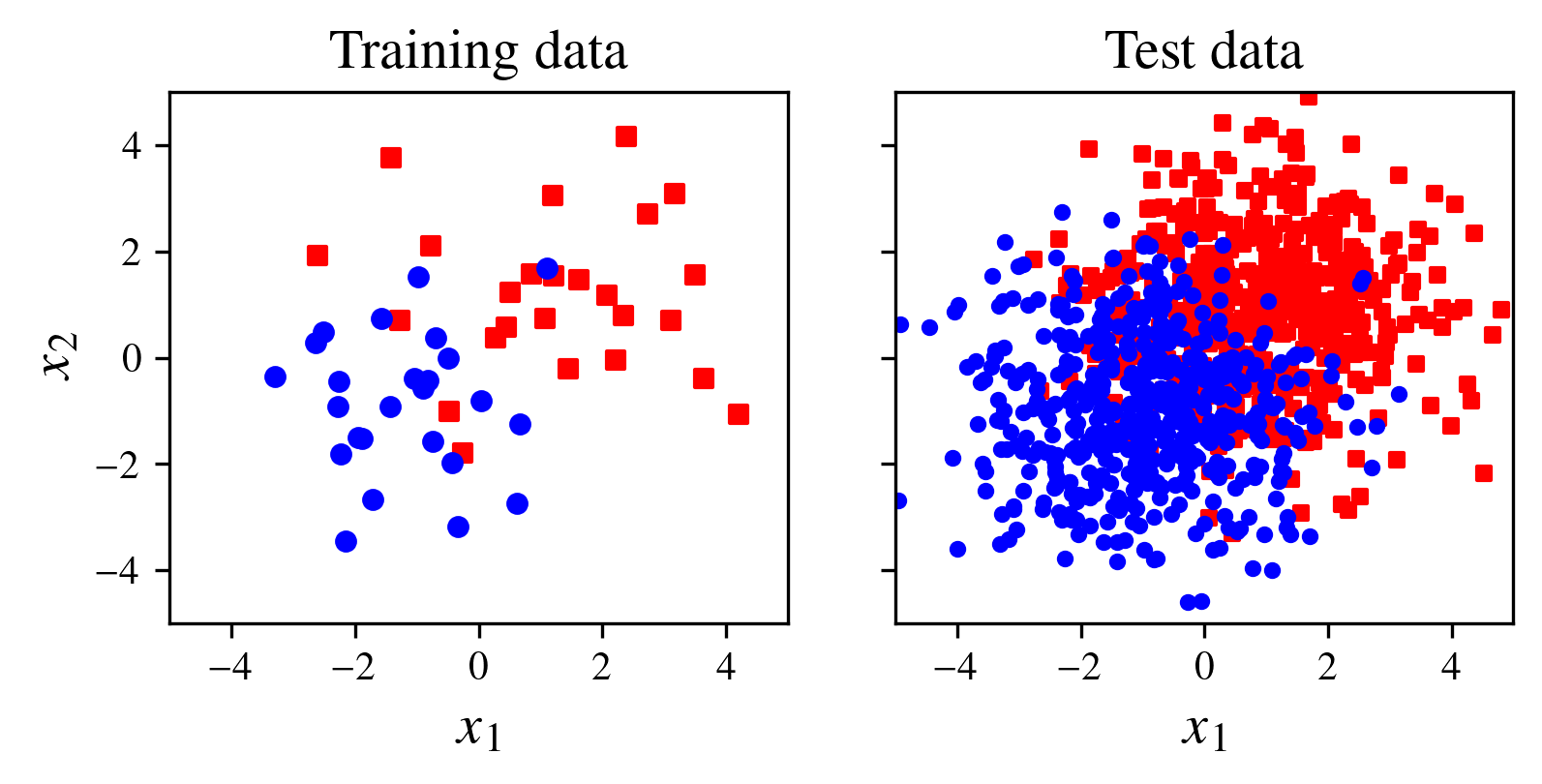}
    \caption{
    Left: Training data (50 samples), Right: Test data (1,000 samples). Both are balanced datasets.
    }
    \label{fig_sc_data}
\end{figure}

\section{Experiments} \label{sec_expe}
\subsection{Determination of Critical Points for Analysis}
Experiments were conducted with the objective of verifying whether the steepest descent direction of the WS upper bound derived in this study has the effect of improving generalization.
In this verification, a key factor is whether the attainment of sharp minima can be avoided.
Therefore, we searched for initial parameters that lead to sharp minima using a plain gradient descent method without any modifications.
The specific procedure for this process is described below.

First, we constructed a three-layer feedforward NN with input layer dimension $M=2$ and hidden layer dimension $N=3$.
The activation function of the hidden layer was set as the sigmoid function.
Specifically, by using
\ali{
f(y) &= 1/(1+\exp(-y)), \nn \\
f^\prime(y) &= (1-f(y))f(y), \nn \\
f^{\prime\prime}(y) &= (1-2f(y))(1-f(y))f(y), \nn \\
f^{\prime\prime\prime}(y) &= 6 (f(y) - a^+)(f(y) - a^-) (1-f(y))f(y), \nn \\
a^+ &= (3+\sqrt{3})/6, \ a^- = (3-\sqrt{3})/6 \nn
}
to construct $\bm{f}(\bm{y})$, $\bm{f}^{(k)}(\bm{y})$, $\bm{F}^{(k)}(\bm{y})$, $\bm{f}^{(k)}_0(\bm{y})$, and $\bm{F}^{(k)}_0(\bm{y})$ in accordance with Appendix~\ref{secc_act}, both the WS upper bound and its steepest descent direction can be obtained as closed-form functions.

The task assigned to the NN is a two-class classification problem aimed at estimating which of two Gaussian distributions a given data sample $\bm{x}$ in a two-dimensional space was generated from.
Regarding the distribution parameters, Class 0 has a mean of $\mat{1 & 1}^\top$, and Class 1 has a mean of $\mat{-1 & -1}^\top$.
This specific classification problem has also been utilized in prior research investigating the eigenspectrum of the Hessian~\cite{sagunEigenvalues2017}.
The variances along both the $x_1$ and $x_2$ axes were set to 2, and the covariance was set to 0.
The training dataset size was configured as $I = 50$, and the test dataset size was set to $10^3$, with an equal split between Class 0 and Class 1.
For reference, a scatter plot generated under these conditions is shown in Fig.~\ref{fig_sc_data}.

\begin{figure}[t] 
    \centering
    \includegraphics[scale=0.73]{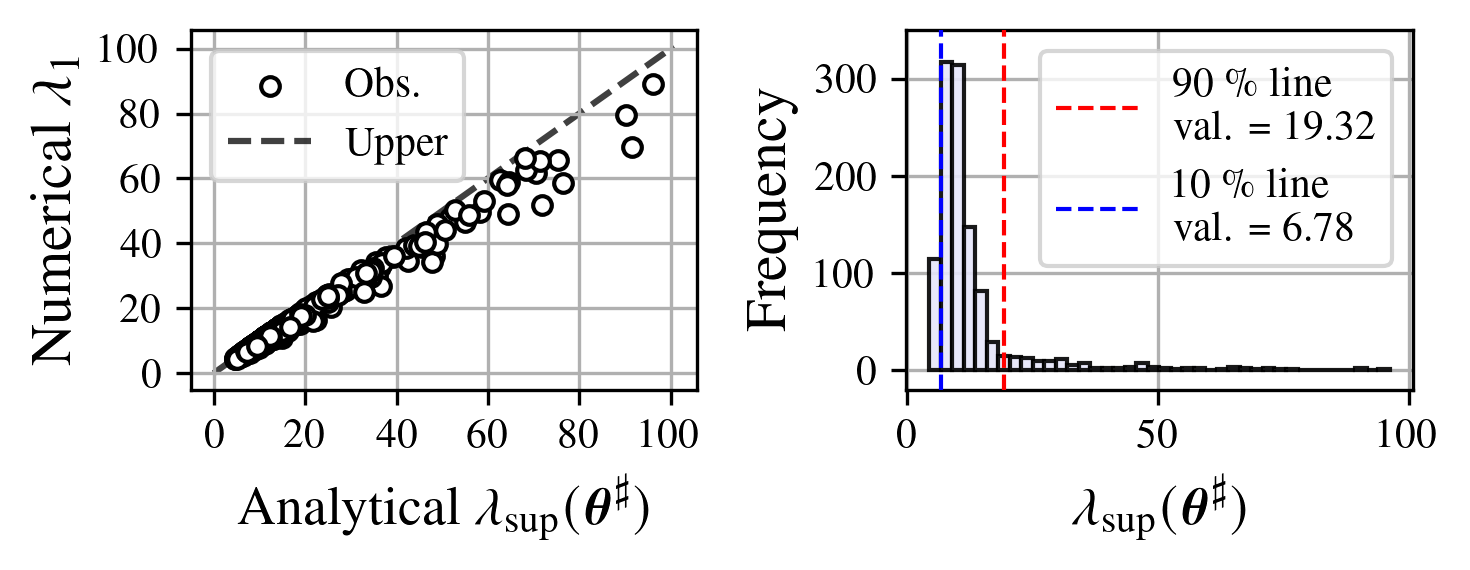}
    \caption{
    Left: Comparison between $\lambda_1$ and $\lambda_\mathrm{sup}(\bm{\theta}^\sharp)$. 
    Right: Histogram of $\lambda_\mathrm{sup}(\bm{\theta}^\sharp)$.
    Results are shown for 1,124 unique critical points that satisfy the convergence criteria.
    }
    \label{fig_sc_hist}
\end{figure}

To construct the NNs, the initial parameters $\bm{\theta} = \bm{\theta}^{(0)}$ were generated using 3,000 different random seeds.
For the random number generation, a uniform distribution over the open interval $(-10, 10)^D$ was adopted.
The reason for employing a wide range for random number generation was to observe a diverse set of critical points.
For the parameter update, a plain gradient descent method given by
\ali{
\bm{\theta}^{(t+1)} = \bm{\theta}^{(t)} - \gamma \frac{\pa L(\bm{\theta})}{\pa \bm{\theta}}\bigg|_{\bm{\theta}=\bm{\theta}^{(t)}}, \ t \in \{0, \cdots, T-1 \}, \nn
}
where $T$ denotes the maximum number of training iterations, and $T=10^4$ was utilized.
When the gradient norm of the loss $L$ fell below 0.1\% of its initial gradient norm, it was determined that convergence to a critical point had been achieved, and the training was terminated.
Furthermore, since it is inappropriate to include duplicate critical points in the analysis, only one of the closely located critical points was retained.
Specifically, duplicate points were defined as those where the distance between critical points was less than $\mathrm{Ave.} - 3 \times \mathrm{Std.}$, where ``$\mathrm{Ave.}$'' and ``$\mathrm{Std.}$'' denote the mean and standard deviation of the Euclidean distances among all critical points.
Critical points are represented as $\bm{\theta}^{\sharp}$.

As a result of executing the aforementioned procedure, 2,799 out of the 3,000 parameters satisfied the convergence condition to a critical point.
Furthermore, the number of non-duplicate critical points was found to be 1,124.
A scatter plot of the maximum eigenvalues and their upper bounds at these critical points is shown in the left panel of Fig.~\ref{fig_sc_hist}.
As can be seen from the figure, the WS upper bound is tight with respect to the maximum eigenvalue.
Therefore, the WS upper bound can be regarded as a function that evaluates the sharpness of the loss function.
Additionally, a histogram regarding the WS upper bound of the 1,124 critical points is shown in the right panel of Fig.~\ref{fig_sc_hist}.
From this plot, it can be confirmed that the distribution of the WS upper bound is right-skewed with a long tail.
Therefore, the 113 critical points whose WS upper bounds lie in the top 90\% or higher are termed ``Sharp Minima'' and serve as the primary subject of analysis in this study.
To perform a comparative analysis, the 113 critical points whose WS upper bounds lie in the bottom 10\% or lower are termed ``Flat Minima'' and are also included in the analysis of this study.
The red and blue dashed lines in the right panel of Fig.~\ref{fig_sc_hist} represent these respective thresholds.
That is, the analysis in this study targets a total of 226 independent critical points.

\begin{figure}[t]
    \centering
    \begin{subfigure}{\linewidth}
        \centering
        \includegraphics[width=1\linewidth]{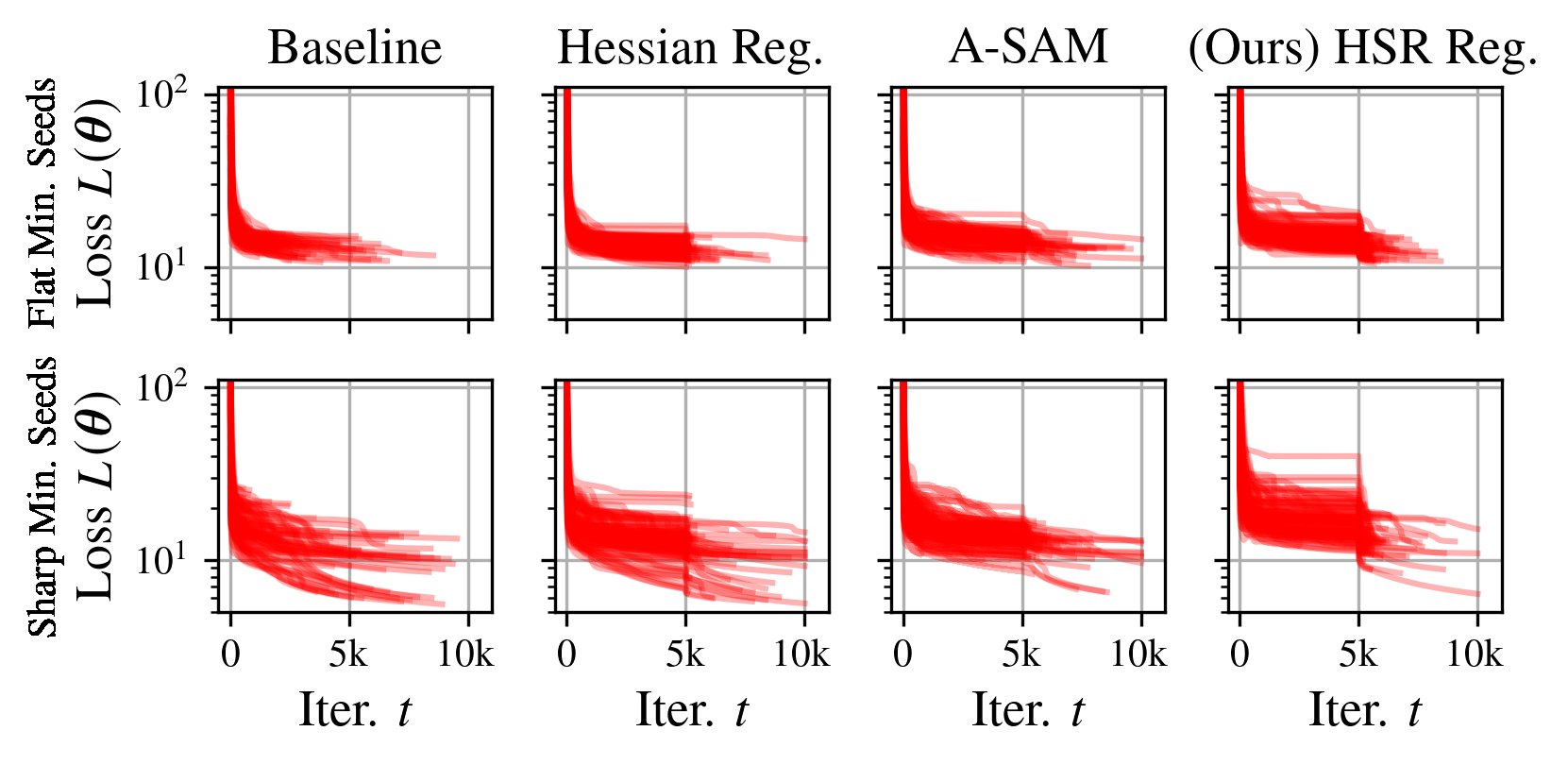}
        \caption{$L(\bm{\theta})$ Dynamics}
        \label{fig_eigen_dyn_1}
    \end{subfigure}

    \vspace{0.5em}

    \begin{subfigure}{\linewidth}
        \centering
        \includegraphics[width=1\linewidth]{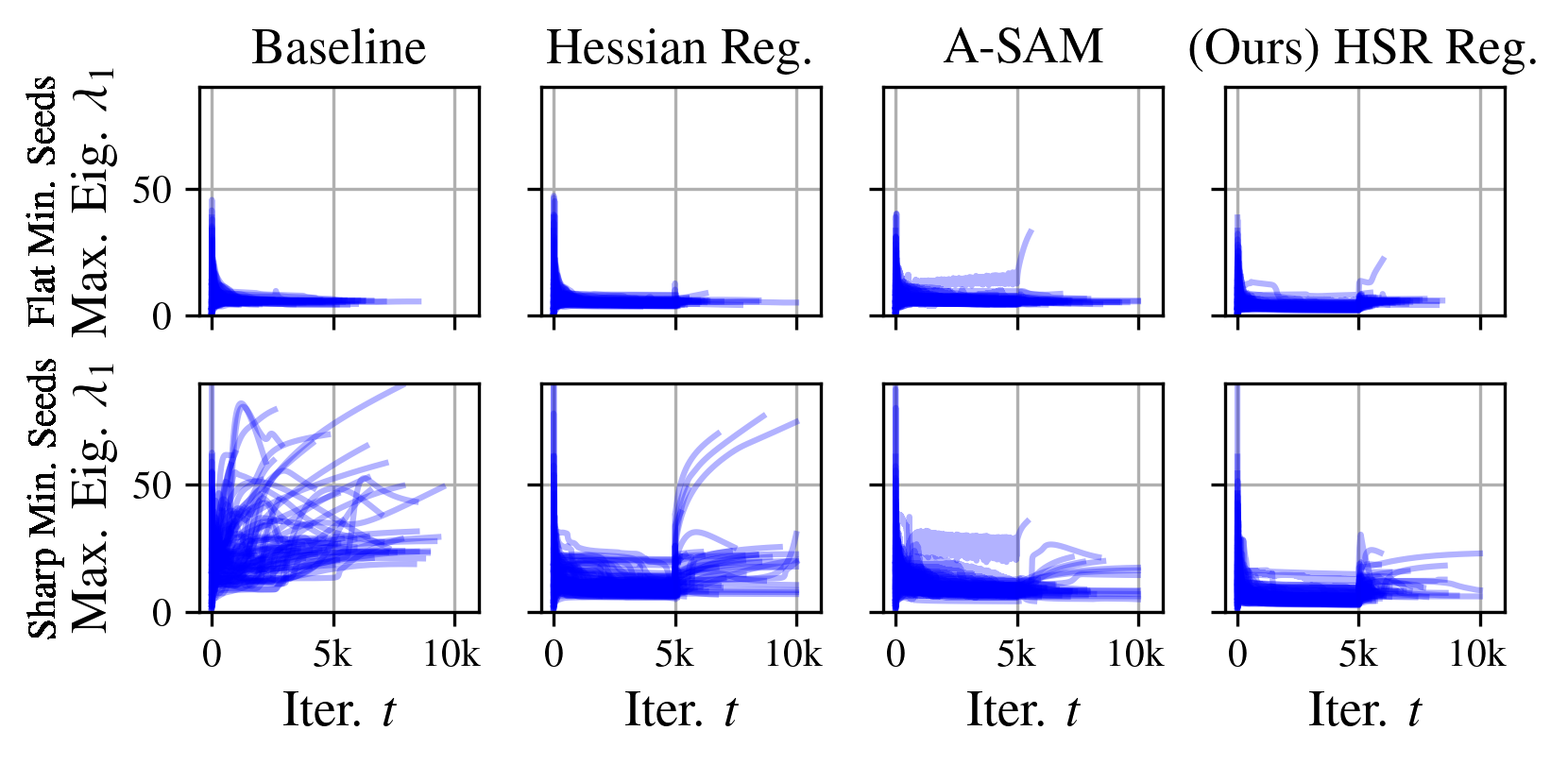}
        \caption{$\lambda_1$ Dynamics}
        \label{fig_eigen_dyn_2}
    \end{subfigure}
    \caption{
    Comparison of training dynamics for each method.
    Results for all seeds belonging to Sharp Minima Seeds and Flat Minima Seeds are shown.
    }
    \label{fig_eigen_dyn}
\end{figure}

\subsection{Controllability of the Eigenspectrum}
Here, we discuss the capability of HSR regularization to control the eigenspectrum.
For convenience, the 113 random seeds that resulted in sharp minima via the plain gradient descent method are referred to as ``Sharp Minima Seeds,'' while the 113 random seeds that resulted in flat minima are referred to as ``Flat Minima Seeds.''
If initializing training with parameters $\bm{\theta}^{(0)}$ from the Sharp Minima Seeds can successfully guide the model toward flat minima, it can be interpreted that HSR regularization is effective.
However, it is also necessary to verify whether executing training with initial parameters $\bm{\theta}^{(0)}$ from the Flat Minima Seeds induces any adverse effects.
Therefore, we applied HSR regularization to both the Sharp Minima Seeds and the Flat Minima Seeds.
Furthermore, to evaluate the relative effectiveness of HSR regularization, we also implemented existing methods designed to reduce eigenvalues, namely Hessian Regularization~\cite{liuHessian2023} and Adaptive Sharpness-Aware Minimization~\cite{kwonASAM2021} (A-SAM).
Although SAM~\cite{foretSharpnessAware2021} is available as a more basic approach, it should be noted that we employed its advanced variant, A-SAM~\cite{kwonASAM2021}, due to previously documented issues regarding parameter scaling in SAM.
Hessian Regularization is a method that minimizes the Hessian trace, whereas A-SAM identifies and moves toward a flatter region by introducing perturbations to the parameters.

For all methods, the maximum number of training iterations was set to $T=10^4$.
However, the operation to reduce the eigenvalues was executed only during the first 5,000 iterations, whereas the operation to minimize only the loss function $L(\bm{\theta})$ was performed during the remaining 5,000 iterations.
The reason for this strategy is that we prioritized reducing the eigenvalues in the first half of training and focused on driving the loss function to reach a critical point in the second half.
Specifically, during the first 5,000 training iterations, the weight for the loss gradient term was $\gamma_1=0.01$, and the weight for HSR regularization was $\gamma_2 = 0.01$.
Similarly, the weight for the Hessian trace term in Hessian Regularization was set to 0.01.
Additionally, the perturbation radius for A-SAM was configured to 0.1, as a larger perturbation radius failed to yield a noticeable reduction in the maximum eigenvalue.
In the subsequent 5,000 training iterations, a plain gradient descent method with a loss gradient term weight of $\gamma_1=0.01$ was implemented for all methods.
As before, when the gradient norm of the loss $L$ fell below 0.1\% of its initial gradient norm, it was determined that convergence to a critical point had been achieved, and the training was terminated.

The training dynamics obtained through these procedures are shown in Fig.~\ref{fig_eigen_dyn}.
Fig.~\ref{fig_eigen_dyn_1} represents the evolution of the loss, and Fig.~\ref{fig_eigen_dyn_2} illustrates the evolution of the maximum eigenvalue.
Here, ``Baseline'' refers to the plain gradient descent method that minimizes only the loss, ``Hessian Reg.'' denotes Hessian Regularization, ``A-SAM'' represents Adaptive SAM, and ``HSR Reg.'' signifies HSR regularization.
Furthermore, the upper panel of each figure displays the results for the Flat Minima Seeds, which are the initial parameters where even the plain gradient descent method reaches flat minima.
Looking at the results for the Flat Minima Seeds, it can be confirmed that both the loss and the eigenvalues exhibit almost identical behavior across all methods.
The lower panel of each figure shows the results for the Sharp Minima Seeds, indicating that the loss decreases across all methods.
On the other hand, it can be observed that the maximum eigenvalue reaches a large value at the critical point under the plain gradient descent method.
In contrast, the three methods that control the eigenvalues finish training while maintaining the maximum eigenvalue at a low value.

\begin{figure}[t] 
    \centering
    \includegraphics[scale=0.9]{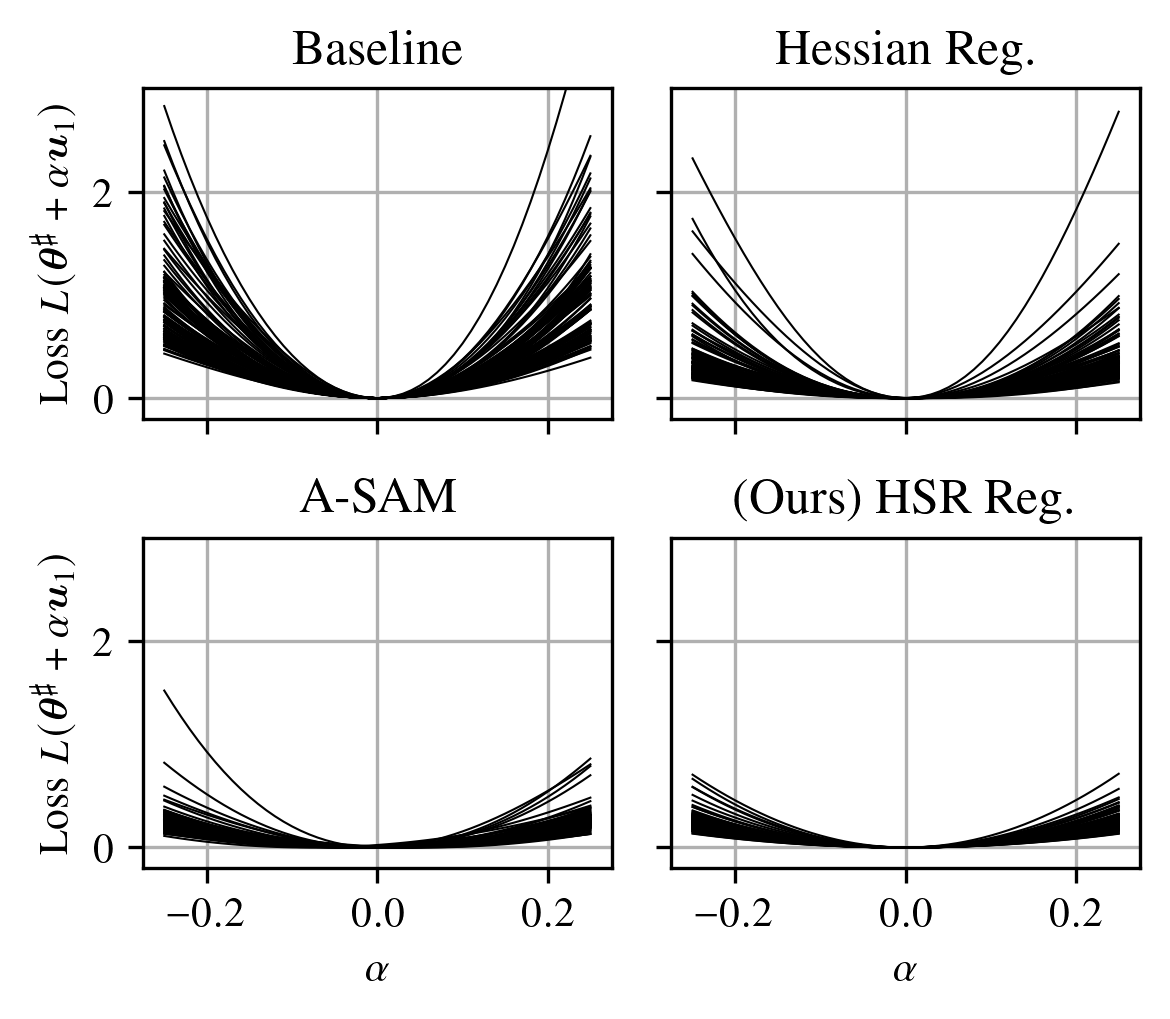}
    \caption{
    Loss landscapes at critical points for Sharp Minima Seeds (113 critical points).
    One-dimensional visualization along the direction of the eigenvector $\bm{u}_1$ corresponding to the maximum eigenvalue.
    For visual clarity, the curves are vertically shifted so that the loss at each critical point is aligned at 0.}
    \label{fig_loss_landscape}
\end{figure}

Fig.~\ref{fig_loss_landscape} visualizes the loss function at the converged critical points at the end of training for the Sharp Minima Seeds.
Since the loss function is high-dimensional, it cannot be visualized directly.
Therefore, we performed the visualization along the direction of the eigenvector $\bm{u}_1$ corresponding to the maximum eigenvalue $\lambda_1$.
From this plot, it can be seen that each method controlling the eigenvalues prefers flatter minima compared to the plain gradient descent method.
Among them, it can be confirmed that HSR regularization consistently converges to flat minima.

\begin{figure*}[t]
    \centering
    \begin{subfigure}{0.49\linewidth}
        \centering
        \includegraphics[scale=0.8]{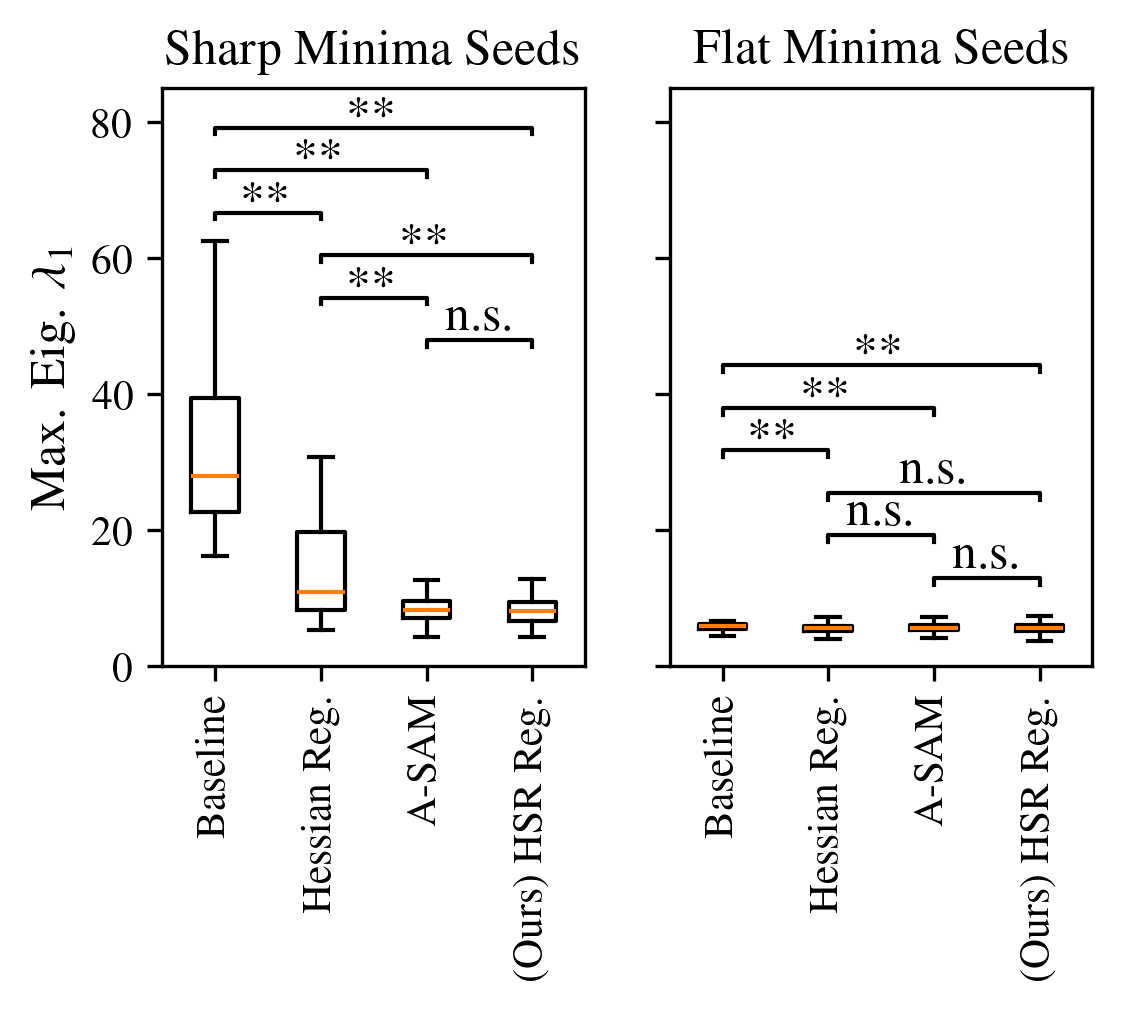}
        \caption{Maximum eigenvalue}
        \label{fig_e_max}
    \end{subfigure}
    \hfill
    \begin{subfigure}{0.49\linewidth}
        \centering
        \includegraphics[scale=0.8]{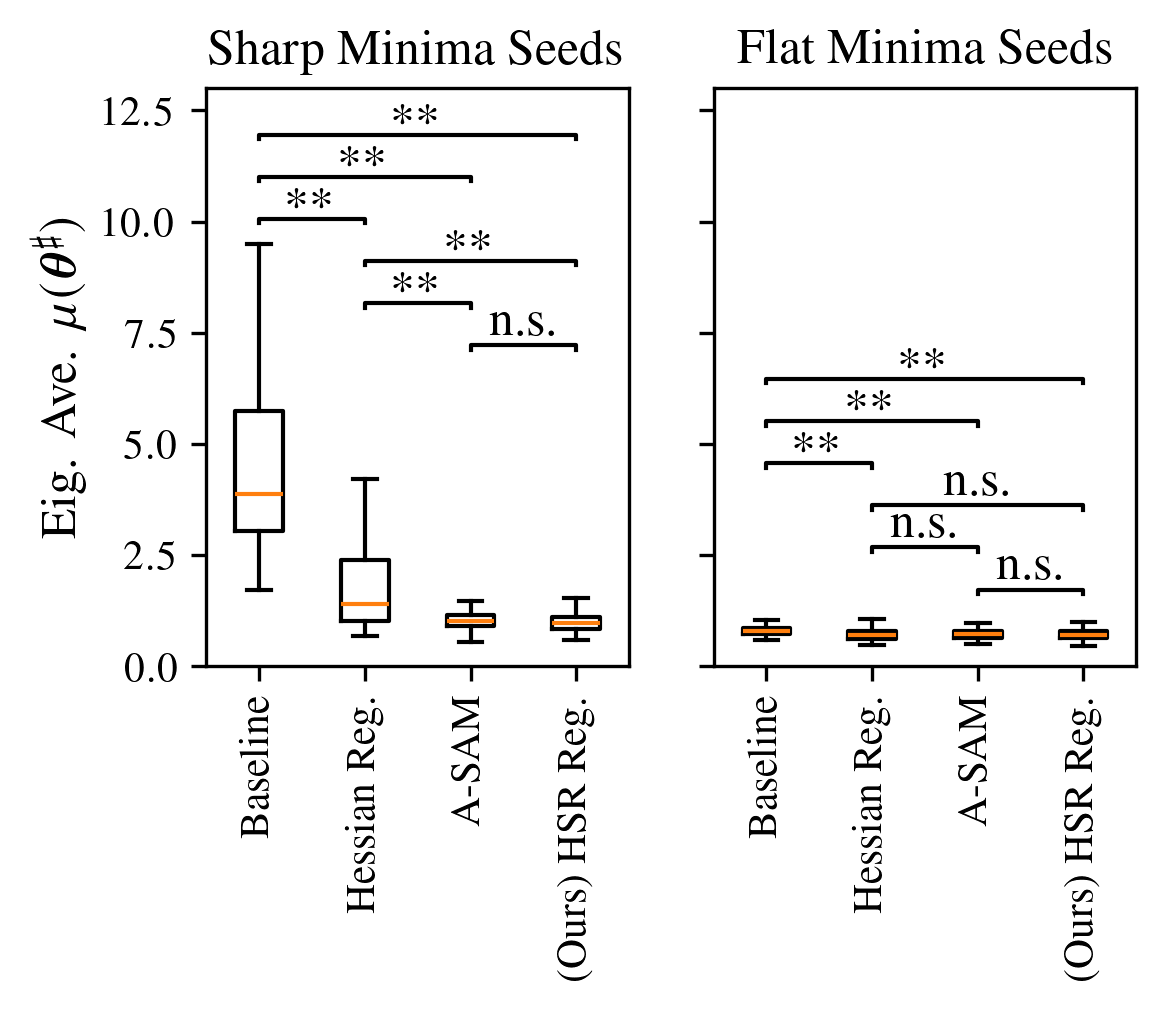}
        \caption{Average of eigenvalues}
        \label{fig_e_ave}
    \end{subfigure}

    \vspace{0.5em}

    \begin{subfigure}{0.49\linewidth}
        \centering
        \includegraphics[scale=0.8]{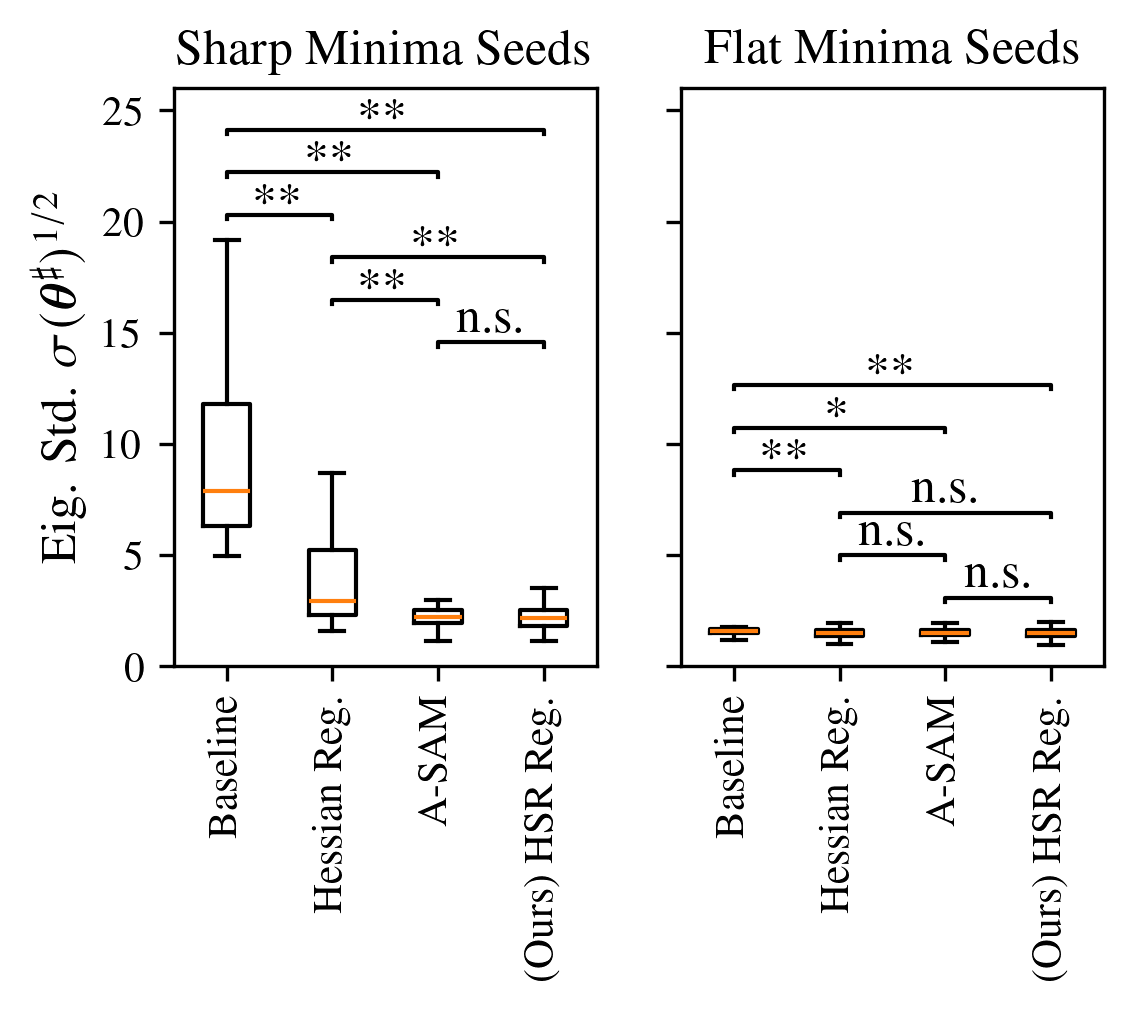}
        \caption{Standard deviation of eigenvalues}
        \label{fig_e_std}
    \end{subfigure}
    \hfill
    \begin{subfigure}{0.49\linewidth}
        \centering
        \includegraphics[scale=0.8]{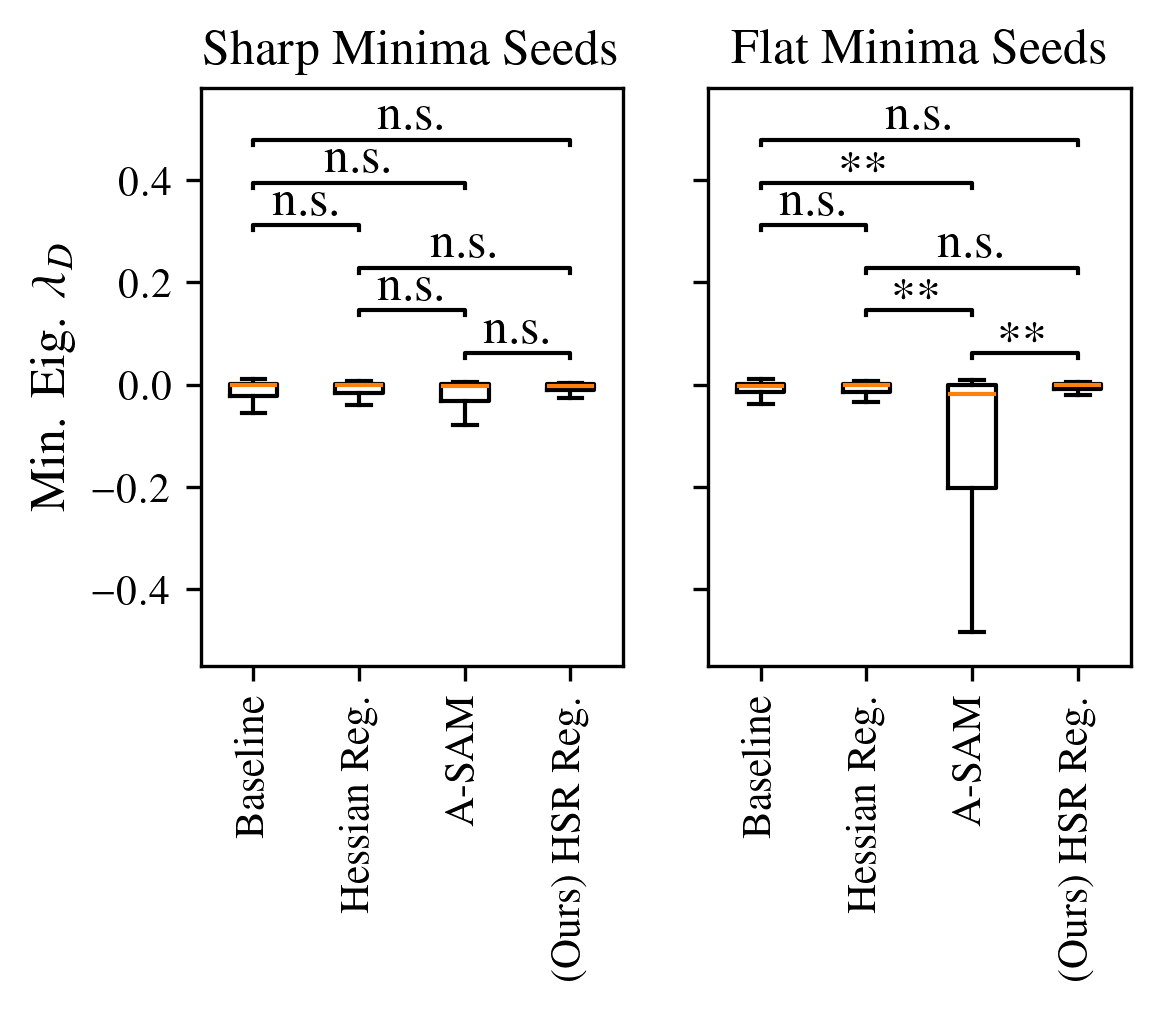}
        \caption{Minimum eigenvalue}
        \label{fig_e_min}
    \end{subfigure}

    \caption{
Statistics of the Hessian eigenspectrum at critical points.
Asterisks indicate p-values from the two-sided Wilcoxon signed-rank test (**: 0.1\% level, *: 1\% level).
Boxes represent the IQR, and whiskers indicate the $\pm 1.5 \times \mathrm{IQR}$ range.
    }
    \label{fig_eigen_stats}
\end{figure*}

To quantitatively compare the sharpness of the critical points, we examined several statistics regarding the eigenspectrum at the converged critical points.
The results are shown in Fig.~\ref{fig_eigen_stats}.
Fig.~\ref{fig_e_max} presents a comparison of the maximum eigenvalues.
In the case of the Sharp Minima Seeds, it can be seen that all methods successfully reduced the maximum eigenvalue compared to the plain gradient descent method without regularization.
Among these, A-SAM and HSR regularization were particularly effective.
The reason HSR regularization outperformed Hessian regularization is considered to be that while Hessian regularization relies solely on the Hessian trace, HSR regularization additionally incorporates the squared Hessian trace.
In the case of the Flat Minima Seeds, all methods achieved low values.
Therefore, no adverse effects of eigenvalue control on the maximum eigenvalue were observed.

Fig.~\ref{fig_e_ave} and Fig.~\ref{fig_e_std} present the mean and standard deviation of the eigenspectrum, respectively.
For these statistics, almost the same trends as the maximum eigenvalue were observed.
In deep learning, it is widely known that the resulting eigenspectrum tends to have only a few top eigenvalues with non-zero values, while the remaining eigenvalues are close to zero~\cite{wuDissecting2022}\cite{xiePowerLaw2022}\cite{papyanTraces2020}.
In such a situation, the operation to decrease the maximum eigenvalue is equivalent to the operation to lower the mean of the eigenvalues.
Furthermore, lowering the top eigenvalues implies that they approach the mean, which leads to a reduction in the standard deviation of the eigenspectrum.
That is, a low maximum eigenvalue and low values for both the mean and standard deviation of the eigenspectrum are considered to be strongly related properties.

\begin{figure}[t] 
    \centering
    \includegraphics[scale=0.88]{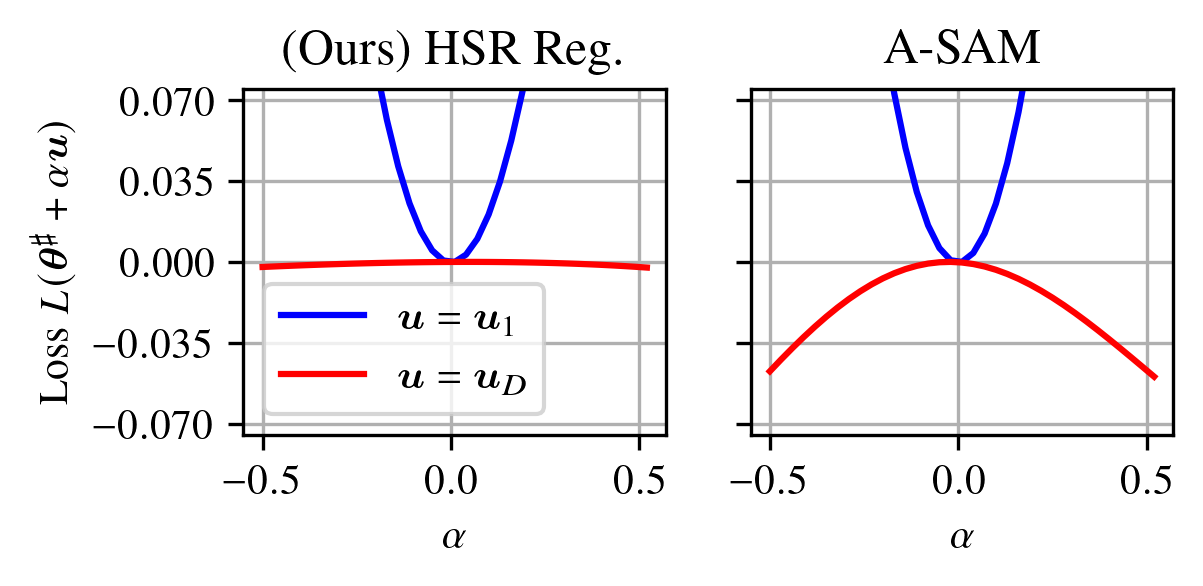}
    \caption{
    Loss landscape along the eigenvector directions at the critical point.
Here, $\bm{u}_1$ and $\bm{u}_D$ represent the eigenvectors corresponding to $\lambda_1$ and $\lambda_D$, respectively.
The result is obtained using a random seed where the minimum eigenvalue is located around $-1.5 \times \mathrm{IQR}$ in the distribution shown in Fig.~\ref{fig_e_min}.
For better visual clarity, the loss function is vertically shifted such that its minimum value is 0 along the $\bm{u}_1$ direction, and its maximum value is 0 along the $\bm{u}_D$ direction.
}
    \label{fig_l_critpoint}
\end{figure}

Fig.~\ref{fig_e_min} summarizes the minimum eigenvalues of the critical points reached by each method.
From this plot, it can be seen that the minimum eigenvalues are close to 0 in most cases.
On the other hand, for the case of A-SAM with the Flat Minima Seeds, the distribution of the minimum eigenvalues tended to be left-skewed with a long tail in the negative direction.
Since this suggests the presence of saddle points, we performed a one-dimensional visualization of the loss at the critical point located at the tip of the whisker for A-SAM.
For comparison, we also examined the loss shape at the critical point located at the tip of the whisker of the minimum eigenvalue for HSR regularization.
The resulting plots are shown in Fig.~\ref{fig_l_critpoint}.
The blue and red lines represent the loss shapes along the directions of the eigenvectors corresponding to the maximum and minimum eigenvalues, respectively.
From these results, a clear saddle point was observed for A-SAM.
In recent deep learning literature, this may not pose a significant issue because several methods exist to escape from saddle points.
However, since the termination of training at a saddle point carries a potential risk of training failure, it cannot be considered a desirable state.
In contrast, no clear saddle point was observed for HSR regularization.
The reason for this is considered to be that HSR regularization controls not only the maximum eigenvalue but also the lower bound of the minimum eigenvalue.

\begin{figure}[t] 
    \centering
    \includegraphics[scale=0.9]{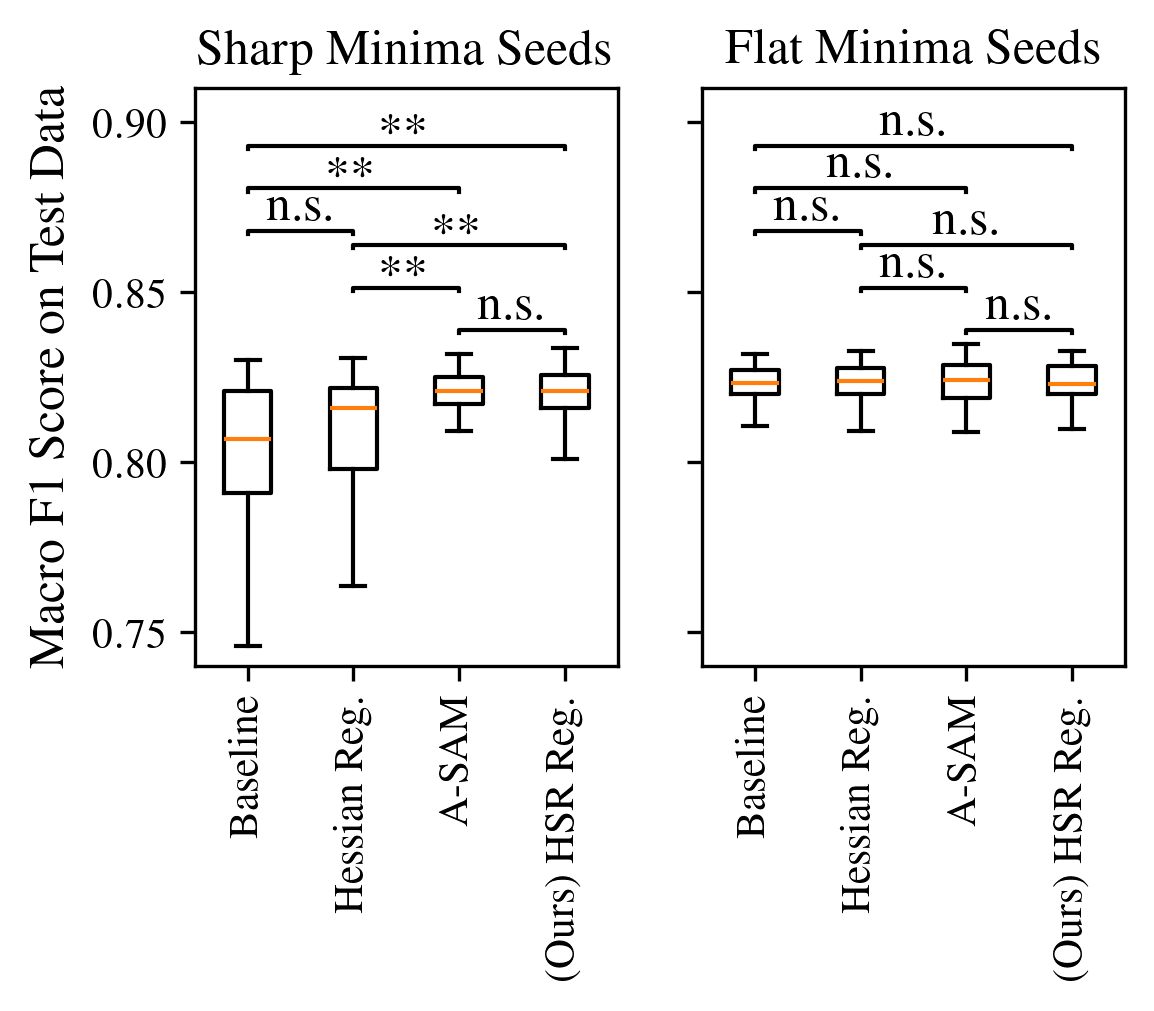}
    \caption{
Macro F1-scores at the critical point. 
Asterisks indicate statistical significance based on the two-sided Wilcoxon signed-rank test 
($^{**}\!$: 0.1\% level, $^{*}\!$: 1\% level). 
Boxes represent the IQR, and whiskers extend to $\pm 1.5 \times \text{IQR}$.}
    \label{fig_f1}
\end{figure}

\subsection{Impact on Generalization Performance}
In the previous section, it was confirmed that HSR regularization has the effect of reducing the maximum eigenvalue while avoiding convergence to saddle points.
Although this is a desirable property from the perspective of pursuing flat minima, it would be counterproductive if it did not lead to an improvement in generalization performance.
Therefore, in this section, we verify whether HSR regularization has the effect of improving generalization performance.

Fig.~\ref{fig_f1} summarizes the macro F1 scores at the critical points reached by each method, evaluated using the test dataset.
Looking at the results for the Sharp Minima Seeds, it can be confirmed that the three methods controlling the eigenvalues successfully improved the test performance compared to the plain gradient descent method without regularization.
Among these, A-SAM and HSR regularization exhibited a higher improvement than Hessian regularization.
Therefore, it can be said that HSR regularization is expected to achieve a performance improvement comparable to that of A-SAM.
Furthermore, the results for the Flat Minima Seeds suggest that HSR regularization does not induce any adverse effects on initial parameters that inherently lead to flat minima without any specific modifications.

To investigate the reasons behind the performance improvement of HSR regularization, we visualized the decision boundaries for both the unregularized method and HSR regularization in the case of the Sharp Minima Seeds.
The results are shown in Fig.~\ref{fig_bline}.
The upper panel represents the decision boundaries without regularization, while the lower panel shows those obtained with HSR regularization.
Since the results in the same column are derived from the identical initial parameters $\bm{\theta}^{(0)}$, they have a one-to-one correspondence.
From these plots, it can be seen that the application of HSR regularization has the effect of simplifying the decision boundaries.
This implies the suppression of overfitting, which is a desirable property from the perspective of generalization.

\begin{figure}[t] 
    \centering
    \includegraphics[scale=0.78]{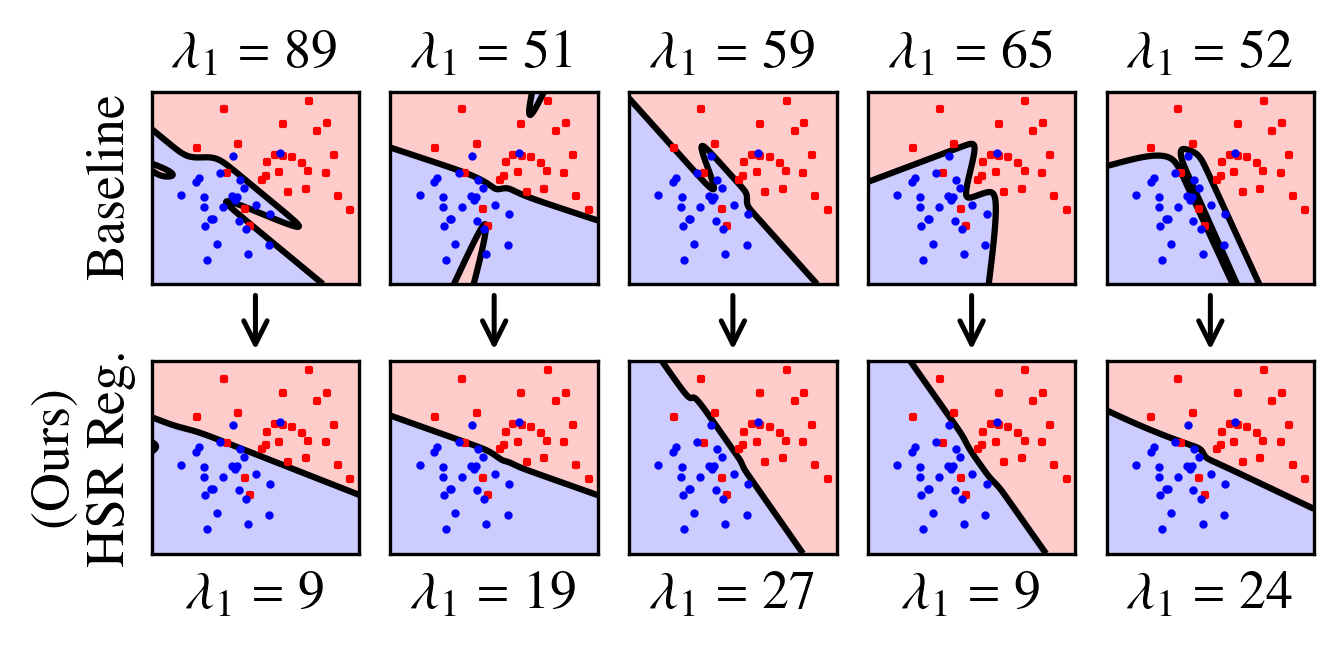}
    \caption{
    Relationship between decision boundaries and the presence of HSR regularization for Sharp Minima Seeds.
    Top: Without regularization (baseline). Bottom: With HSR regularization.
    }
    \label{fig_bline}
\end{figure}

\section{Limitations and Future Work}
In prior studies, a closed-form expression for the direction toward flat minima under the CE loss in feedforward NNs had not been discovered.
Therefore, in this study, we derived the steepest descent direction of the WS upper bound, which represents one such expression, in a closed form.
Through this derivation, it was confirmed that the eigenvalues can be controlled, and the proposed method demonstrated superior performance compared to the existing Hessian regularization.
Furthermore, it was confirmed to yield a performance improvement comparable to that of A-SAM.
These findings suggest that the closed-form expression of the steepest descent direction of the WS upper bound is a valuable function from the perspective of clarifying the relationship between NNs and generalization.
However, because the resulting function is extensively large, the theoretical investigation remains insufficient at this stage.
Accordingly, as one of our future tasks, we intend to conduct a detailed analysis of this closed-form function.

Furthermore, extending this approach toward practical applications is also crucial.
At present, the value of the WS upper bound demonstrated in this study remains limited to the theoretical aspect.
This limitation arises because the steepest descent direction of the WS upper bound can only be applied to a three-layer feedforward NN designed for solving two-class classification problems.
As a method intended for NNs capable of addressing a wide variety of tasks, such as multi-class classification, regression, and segmentation, this current formulation is overly restrictive.
In addition, HSR regularization cannot be applied to NNs with four or more layers.
Therefore, for future work, we plan to devise a methodology that enables the implementation of HSR regularization for arbitrary layer architectures and diverse tasks.

\appendices
\section{Preliminaries}
\subsection{Derivative Layout} \label{secc_layout}
In this study, we adopt the denominator layout for arranging derivatives in gradients and Jacobians. 
That is,
\ali{
&\bm{\varkappa}(\theta) \in \mathbb{R}^{D_1} \land \theta \in \mathbb{R} 
\Rightarrow
\frac{\pa \bm{\varkappa}(\theta)}{\pa \theta} = 
\mat{\frac{\pa \varkappa_1}{\pa \theta} \cdots \frac{\pa \varkappa_{D_1}}{\pa \theta} }, \nn \\
&\varkappa(\bm{\theta}) \in \mathbb{R} \land \bm{\theta} \in \mathbb{R}^{D_2} 
\Rightarrow
\frac{\pa \varkappa(\bm{\theta})}{\pa \bm{\theta}} = 
\mat{\frac{\pa \varkappa(\bm{\theta})}{\pa \theta_1} \cdots \frac{\pa \varkappa(\bm{\theta})}{\pa \theta_{D_2} } }^\top \nn
}
holds.
The Jacobian is defined as
\ali{
&\bm{\varkappa}(\bm{\theta}) \in \mathbb{R}^{\mathcal{D}_1} \land \bm{\theta} \in \mathbb{R}^{\mathcal{D}_2} \Rightarrow \nn \\
&\frac{\pa \bm{\varkappa}(\bm{\theta})}{\pa \bm{\theta}} = 
\frac{\pa }{\pa \bm{\theta}} \bm{\varkappa}(\bm{\theta})^\top
=
\mat{
\frac{\pa \varkappa_1}{\pa \theta_1} &\cdots& \frac{\pa \varkappa_{D_1}}{\pa \theta_1} \\
\vdots & \ddots & \vdots \\
\frac{\pa \varkappa_1}{\pa \theta_{D_2}} &\cdots& \frac{\pa \varkappa_{D_1}}{\pa \theta_{D_2}}
}.\nn
}
In this case, according to Eq.~(4) of \cite{khamisLearning2020}, the Leibniz rule for the gradient of an inner product is given by
\ali{
&\bm{\varkappa}(\bm{\theta}), \bm{\xi}(\bm{\theta}) \in \mathbb{R}^{\mathcal{D}_1} \land \bm{\theta} \in \mathbb{R}^{\mathcal{D}_2} \Rightarrow \nn \\
&\frac{\pa \bm{\varkappa}(\bm{\theta})^\top \bm{\xi}(\bm{\theta})}{\pa \bm{\theta}} 
= \frac{\pa \bm{\varkappa}(\bm{\theta})}{\pa \bm{\theta}} \bm{\xi}(\bm{\theta})
+ \frac{\pa \bm{\xi}(\bm{\theta})}{\pa \bm{\theta}} \bm{\varkappa}(\bm{\theta}) 
\in \mathbb{R}^{\mathcal{D}_2}.
\label{eqq_ripniz}
}

\subsection{Activation Functions} \label{secc_act}
Let $f(y) \in \mathbb{R}$ be an activation function for a scalar value $y \in \mathbb{R}$.  
For an intermediate-layer data vector $\bm{y} \in \mathbb{R}^{N}$, we define the vector obtained by applying the activation function element-wise as
\ali{ 
\bm{f}(\bm{y}) = \mat{f(y_1) \ \cdots \ f(y_N)}^\top \in \mathbb{R}^{N}. \label{eqq_bm_f}
}
Let $f^{(k)}(y)$ denote the $k$-th derivative of $f(y)$. We define the vector composed of these derivatives as
\ali{
\bm{f}^{(k)}(\bm{y}) = \mat{f^{(k)}(y_1)\ \cdots \ f^{(k)}(y_N)}^\top \in \mathbb{R}^{N}, \label{eqq_bm_vecF_k}
}
where $k \in \mathbb{N}$.
We also define the Jacobian matrix with respect to the input $\bm{y}$ as
\ali{
\bm{F}^{(k)}(\bm{y}) = \frac{\partial \bm{f}^{(k-1)}(\bm{y})}{\partial \bm{y}} 
=\mathrm{diag}(\bm{f}^{(k)}(\bm{y}) ), \label{eqq_bm_matF_k}
}
where $\bm{f}^{(0)} = \bm{f}$.
We further define the vector and matrix obtained by prepending a zero element or a zero row to the gradient and Jacobian, respectively, as
\ali{
\bm{f}^{(k)}_0(\bm{y}) = \mat{0  \\ \bm{f}^{(k)}(\bm{y})}, \ 
\bm{F}^{(k)}_0(\bm{y}) = \mat{\bm{0}_{N}^\top \\ \bm{F}^{(k)}(\bm{y})}, \label{eqq_bm_matvecF_k_0}
}
where $\bm{0}_{N}$ denotes the $N$-dimensional zero vector.
In this paper, however, only the cases $k \in \{1,2,3\}$ are considered.
Therefore, the value of $k$ is indicated using prime notation.
Specifically, $f^{(1)} = f^\prime$, $f^{(2)} = f^{\prime\prime}$, and $f^{(3)} = f^{\prime\prime\prime}$. 
Omae et al.~\cite{omaeWolkowiczStyan2026} summarize the first-, second-, and third-order derivatives of several activation functions, including the linear, sigmoid, tanh, SmoothReLU, and GELU activations. 
Readers may refer to that work for details when needed.

\subsection{Kronecker Product}
The Kronecker product appears frequently throughout this paper. In this subsection, we present several identities that will be used in subsequent derivations.

\begin{lemma}
\ali{
\bm{x} \otimes (\bm{A} \bm{B}) &= (\bm{x} \otimes \bm{A}) \bm{B}, \label{eqq_kron_1} \\
\bm{x} \otimes (\bm{A} + \bm{B}) &= \bm{x} \otimes \bm{A} + \bm{x} \otimes \bm{B}, \label{eqq_kron_1_5}  \\
\bm{x} \otimes \mat{\bm{A} & \bm{B}} &= \mat{\bm{x} \otimes \bm{A} & \bm{x} \otimes \bm{B}}. \label{eqq_kron_2}
}
\end{lemma}
\begin{proof}
If $\bm{A}\bm{B}$ is well-defined, 
\ali{
(\bm{x} \otimes \bm{A}) \bm{B} 
= \mat{x_1 \bm{A} \\ \vdots \\ x_M \bm{A}} \bm{B}
= \mat{x_1 \bm{A}\bm{B} \\ \vdots \\ x_M \bm{A}\bm{B}} = \bm{x} \otimes (\bm{A} \bm{B}). \nn
}
If $\bm{A} + \bm{B}$ is well-defined, 
\ali{
\bm{x} \otimes \bm{A} + \bm{x} \otimes \bm{B} 
&= \mat{x_1 \bm{A} \\ \vdots \\ x_M \bm{A}} + \mat{x_1 \bm{B} \\ \vdots \\ x_M \bm{B}} = \mat{x_1 (\bm{A} + \bm{B}) \\ \vdots \\ x_M (\bm{A} + \bm{B})} \nn \\
&= \bm{x} \otimes (\bm{A} + \bm{B}). \nn
}
If $\mat{\bm{A} & \bm{B}}$ is well-defined, 
\ali{
\mat{\bm{x} \otimes \bm{A} & \bm{x} \otimes \bm{B}}
&= \mat{x_1 \bm{A} & x_1 \bm{B} \\
\vdots & \vdots \\
x_M \bm{A} & x_M \bm{B}} = \mat{x_1 \mat{\bm{A} & \bm{B}} \\
\vdots \\
x_M \mat{\bm{A} & \bm{B}}}\nn \\
& = \bm{x} \otimes \mat{\bm{A} & \bm{B}} \nn
}
is obtained.
\end{proof}

\section{Proofs}
\subsection{Proof for Eqs.~\eqref{eqq_main_theo_direction}, 
\eqref{eqq_main_theo_mu}, and \eqref{eqq_main_theo_sigma} } \label{secc_h2_tr_grad}
From Eq.~\eqref{eq_main_theorem}, Eq.~\eqref{eqq_main_theo_direction} follows directly.
Likewise, Eq.~\eqref{eqq_main_theo_mu} follows from Eq.~\eqref{eq_main_theorem_musig}.
From Eq.~\eqref{eq_main_theorem_musig}, the gradient $\pa \sigma(\bm{\theta})/\pa \bm{\theta}$ is given by
\ali{
\frac{\pa \sigma(\bm{\theta})}{\pa \bm{\theta}} &=  
\frac{\pa (\sigma(\bm{\theta})^2)^{1/2}}{\pa \sigma(\bm{\theta})^2} 
\frac{\pa \sigma(\bm{\theta})^2}{\pa \bm{\theta}} \nn \\
&= \frac{1}{2\sigma(\bm{\theta})} \bigg(
\frac{1}{D} \frac{\pa \mathrm{tr} ( \bm{H}_L(\bm{\theta})^2 ) }{\pa \bm{\theta}}
- \frac{\pa \mu(\bm{\theta})^2 }{\pa \bm{\theta}} \bigg)
\nn \\
&= \frac{1}{2\sigma(\bm{\theta})} \bigg(
\frac{1}{D} \frac{\pa \mathrm{tr} ( \bm{H}_L(\bm{\theta})^2 ) }{\pa \bm{\theta}}
- \frac{\pa \mu(\bm{\theta})^2 }{\pa \mu(\bm{\theta})} \frac{\pa \mu(\bm{\theta}) }{\pa \bm{\theta}} \bigg)
\nn \\
&= \frac{1}{2\sigma(\bm{\theta})} \bigg(
\frac{1}{D} \frac{\pa \mathrm{tr} ( \bm{H}_L(\bm{\theta})^2 ) }{\pa \bm{\theta}}
- 2\mu(\bm{\theta})  \frac{\pa \mu(\bm{\theta}) }{\pa \bm{\theta}} \bigg)
\nn \\
&=  
\frac{1}{2 \sigma(\bm{\theta}) D} \frac{\pa \mathrm{tr} (\bm{H}_L(\bm{\theta})^2)}{\pa \bm{\theta}} 
- \frac{\mu(\bm{\theta})}{\sigma(\bm{\theta})} \frac{\pa \mu(\bm{\theta})}{\pa \bm{\theta}}. \nn
}

\subsection{Proof for Eq.~\eqref{eqq_h1_w_div_theorem}} \label{secc_H1_grad_w}
\subsubsection{Gradient Decomposition}
From Eq.~\eqref{eqq_tr_h1_origin}, the gradient of the Hessian trace can be decomposed as
\ali{
&\frac{\pa \mathrm{tr}(\bm{H}_L(\bm{\theta}))}{\pa \bm{\theta}} = 
\sum_{i=1}^I \frac{\pa s^\prime(z_i)(1+\bm{f}(\bm{y}_i)^\top\bm{f}(\bm{y}_i)) }{\pa \bm{\theta}} \nn \\
&+\sum_{i=1}^I (1+\bm{x}_i^\top \bm{x}_i) \bigg( \frac{\pa s^\prime(z_i) \| \bm{F}^{\prime}(\bm{y}_i) \widetilde{\bm{V}}^\top\|^2}{\pa \bm{\theta}} + \frac{\pa \delta_i \widetilde{\bm{V}} \bm{f}^{\prime\prime}(\bm{y}_i)}{\pa \bm{\theta}}  \bigg).
\label{eqq_tr_h1_theta_div}
}
Applying Eq.~\eqref{eqq_ripniz} to the first gradient term, we obtain
\ali{
&\frac{\partial s^\prime(z)(1+\bm{f}(\bm{y})^\top \bm{f}(\bm{y}) )}{\partial \bm{\theta}} \nn\\
&= \frac{\partial s^\prime(z)}{\partial \bm{\theta}} (1+\bm{f}(\bm{y})^\top \bm{f}(\bm{y}) )
+ s^\prime(z) \frac{\partial \bm{f}(\bm{y})^\top \bm{f}(\bm{y})}{\partial \bm{\theta}} \nn \\
&= \frac{\partial s^\prime(z)}{\partial \bm{\theta}} (1+\bm{f}(\bm{y})^\top \bm{f}(\bm{y}) )
+ 2 s^\prime(z) \frac{\partial \bm{f}(\bm{y})}{\partial \bm{\theta}} \bm{f}(\bm{y}). \label{eqq_szinpro_theta_div}
}
Applying Eq.~\eqref{eqq_ripniz} to the second gradient term, we obtain
\ali{
&\frac{\pa s^\prime(z) \| \bm{F}^{\prime}(\bm{y}) \widetilde{\bm{V}}^\top\|^2}{\pa \bm{\theta}}
= \frac{\pa s^\prime(z) (\bm{F}^{\prime}(\bm{y}) \widetilde{\bm{V}}^\top)^\top \bm{F}^{\prime}(\bm{y}) \widetilde{\bm{V}}^\top}{\pa \bm{\theta}}\nn \\
&= \frac{\pa s^\prime(z)}{\pa \bm{\theta}} \widetilde{\bm{V}} \bm{F}^{\prime}(\bm{y})^2 \widetilde{\bm{V}}^\top
+s^\prime(z) \frac{\pa (\bm{F}^{\prime}(\bm{y}) \widetilde{\bm{V}}^\top)^\top \bm{F}^{\prime}(\bm{y}) \widetilde{\bm{V}}^\top}{\pa \bm{\theta}} \nn \\
&= \bigg(\frac{\pa s^\prime(z)}{\pa \bm{\theta}} \widetilde{\bm{V}} \bm{F}^{\prime}(\bm{y})
+ 2 s^\prime(z) \frac{\pa \bm{F}^{\prime}(\bm{y}) \widetilde{\bm{V}}^\top }{\pa \bm{\theta}}\bigg) \bm{F}^{\prime}(\bm{y}) \widetilde{\bm{V}}^\top. \label{eqq_spz_nrm2_theta_div}
}
Since $\widetilde{\bm{V}} \bm{f}^{\prime\prime}(\bm{y}_i)$ is the inner product of $\widetilde{\bm{V}}^\top$ and $\bm{f}^{\prime\prime}(\bm{y}_i)$, Eq.~\eqref{eqq_ripniz} gives
\ali{
&\frac{\pa \delta \widetilde{\bm{V}} \bm{f}^{\prime\prime}(\bm{y})}{\pa \bm{\theta}}
 = \frac{\pa \delta}{\pa \bm{\theta}} \widetilde{\bm{V}} \bm{f}^{\prime\prime}(\bm{y})
+ \delta \frac{\pa \widetilde{\bm{V}} \bm{f}^{\prime\prime}(\bm{y})}{\pa \bm{\theta}} \nn \\
& = \frac{\pa \delta}{\pa \bm{\theta}} \widetilde{\bm{V}} \bm{f}^{\prime\prime}(\bm{y})
+ \delta \frac{\pa \widetilde{\bm{V}}^\top}{\pa \bm{\theta}} \bm{f}^{\prime\prime}(\bm{y})
+ \delta \frac{\pa \bm{f}^{\prime\prime}(\bm{y})}{\pa \bm{\theta}} \widetilde{\bm{V}}^\top \label{eqq_vfpp_y_div}
}
for the third gradient term.

\subsubsection{Individual Gradients}
\begin{lemma}
\ali{
&\frac{\pa s^\prime(z) \| \bm{F}^{\prime}(\bm{y}) \widetilde{\bm{V}}^\top\|^2}{\pa \bm{w}}
= \bm{h}(\bm{x}) 
 \otimes \Big(s^{\prime\prime}(z) \bm{F}^{\prime}(\bm{y}) \widetilde{\bm{V}}^\top \widetilde{\bm{V}} \bm{F}^{\prime}(\bm{y}) \nn \\
 &+ 2 s^\prime(z) \mathrm{diag}(\widetilde{\bm{V}}^\top ) \bm{F}^{\prime\prime}(\bm{y})
 \Big) \bm{F}^{\prime}(\bm{y}) \widetilde{\bm{V}}^\top  \in \mathbb{R}^{(M+1)N}. \label{eqq_spzfpyv_w_div_fin}
}
\end{lemma}

\begin{proof}
Using Eqs.~(51) and (54) in~\cite{omaeWolkowiczStyan2026}, we obtain
\ali{
\frac{\pa z}{\pa \bm{W}_{:m}} = 
\frac{\pa \bm{y}}{\pa \bm{W}_{:m}}
\frac{\pa \bm{r}}{\pa \bm{y}}
\frac{\pa \bm{h}(\bm{r})}{\pa \bm{r}}
\frac{\pa z}{\pa \bm{h}(\bm{r})} = x_m \bm{F}^\prime (\bm{y}) \widetilde{\bm{V}}^\top.
\label{eqq_z_Wm_div}
}
Therefore,
\ali{
\frac{\partial s^\prime(z)}{\partial \bm{W}_{:m}} &= 
\frac{\partial z}{\partial \bm{W}_{:m}}
\frac{\partial s^\prime(z)}{\partial z}
 =
 x_{m} 
 s^{\prime\prime}(z)
 \bm{F}^{\prime}(\bm{y})  
\widetilde{\bm{V}}^\top, \nn \\
\frac{\partial s^\prime(z)}{\partial \bm{w}} 
& =
 \bm{h}(\bm{x}) 
 \otimes (s^{\prime\prime}(z) \bm{F}^{\prime}(\bm{y}) \widetilde{\bm{V}}^\top). \label{eqq_spz_w}
}
From Eq.~(51) in~\cite{omaeWolkowiczStyan2026}, we have
\ali{
\frac{\pa \bm{y}}{\pa \bm{w}} = \bm{h}(\bm{x})\otimes\bm{E}_N. \label{eqq_y_w_div}
}
Therefore, by Eq.~\eqref{eqq_kron_1}, 
\ali{
&\frac{\pa \bm{F}^{\prime}(\bm{y}) \widetilde{\bm{V}}^\top }{\pa \bm{w}}
= \frac{\pa \bm{f}^{\prime}(\bm{y}) \odot \widetilde{\bm{V}}^\top }{\pa \bm{w}}\nn \\
&=\mat{v_1 \frac{\pa f^{\prime}(y_1) }{\pa \bm{w}} & \cdots & v_N\frac{\pa f^{\prime}(y_N) }{\pa \bm{w}}} \nn \\
&=\mat{v_1 \frac{\pa f^{\prime}(y_1) }{\pa y_1} \frac{\pa y_1 }{\pa \bm{w}} & \cdots & v_N \frac{\pa f^{\prime}(y_N) }{\pa y_N} \frac{\pa y_N }{\pa \bm{w}} } \nn\\
&=\mat{v_1 f^{\prime\prime}(y_1) \frac{\pa y_1 }{\pa \bm{w}} & \cdots & v_N f^{\prime\prime}(y_N) \frac{\pa y_N }{\pa \bm{w}} }\nn \\
&=\mat{\frac{\pa y_1 }{\pa \bm{w}} & \cdots & \frac{\pa y_N }{\pa \bm{w}} } 
\mat{
v_1 f^{\prime\prime}(y_1) &\cdots &0 \\
\vdots & \ddots & \vdots \\
0 &\cdots &v_N f^{\prime\prime}(y_N) \\
}\nn \\
&= \frac{\pa \bm{y} }{\pa \bm{w}} \mathrm{diag}(\widetilde{\bm{V}}^\top) \bm{F}^{\prime\prime}(\bm{y})
= \bm{h}(\bm{x})\otimes \mathrm{diag}(\widetilde{\bm{V}}^\top) \bm{F}^{\prime\prime}(\bm{y}). \label{eqq_fyv_w}
}
Substituting Eqs.~\eqref{eqq_spz_w} and \eqref{eqq_fyv_w} into Eq.~\eqref{eqq_spz_nrm2_theta_div}, we obtain
\ali{
&\frac{\pa s^\prime(z) \| \bm{F}^{\prime}(\bm{y}) \widetilde{\bm{V}}^\top\|^2}{\pa \bm{w}}\nn \\
&= \bigg(\frac{\pa s^\prime(z)}{\pa \bm{w}} \widetilde{\bm{V}} \bm{F}^{\prime}(\bm{y})
+ 2 s^\prime(z) \frac{\pa \bm{F}^{\prime}(\bm{y}) \widetilde{\bm{V}}^\top }{\pa \bm{w}}\bigg) \bm{F}^{\prime}(\bm{y}) \widetilde{\bm{V}}^\top \nn \\
&=  \bm{h}(\bm{x}) 
 \otimes \Big(s^{\prime\prime}(z) \bm{F}^{\prime}(\bm{y}) \widetilde{\bm{V}}^\top \widetilde{\bm{V}} \bm{F}^{\prime}(\bm{y}) \nn \\
&+ 2 s^\prime(z) \mathrm{diag}(\widetilde{\bm{V}}^\top ) \bm{F}^{\prime\prime}(\bm{y})
\Big) \bm{F}^{\prime}(\bm{y}) \widetilde{\bm{V}}^\top,
}
where Eq.~\eqref{eqq_kron_1_5} was used.
\end{proof}

\begin{lemma}
\ali{
&\frac{\pa \delta \widetilde{\bm{V}} \bm{f}^{\prime\prime}(\bm{y})}{\pa \bm{w}} = \bm{h}(\bm{x}) \otimes \nn \\
& (s^\prime(z) \bm{F}^\prime (\bm{y}) \widetilde{\bm{V}}^\top \widetilde{\bm{V}} \bm{f}^{\prime\prime}(\bm{y}) + \delta \bm{F}^{\prime\prime\prime}(\bm{y}) \widetilde{\bm{V}}^\top ) \in \mathbb{R}^{(M+1)N }.\label{eqq_dvfpp_w_div_fin}
}
\end{lemma}
\begin{proof}
Combining Eq.~\eqref{eqq_z_Wm_div} with Eq.~(52) in~\cite{omaeWolkowiczStyan2026} yields
\ali{
\frac{\pa \delta}{\pa \bm{W}_{:m}} &=
\frac{\pa p}{\pa \bm{W}_{:m}} 
= \frac{\pa z}{\pa \bm{W}_{:m}} \frac{\pa p}{\pa z} 
= x_m s^\prime(z) \bm{F}^\prime (\bm{y}) \widetilde{\bm{V}}^\top, \nn \\
\frac{\pa \delta}{\pa \bm{w}} &=
\bm{h}(\bm{x}) \otimes (s^\prime(z) \bm{F}^\prime (\bm{y}) \widetilde{\bm{V}}^\top). \label{eqq_delta_Wm_div}
}
Since $\widetilde{\bm{V}}$ does not depend on $\bm{w}$,
\ali{
\frac{\pa \widetilde{\bm{V}}^\top}{\pa \bm{w}} = \bm{0}_{(M+1)N \times N}. \label{eqq_v_w_div}
}
From Eq.~\eqref{eqq_y_w_div} and Eq.~(35) in~\cite{omaeWolkowiczStyan2026}, we obtain
\ali{
\frac{\pa \bm{f}(\bm{y})}{\pa \bm{w}} 
&= \frac{\pa \bm{y}}{\pa \bm{w}} \frac{\pa \bm{f}(\bm{y}) }{\pa \bm{y}}
= \bm{h}(\bm{x}) \otimes \bm{F}^\prime(\bm{y}), \label{eqq_pfy_w} \\
\frac{\pa \bm{f}^\prime(\bm{y}) }{\pa \bm{w}}
&= \frac{\pa \bm{y}}{\pa \bm{w}} \frac{\pa \bm{f}^\prime(\bm{y}) }{\pa \bm{y}}
= \bm{h}(\bm{x})\otimes \bm{F}^{\prime\prime}(\bm{y}), \label{eqq_fpyy_and_yw_div} \\
\frac{\pa \bm{f}^{\prime\prime}(\bm{y}) }{\pa \bm{w}}
&= \frac{\pa \bm{y}}{\pa \bm{w}} \frac{\pa \bm{f}^{\prime\prime}(\bm{y}) }{\pa \bm{y}}
= \bm{h}(\bm{x})\otimes \bm{F}^{\prime\prime\prime}(\bm{y}). \label{eqq_fpp_div}
}
By substituting Eqs.~\eqref{eqq_delta_Wm_div}, \eqref{eqq_v_w_div}, and \eqref{eqq_fpp_div} into Eq.~\eqref{eqq_vfpp_y_div}, we obtain
\ali{
&\frac{\pa \delta \widetilde{\bm{V}} \bm{f}^{\prime\prime}(\bm{y})}{\pa \bm{w}}
= \frac{\pa \delta}{\pa \bm{w}} \widetilde{\bm{V}} \bm{f}^{\prime\prime}(\bm{y})
+ \delta \frac{\pa \widetilde{\bm{V}}^\top}{\pa \bm{w}} \bm{f}^{\prime\prime}(\bm{y})
+ \delta \frac{\pa \bm{f}^{\prime\prime}(\bm{y})}{\pa \bm{w}} \widetilde{\bm{V}}^\top \nn \\
&= \bm{h}(\bm{x}) \otimes (s^\prime(z) \bm{F}^\prime (\bm{y}) \widetilde{\bm{V}}^\top \widetilde{\bm{V}} \bm{f}^{\prime\prime}(\bm{y}) + \delta \bm{F}^{\prime\prime\prime}(\bm{y}) \widetilde{\bm{V}}^\top ), \nn
}
where Eq.~\eqref{eqq_kron_1_5} was used.
\end{proof}

\begin{lemma}
\ali{
&\frac{\partial s^\prime(z)(1+\bm{f}(\bm{y})^\top \bm{f}(\bm{y}) )}{\partial \bm{w}} = \bm{h}(\bm{x}) \otimes \nn \\
&\Big(
s^{\prime\prime}(z) (1+\bm{f}(\bm{y})^\top \bm{f}(\bm{y}) )\bm{F}^{\prime}(\bm{y}) \widetilde{\bm{V}}^\top
+
2 s^\prime(z) \bm{F}^\prime(\bm{y})  \bm{f}(\bm{y}) \Big)\nn \\
& \in \mathbb{R}^{(M+1)N}. \label{eqq_szfyfy_w_div_fin}
}
\end{lemma}
\begin{proof}
Substituting Eqs.~\eqref{eqq_spz_w} and \eqref{eqq_pfy_w} into Eq.~\eqref{eqq_szinpro_theta_div} yields
\ali{
&\frac{\partial s^\prime(z)(1+\bm{f}(\bm{y})^\top \bm{f}(\bm{y}) )}{\partial \bm{w}} \nn \\
&= \frac{\partial s^\prime(z)}{\partial \bm{w}} (1+\bm{f}(\bm{y})^\top \bm{f}(\bm{y}) )
+ 2 s^\prime(z) \frac{\partial \bm{f}(\bm{y})}{\partial \bm{w}} \bm{f}(\bm{y}) \nn \\
&= \bm{h}(\bm{x}) \otimes \Big(
s^{\prime\prime}(z) (1+\bm{f}(\bm{y})^\top \bm{f}(\bm{y}) )\bm{F}^{\prime}(\bm{y}) \widetilde{\bm{V}}^\top \nn \\
&+
2 s^\prime(z) \bm{F}^\prime(\bm{y})  \bm{f}(\bm{y})
\Big), \nn
}
where Eq.~\eqref{eqq_kron_1_5} was used.
\end{proof}

\subsubsection{Completion of the Proof}
From Eqs.~\eqref{eqq_W_i_I}, \eqref{eqq_W_i_II}, \eqref{eqq_W_i_III},
\eqref{eqq_spzfpyv_w_div_fin},
\eqref{eqq_dvfpp_w_div_fin}, and \eqref{eqq_szfyfy_w_div_fin}, 
\ali{
\bm{h}(\bm{x}_i) \otimes \bm{\mathcal{W}}_{i}^\mathrm{I} &=
(1+\bm{x}_i^\top \bm{x}_i) \frac{\pa s^\prime(z_i) \| \bm{F}^{\prime}(\bm{y}_i) \widetilde{\bm{V}}^\top\|^2}{\pa \bm{w}},  \nn\\
\bm{h}(\bm{x}_i) \otimes \bm{\mathcal{W}}_{i}^\mathrm{II} &=
(1+\bm{x}_i^\top \bm{x}_i) \frac{\pa \delta_i \widetilde{\bm{V}} \bm{f}^{\prime\prime}(\bm{y}_i)}{\pa \bm{w}},  \nn\\
\bm{h}(\bm{x}_i) \otimes \bm{\mathcal{W}}_{i}^\mathrm{III} &= 
\frac{\pa s^\prime(z_i)(1+\bm{f}(\bm{y}_i)^\top\bm{f}(\bm{y}_i)) }{\pa \bm{w}}\nn
}
hold.
Therefore, from Eq.~\eqref{eqq_tr_h1_theta_div}, we obtain
\ali{
&\frac{\pa \mathrm{tr}(\bm{H}_L(\bm{\theta}))}{\pa \bm{w}} = 
\sum_{i=1}^I \frac{\pa s^\prime(z_i)(1+\bm{f}(\bm{y}_i)^\top\bm{f}(\bm{y}_i)) }{\pa \bm{w}} \nn \\
&+\sum_{i=1}^I (1+\bm{x}_i^\top \bm{x}_i) \bigg( \frac{\pa s^\prime(z_i) \| \bm{F}^{\prime}(\bm{y}_i) \widetilde{\bm{V}}^\top\|^2}{\pa \bm{w}} + \frac{\pa \delta_i \widetilde{\bm{V}} \bm{f}^{\prime\prime}(\bm{y}_i)}{\pa \bm{w}} \bigg) \nn \\
&= \sum_{i=1}^I \bm{h}(\bm{x}_i) \otimes (\bm{\mathcal{W}}_{i}^\mathrm{I} + \bm{\mathcal{W}}_{i}^\mathrm{II} + \bm{\mathcal{W}}_{i}^\mathrm{III} ). \nn
}

\subsection{Proof for Eq.~\eqref{eqq_h1_v_div_theorem}} \label{secc_H1_grad_v}
\subsubsection{Individual Gradients}
\begin{lemma}
\ali{
&\frac{\pa s^\prime(z) \| \bm{F}^{\prime}(\bm{y}) \widetilde{\bm{V}}^\top\|^2}{\pa \bm{v}}
= \Big(s^{\prime\prime}(z) \bm{h}(\bm{f}(\bm{y})) \widetilde{\bm{V}} \bm{F}^{\prime}(\bm{y}) \nn\\
&+ 2 s^\prime(z) \bm{F}^\prime_0(\bm{y}) \Big)  \bm{F}^{\prime}(\bm{y}) \widetilde{\bm{V}}^\top
 \in \mathbb{R}^{N+1}.
 \label{eqq_spzfpyv_v_div_fin}
}
\end{lemma}
\begin{proof}
Using Eq.~(55) in~\cite{omaeWolkowiczStyan2026}, we obtain
\ali{
\frac{\pa s^\prime(z)}{\pa \bm{v}} = \frac{\pa s^\prime(z)}{\pa z} \frac{\pa z}{\pa \bm{v}} 
= s^{\prime\prime}(z) \bm{h}(\bm{f}(\bm{y})). \label{eqq_szp_v}
}
We also have
\ali{
&\frac{\pa \bm{F}^{\prime}(\bm{y}) \widetilde{\bm{V}}^\top }{\pa \bm{v}}
= 
\frac{\pa \bm{f}^{\prime}(\bm{y}) \odot \widetilde{\bm{V}}^\top }{\pa \bm{v}} \nn\\
&=\mat{f^{\prime}(y_1) \frac{\pa v_1}{\pa \bm{v}} & \cdots & f^{\prime}(y_N)\frac{\pa v_N}{\pa \bm{v}}} \nn \\
&=\mat{\frac{\pa v_1 }{\pa \bm{v}} & \cdots & \frac{\pa v_N }{\pa \bm{v}} } 
\bm{F}^\prime(\bm{y})
= \mat{\bm{0}_N^\top \\ \bm{E}_N} \bm{F}^\prime(\bm{y})
= \bm{F}^\prime_0(\bm{y}).
\label{eqq_ypv_v}
}
Therefore, from Eq.~\eqref{eqq_spz_nrm2_theta_div},
\ali{
&\frac{\pa s^\prime(z) \| \bm{F}^{\prime}(\bm{y}) \widetilde{\bm{V}}^\top\|^2}{\pa \bm{v}}\nn \\
&= \bigg(\frac{\pa s^\prime(z)}{\pa \bm{v}} \widetilde{\bm{V}} \bm{F}^{\prime}(\bm{y})
+ 2 s^\prime(z) \frac{\pa \bm{F}^{\prime}(\bm{y}) \widetilde{\bm{V}}^\top }{\pa \bm{v}}\bigg) \bm{F}^{\prime}(\bm{y}) \widetilde{\bm{V}}^\top \nn \\
&= \Big(s^{\prime\prime}(z) \bm{h}(\bm{f}(\bm{y})) \widetilde{\bm{V}} \bm{F}^{\prime}(\bm{y}) 
+ 
 2 s^\prime(z) \bm{F}^\prime_0(\bm{y}) \Big)  \bm{F}^{\prime}(\bm{y}) \widetilde{\bm{V}}^\top \nn
}
is obtained.
\end{proof}

\begin{lemma}
\ali{
\frac{\pa \delta \widetilde{\bm{V}} \bm{f}^{\prime\prime}(\bm{y})}{\pa \bm{v}}
&= s^\prime(z) \bm{h}(\bm{f}(\bm{y})) \widetilde{\bm{V}} \bm{f}^{\prime\prime}(\bm{y})
+ \delta \bm{f}^{\prime\prime}_0(\bm{y}) \in \mathbb{R}^{N+1}. \label{eqq_dvfpp_v_div_fin}
}
\end{lemma}
\begin{proof}
Since $\bm{f}^{\prime\prime}(\bm{y})$ does not depend on $\bm{v}$, and $\widetilde{\bm{V}}$ contains $\bm{v}$,
\ali{
\frac{\pa \bm{f}^{\prime\prime}(\bm{y})}{\pa \bm{v}} = \bm{0}_{(N+1) \times N}, \ 
\frac{\pa \widetilde{\bm{V}}^\top}{\pa \bm{v}} 
= \mat{\bm{0}_N^\top \\ \bm{E}_N}. \nn
}
Furthermore, from Eq.~(62) in~\cite{omaeWolkowiczStyan2026}, we obtain
\ali{
\frac{\pa \delta}{\pa \bm{v}} = \frac{\pa p}{\pa \bm{v}} = s^\prime(z) \bm{h}(\bm{f}(\bm{y})). \label{eqq_d_v_div}
}
Substituting these results into Eq.~\eqref{eqq_vfpp_y_div}, 
\ali{
\frac{\pa \delta \widetilde{\bm{V}} \bm{f}^{\prime\prime}(\bm{y})}{\pa \bm{v}}
& = \frac{\pa \delta}{\pa \bm{v}} \widetilde{\bm{V}} \bm{f}^{\prime\prime}(\bm{y})
+ \delta \frac{\pa \widetilde{\bm{V}}^\top}{\pa \bm{v}} \bm{f}^{\prime\prime}(\bm{y})
+ \delta \frac{\pa \bm{f}^{\prime\prime}(\bm{y})}{\pa \bm{v}} \widetilde{\bm{V}}^\top \nn \\
&= s^\prime(z) \bm{h}(\bm{f}(\bm{y})) \widetilde{\bm{V}} \bm{f}^{\prime\prime}(\bm{y})
+ \delta \mat{\bm{0}_N^\top \\ \bm{E}_N} \bm{f}^{\prime\prime}(\bm{y}) \nn \\
&= s^\prime(z) \bm{h}(\bm{f}(\bm{y})) \widetilde{\bm{V}} \bm{f}^{\prime\prime}(\bm{y})
+ \delta \bm{f}^{\prime\prime}_0(\bm{y}) \nn 
}
holds.
\end{proof}

\begin{lemma}
\ali{
&\frac{\partial s^\prime(z)(1+\bm{f}(\bm{y})^\top \bm{f}(\bm{y}) )}{\partial \bm{v}} \nn \\
&=s^{\prime\prime}(z) (1+\bm{f}(\bm{y})^\top \bm{f}(\bm{y}) ) \bm{h}(\bm{f}(\bm{y}))
 \in \mathbb{R}^{N+1}.
\label{eqq_szfyfy_v_div_fin}
}
\end{lemma}
\begin{proof}
Since $\bm{f}(\bm{y})$ does not depend on $\bm{v}$,
\ali{
\frac{\pa \bm{f}(\bm{y})}{\pa \bm{v}} = \bm{0}_{(N+1) \times N}. \label{eqq_pfy_v}
}
Substituting Eqs.~\eqref{eqq_pfy_v} and \eqref{eqq_szp_v} into Eq.~\eqref{eqq_szinpro_theta_div}, we obtain Eq.~\eqref{eqq_szfyfy_v_div_fin}.
\end{proof}

\subsubsection{Completion of the Proof}
Substituting Eqs.~\eqref{eqq_spzfpyv_v_div_fin}, \eqref{eqq_dvfpp_v_div_fin}, and \eqref{eqq_szfyfy_v_div_fin} into Eq.~\eqref{eqq_tr_h1_theta_div}, we obtain
\ali{
&\frac{\pa \mathrm{tr}(\bm{H}_L(\bm{\theta}))}{\pa \bm{v}} = 
\sum_{i=1}^I \frac{\pa s^\prime(z_i)(1+\bm{f}(\bm{y}_i)^\top\bm{f}(\bm{y}_i)) }{\pa \bm{v}}\nn \\
&+\sum_{i=1}^I (1+\bm{x}_i^\top \bm{x}_i) \bigg( \frac{\pa s^\prime(z_i) \| \bm{F}^{\prime}(\bm{y}_i) \widetilde{\bm{V}}^\top\|^2}{\pa \bm{v}} + \frac{\pa \delta \widetilde{\bm{V}} \bm{f}^{\prime\prime}(\bm{y}_i)}{\pa \bm{v}} \bigg) \nn \\
&= \sum_{i=1}^I (\bm{\mathcal{V}}_{i}^\mathrm{I} + \bm{\mathcal{V}}_{i}^\mathrm{II} + \bm{\mathcal{V}}_{i}^\mathrm{III} ), \nn
}
where Eqs.~\eqref{eqq_V_i_I}, \eqref{eqq_V_i_II}, and \eqref{eqq_V_i_III} were used.

\subsection{Proof for Eq.~\eqref{eqq_H2_w_div_fin}} \label{secc_H2_grad_w}
\subsubsection{Gradient Decomposition}
From Eq.~\eqref{eqq_tr_h2_origin}, 
\ali{
\frac{\pa \mathrm{tr}(\bm{H}_L(\bm{\theta})^2)}{\pa \bm{\theta}}
&= \sum_{i=1}^I \sum_{j=1}^I (1+ \bm{x}_i^\top\bm{x}_j)^2 \frac{\pa \phi_{ij} }{\pa \bm{\theta}}\nn \\
& + 2\sum_{i=1}^I \sum_{j=1}^I (1+ \bm{x}_i^\top\bm{x}_j) \frac{\pa \psi_{ij} }{\pa \bm{\theta}}\nn \\
& +
\sum_{i=1}^I \sum_{j=1}^I (1+ \bm{f}(\bm{y}_i)^\top\bm{f}(\bm{y}_j))^2 \frac{\pa \omega_{ij} }{\pa \bm{\theta}} \nn \\
&+\sum_{i=1}^I \sum_{j=1}^I \frac{\pa (1+ \bm{f}(\bm{y}_i)^\top\bm{f}(\bm{y}_j))^2 }{\pa \bm{\theta}} \omega_{ij} \label{eqq_H2_theta_div}
}
holds.
We decompose it into three parts
\ali{
&\Bigg[ \frac{\pa \mathrm{tr}(\bm{H}_L(\bm{\theta})^2)}{\pa \bm{\theta}^\dag} \Bigg]_\Phi
= \sum_{i=1}^I \sum_{j=1}^I (1+ \bm{x}_i^\top\bm{x}_j)^2 \frac{\pa \phi_{ij} }{\pa \bm{\theta}^\dag}, \label{eqq_phi_tr2_div}\\
&\Bigg[ \frac{\pa \mathrm{tr}(\bm{H}_L(\bm{\theta})^2)}{\pa \bm{\theta}^\dag} \Bigg]_\Psi
= 2\sum_{i=1}^I \sum_{j=1}^I (1+ \bm{x}_i^\top\bm{x}_j) \frac{\pa \psi_{ij} }{\pa \bm{\theta}^\dag}, \label{eqq_psi_tr2_div}\\
&\Bigg[ \frac{\pa \mathrm{tr}(\bm{H}_L(\bm{\theta})^2)}{\pa \bm{\theta}^\dag} \Bigg]_\Omega
= \sum_{i=1}^I \sum_{j=1}^I (1+ \bm{f}(\bm{y}_i)^\top\bm{f}(\bm{y}_j))^2 \frac{\pa \omega_{ij} }{\pa \bm{\theta}^\dag} \nn \\
&+\sum_{i=1}^I \sum_{j=1}^I \frac{\pa (1+ \bm{f}(\bm{y}_i)^\top\bm{f}(\bm{y}_j))^2 }{\pa \bm{\theta}^\dag} \omega_{ij}, \ 
\bm{\theta}^\dag \in \{\bm{w}, \bm{v} \}. \label{eqq_omega_tr2_div}
}

\begin{lemma}
\ali{
\frac{\pa \phi_{ij}}{\pa \bm{\theta}} &= \frac{\pa \bm{\phi}_{ij}}{\pa \bm{\theta}}(\bm{o}_i \otimes \bm{o}_j)
+ \frac{\pa \bm{o}_i}{\pa \bm{\theta}} \bm{\Phi}_{ij} \bm{o}_j
+ \frac{\pa \bm{o}_{j} }{\pa \bm{\theta}} \bm{\Phi}_{ji} \bm{o}_{i}, \label{eqq_phi_theta_dif}\\
\frac{\pa \psi_{ij} }{\pa \bm{\theta}} &= 
\frac{\pa \bm{\psi}_{ij}}{\pa \bm{\theta}}(\bm{o}_i \otimes \bm{o}_j)
+ \frac{\pa \bm{o}_i}{\pa \bm{\theta}} \bm{\Psi}_{ij} \bm{o}_j
+ \frac{\pa \bm{o}_{j} }{\pa \bm{\theta}} \bm{\Psi}_{ji} \bm{o}_{i}, \label{eqq_psi_theta_dif} \\
&\frac{\pa \phi_{ij}}{\pa \bm{\theta}}, \frac{\pa \psi_{ij} }{\pa \bm{\theta}} \in \mathbb{R}^{D}. \nn
}
\end{lemma}
\begin{proof}
Since quadratic forms are equivalent to inner products, from Eq.~\eqref{eqq_ripniz} we have
\ali{
\frac{\pa \phi_{ij}}{\pa \bm{\theta}} 
&= \frac{\pa \bm{o}_{i}^\top \bm{\Phi}_{ij} \bm{o}_{j}}{\pa \bm{\theta}}
= \bigg(\frac{\pa (\bm{o}_{i}^\top \bm{\Phi}_{ij})^\top}{\pa \bm{\theta}} \bigg) \bm{o}_{j}
+ \bigg(\frac{\pa \bm{o}_{j} }{\pa \bm{\theta}} \bigg)(\bm{o}_{i}^\top \bm{\Phi}_{ij})^\top \nn \\
&= \bigg(\frac{\pa \bm{\Phi}_{ij}^\top \bm{o}_{i}}{\pa \bm{\theta}} \bigg) \bm{o}_{j}
+ \bigg(\frac{\pa \bm{o}_{j} }{\pa \bm{\theta}} \bigg)\bm{\Phi}_{ij}^\top \bm{o}_{i} \nn \\
&= \frac{\pa \bm{\phi}_{ij}}{\pa \bm{\theta}}(\bm{o}_i \otimes \bm{o}_j)
+ \frac{\pa \bm{o}_i}{\pa \bm{\theta}} \bm{\Phi}_{ij} \bm{o}_j
+ \frac{\pa \bm{o}_{j} }{\pa \bm{\theta}} \bm{\Phi}_{ij}^\top \bm{o}_{i} \nn \\
&= \frac{\pa \bm{\phi}_{ij}}{\pa \bm{\theta}}(\bm{o}_i \otimes \bm{o}_j)
+ \frac{\pa \bm{o}_i}{\pa \bm{\theta}} \bm{\Phi}_{ij} \bm{o}_j
+ \frac{\pa \bm{o}_{j} }{\pa \bm{\theta}} \bm{\Phi}_{ji} \bm{o}_{i}. \nn
}
Using Eqs.~\eqref{eqq_Phi_11}--\eqref{eqq_Phi_22}, we obtain
\ali{
\bm{\Phi}_{ij}^\top = \bm{\Phi}_{ji} \nn
}
and we also use
\ali{
\frac{\pa \bm{\Phi}_{ij}^\top \bm{o}_{i}}{\pa \bm{\theta}}\bm{o}_{j} &=
\frac{\pa \mat{(\bm{\Phi}_{ij}^\top)_{1:} \bm{o}_{i} \\ (\bm{\Phi}_{ij}^\top)_{2:} \bm{o}_{i}}}{\pa \bm{\theta}}\bm{o}_{j} \nn \\
& =\mat{
\frac{\pa (\bm{\Phi}_{ij}^\top)_{1:} \bm{o}_{i}}{\pa \bm{\theta}} & 
\frac{\pa (\bm{\Phi}_{ij}^\top)_{2:} \bm{o}_{i}}{\pa \bm{\theta}}
}
\mat{s^\prime(z_j) \\ \delta_j}\nn \\
&=
s^\prime(z_j) \frac{\pa (\bm{\Phi}_{ij}^\top)_{1:} \bm{o}_{i}}{\pa \bm{\theta}} +
\delta_j \frac{\pa (\bm{\Phi}_{ij}^\top)_{2:} \bm{o}_{i}}{\pa \bm{\theta}}\nn \\
&= 
s^\prime(z_j) \bigg( \frac{\pa (\bm{\Phi}_{ij})_{11} s^\prime(z_i)}{\pa \bm{\theta}} + \frac{\pa (\bm{\Phi}_{ij})_{21} \delta_i}{\pa \bm{\theta}} \bigg) \nn \\
&+
\delta_j \bigg( \frac{\pa (\bm{\Phi}_{ij})_{12} s^\prime(z_i)}{\pa \bm{\theta}} + \frac{\pa (\bm{\Phi}_{ij})_{22} \delta_i}{\pa \bm{\theta}} \bigg)
\nn \\
&= 
s^\prime(z_j) s^\prime(z_i) \frac{\pa (\bm{\Phi}_{ij})_{11}}{\pa \bm{\theta}} 
+ \delta_j s^\prime(z_i) \frac{\pa (\bm{\Phi}_{ij})_{12}}{\pa \bm{\theta}}  \nn \\
&+ s^\prime(z_j) \delta_i \frac{\pa (\bm{\Phi}_{ij})_{21}}{\pa \bm{\theta}} 
+ \delta_j \delta_i \frac{\pa (\bm{\Phi}_{ij})_{22}}{\pa \bm{\theta}} \nn \\
&+ s^\prime(z_j) (\bm{\Phi}_{ij})_{11} \frac{\pa s^\prime(z_i)}{\pa \bm{\theta}} 
+ s^\prime(z_j) (\bm{\Phi}_{ij})_{21} \frac{\pa \delta_i}{\pa \bm{\theta}} \nn \\
& + \delta_j (\bm{\Phi}_{ij})_{12} \frac{\pa s^\prime(z_i)}{\pa \bm{\theta}} 
+ \delta_j (\bm{\Phi}_{ij})_{22} \frac{\pa \delta_i}{\pa \bm{\theta}}\nn \\
&= 
\frac{\pa \bm{\phi}_{ij}}{\pa \bm{\theta}}(\bm{o}_i \otimes \bm{o}_j)
+ \frac{\pa \bm{o}_i}{\pa \bm{\theta}} \bm{\Phi}_{ij} \bm{o}_j. \nn
}
From these results, $\pa \psi_{ij} / \pa \bm{\theta}$ can be proved in the same manner.
Note that $\bm{\phi}_{ij}$ is the vector obtained by flattening $\bm{\Phi}_{ij}$, 
and $\bm{\psi}_{ij}$ is the vector obtained by flattening $\bm{\Psi}_{ij}$, respectively, i.e.,
\ali{
\bm{\phi}_{ij} &= \mat{
(\bm{\Phi}_{ij})_{11}&
(\bm{\Phi}_{ij})_{12}&
(\bm{\Phi}_{ij})_{21}&
(\bm{\Phi}_{ij})_{22}}^\top, \nn \\
\bm{\psi}_{ij} &= \mat{
(\bm{\Psi}_{ij})_{11}&
(\bm{\Psi}_{ij})_{12}&
(\bm{\Psi}_{ij})_{21}&
(\bm{\Psi}_{ij})_{22}}^\top. \nn
}
Also, the Jacobians are given by
\ali{
\frac{\pa \bm{\phi}_{ij}}{\pa \bm{\theta}} &= \mat{
\frac{\pa (\bm{\Phi}_{ij})_{11}}{\pa \bm{\theta}} &
\frac{\pa (\bm{\Phi}_{ij})_{12}}{\pa \bm{\theta}} &
\frac{\pa (\bm{\Phi}_{ij})_{21}}{\pa \bm{\theta}} &
\frac{\pa (\bm{\Phi}_{ij})_{22}}{\pa \bm{\theta}}}, \label{eqq_jac_phi_theta}\\
\frac{\pa \bm{\psi}_{ij}}{\pa \bm{\theta}} &= \mat{
\frac{\pa (\bm{\Psi}_{ij})_{11}}{\pa \bm{\theta}} &
\frac{\pa (\bm{\Psi}_{ij})_{12}}{\pa \bm{\theta}} &
\frac{\pa (\bm{\Psi}_{ij})_{21}}{\pa \bm{\theta}} &
\frac{\pa (\bm{\Psi}_{ij})_{22}}{\pa \bm{\theta}}}.  \label{eqq_jac_psi_theta}
}
\end{proof}

\begin{lemma}
\ali{
\frac{\pa \omega_{ij} }{\pa \bm{\theta}}
= \frac{\pa s^\prime(z_i)  }{\pa \bm{\theta}} s^\prime(z_j) 
+ \frac{\pa s^\prime(z_j)  }{\pa \bm{\theta}} s^\prime(z_i) \in \mathbb{R}^{D \times 1}.
}
\end{lemma}
\begin{proof}
From Eq.~\eqref{eq_ome_ij}, 
\ali{
\frac{\pa \omega_{ij} }{\pa \bm{\theta}}
&= \frac{\pa \bm{o}_{i}^\top \bm{\Omega}_{ij} \bm{o}_{j} }{\pa \bm{\theta}}
= \frac{\pa s^\prime(z_i)  s^\prime(z_j) }{\pa \bm{\theta}}\nn \\
&= \frac{\pa s^\prime(z_i)  }{\pa \bm{\theta}} s^\prime(z_j) 
+ \frac{\pa s^\prime(z_j)  }{\pa \bm{\theta}} s^\prime(z_i) \nn
}
is obtained.
\end{proof}

\subsubsection{Gradient with Respect to Affine Parameters from Input to Hidden Layer}
$\pa \phi_{ij}/\pa \bm{\theta}$ contains $\pa \bm{o} / \pa \bm{\theta}$.
Therefore, we show the Jacobian of $\bm{o}$ with respect to $\bm{w}$ here.
\begin{lemma}
\ali{
\frac{\pa \bm{o}}{\pa \bm{w}} 
= \bm{h}(\bm{x}) \otimes (\bm{F}^\prime (\bm{y}) \widetilde{\bm{V}}^\top \bm{s}^{\prime\prime/\prime}(z)^\top)
\in \mathbb{R}^{(M+1)N \times 2}. \label{eqq_o_w_dif}
}
\end{lemma}
\begin{proof}
From Eqs.~\eqref{eqq_spp_sp}, \eqref{eqq_kron_1}, \eqref{eqq_kron_2}, \eqref{eqq_spz_w}, and \eqref{eqq_delta_Wm_div},
\ali{
&\frac{\pa \bm{o}}{\pa \bm{w}} = 
\mat{\frac{\pa s^\prime(z)}{\pa \bm{w}} & \frac{\pa \delta}{\pa \bm{w}} } \nn \\
&= \mat{ 
\bm{h}(\bm{x}) \otimes (s^{\prime\prime}(z) \bm{F}^{\prime}(\bm{y}) \widetilde{\bm{V}}^\top)& 
\bm{h}(\bm{x}) \otimes (s^\prime(z) \bm{F}^\prime (\bm{y}) \widetilde{\bm{V}}^\top) 
} \nn \\
&= \bm{h}(\bm{x}) \otimes \mat{ 
 s^{\prime\prime}(z) \bm{F}^{\prime}(\bm{y}) \widetilde{\bm{V}}^\top& 
 s^\prime(z) \bm{F}^\prime (\bm{y}) \widetilde{\bm{V}}^\top
} \nn \\
&= \bm{h}(\bm{x}) \otimes \bm{F}^{\prime}(\bm{y}) \widetilde{\bm{V}}^\top 
\mat{ s^{\prime\prime}(z) & s^\prime(z) } \nn \\
&= \bm{h}(\bm{x}) \otimes (\bm{F}^\prime (\bm{y}) \widetilde{\bm{V}}^\top \bm{s}^{\prime\prime/\prime}(z)^\top) \nn
}
is obtained.
\end{proof}

\subsubsection{Gradients of the Phi Term}
Here, we clarify
$
[ \pa \mathrm{tr}(\bm{H}_L(\bm{\theta})^2) / \pa \bm{w} ]_\Phi
$
shown in Eq.~\eqref{eqq_phi_tr2_div}.

\begin{lemma}
\ali{
\frac{\pa (\bm{\Phi}_{ij})_{11}}{\pa \bm{w}}
&= \bm{h}(\bm{x}_{i}) \otimes \bm{\mathfrak{a}}^\Phi_{ij} + \bm{h}(\bm{x}_{j}) \otimes \bm{\mathfrak{a}}^\Phi_{ji}
 \in \mathbb{R}^{(M+1)N}.
\label{eqq_Phi_11_w_div}
}
\end{lemma}
\begin{proof}
From $y_n = \sum_{m=1}^M w_{nm}x_m + b_n$, $w_{nm}$ appears only in $y_{n}$.
Also, from Eq.~\eqref{eqq_Phi_11},
\ali{
&\frac{\pa (\bm{\Phi}_{ij})_{11}^{1/2}}{\pa w_{nm}}
= v_n^2 \frac{\pa f^\prime(y_{jn}) f^\prime(y_{in})}{\pa w_{nm}}\nn \\
&= v_n^2 \bigg( \frac{\pa f^\prime(y_{jn})}{\pa w_{nm}} f^\prime(y_{in}) 
+ \frac{\pa f^\prime(y_{in})}{\pa w_{nm}} f^\prime(y_{jn})\bigg) \nn \\
&= v_n^2 \bigg( 
\frac{\pa f^\prime(y_{jn})}{\pa y_{jn}} \frac{\pa y_{jn}}{\pa w_{nm}} f^\prime(y_{in}) 
+ \frac{\pa f^\prime(y_{in})}{\pa y_{in}} \frac{\pa y_{in}}{\pa w_{nm}} f^\prime(y_{jn})
\bigg)\nn \\
&= v_n^2 \bigg( 
f^{\prime\prime}(y_{jn}) x_{jm} f^\prime(y_{in}) 
+ f^{\prime\prime}(y_{in}) x_{im} f^\prime(y_{jn})
\bigg) \nn
}
holds.
Therefore, 
\ali{
&\frac{\pa (\bm{\Phi}_{ij})_{11}^{1/2}}{\pa \bm{W}_{:m}} \nn \\
&= \widetilde{\bm{V}}^{\odot 2 \top} \odot \bigg( 
x_{im} \bm{f}^{\prime\prime}(\bm{y}_{i}) \odot \bm{f}^\prime(\bm{y}_{j}) 
+ 
x_{jm} \bm{f}^{\prime\prime}(\bm{y}_{j}) \odot \bm{f}^\prime(\bm{y}_{i}) \bigg) \nn \\
&= 
x_{im} \bm{F}^{\prime\prime}(\bm{y}_{i}) \bm{F}^\prime(\bm{y}_{j}) \widetilde{\bm{V}}^{\odot 2 \top}
+ 
x_{jm}  \bm{F}^{\prime\prime}(\bm{y}_{j}) \bm{F}^\prime(\bm{y}_{i}) \widetilde{\bm{V}}^{\odot 2 \top} \nn
}
is obtained, where $\widetilde{\bm{V}}^{\odot 2 \top} = \widetilde{\bm{V}}^\top \odot \widetilde{\bm{V}}^\top$.
Using this, we obtain
\ali{
\frac{\pa (\bm{\Phi}_{ij})_{11}^{1/2}}{\pa \bm{w}} &= 
\bm{h}(\bm{x}_{i}) \otimes ( \bm{F}^{\prime\prime}(\bm{y}_{i}) \bm{F}^\prime(\bm{y}_{j}) \widetilde{\bm{V}}^{\odot 2 \top} )\nn \\
&+
\bm{h}(\bm{x}_{j}) \otimes (\bm{F}^{\prime\prime}(\bm{y}_{j}) \bm{F}^\prime(\bm{y}_{i}) \widetilde{\bm{V}}^{\odot 2 \top} ).\nn 
}
Therefore, using Eq.~\eqref{eqq_a_phi_ij}, we obtain
\ali{
\frac{\pa (\bm{\Phi}_{ij})_{11}}{\pa \bm{w}} 
&= 
\frac{\pa (\bm{\Phi}_{ij})_{11}}{\pa (\bm{\Phi}_{ij})_{11}^{1/2}} 
\frac{\pa (\bm{\Phi}_{ij})_{11}^{1/2}}{\pa \bm{w}} \nn \\
&= 
2 (\bm{\Phi}_{ij})_{11}^{1/2} \Big(
\bm{h}(\bm{x}_{i}) \otimes ( \bm{F}^{\prime\prime}(\bm{y}_{i}) \bm{F}^\prime(\bm{y}_{j}) \widetilde{\bm{V}}^{\odot 2 \top} )\nn \\
&+
\bm{h}(\bm{x}_{j}) \otimes (\bm{F}^{\prime\prime}(\bm{y}_{j}) \bm{F}^\prime(\bm{y}_{i}) \widetilde{\bm{V}}^{\odot 2 \top} ) \Big) \nn \\
&=
\bm{h}(\bm{x}_{i}) \otimes ( 2 (\bm{\Phi}_{ij})_{11}^{1/2} \bm{F}^{\prime\prime}(\bm{y}_{i}) \bm{F}^\prime(\bm{y}_{j}) \widetilde{\bm{V}}^{\odot 2 \top} )\nn \\
&+\bm{h}(\bm{x}_{j}) \otimes ( 2 (\bm{\Phi}_{ji})_{11}^{1/2} \bm{F}^{\prime\prime}(\bm{y}_{j}) \bm{F}^\prime(\bm{y}_{i}) \widetilde{\bm{V}}^{\odot 2 \top} ) 
\nn \\
&= \bm{h}(\bm{x}_{i}) \otimes \bm{\mathfrak{a}}^\Phi_{ij} + \bm{h}(\bm{x}_{j}) \otimes \bm{\mathfrak{a}}^\Phi_{ji}. \nn
}
Note that we utilized $(\bm{\Phi}_{ij})_{11} = (\bm{\Phi}_{ji})_{11}$ from Eq.~\eqref{eqq_Phi_11}.
\end{proof}

\begin{lemma}
\ali{
&\frac{\pa (\bm{\Phi}_{ij})_{12}}{\pa \bm{w}}
= \bm{h}(\bm{x}_{i}) \otimes \bm{\mathfrak{b}}^\Phi_{ij} + \bm{h}(\bm{x}_{j}) \otimes \bm{\mathfrak{c}}^\Phi_{ji}\in \mathbb{R}^{(M+1)N},
\label{eqq_Phi_12_w_div} \\
&\frac{\pa (\bm{\Phi}_{ij})_{21}}{\pa \bm{w}}
= \bm{h}(\bm{x}_{i}) \otimes \bm{\mathfrak{c}}^\Phi_{ij} + \bm{h}(\bm{x}_{j}) \otimes \bm{\mathfrak{b}}^\Phi_{ji} \in \mathbb{R}^{(M+1)N}.
\label{eqq_Phi_21_w_div}
}
\end{lemma}
\begin{proof}
From $y_n = \sum_{m=1}^M w_{nm}x_m + b_n$, $w_{nm}$ appears only in $y_{n}$.
Also, from Eq.~\eqref{eqq_Phi_12},
\ali{
&\frac{\pa (\bm{\Phi}_{ij})_{12}}{\pa w_{nm}}
= v_n^3 \frac{\pa f^\prime(y_{in})^2 f^{\prime\prime}(y_{jn})}{\pa w_{nm}}\nn \\
&= v_n^3 \bigg( \frac{\pa f^\prime(y_{in})^2}{\pa w_{nm}} f^{\prime\prime}(y_{jn}) + 
\frac{\pa f^{\prime\prime}(y_{jn})}{\pa w_{nm}} f^\prime(y_{in})^2 \bigg)\nn \\
&=  
2 x_{im} v_n^3 f^\prime(y_{in}) f^{\prime\prime}(y_{in}) f^{\prime\prime}(y_{jn}) + 
x_{jm} v_n^3 f^{\prime\prime\prime}(y_{jn}) f^\prime(y_{in})^2, \nn
}
where
\ali{
\frac{\pa f^\prime(y_{in})^2}{\pa w_{nm}} 
&= \frac{\pa f^\prime(y_{in})^2}{\pa f^\prime(y_{in})}
\frac{\pa f^\prime(y_{in})}{\pa y_{in}}
\frac{\pa y_{in}}{\pa w_{nm}}, \nn \\
\frac{\pa f^{\prime\prime}(y_{jn})}{\pa w_{nm}} 
&=\frac{\pa f^{\prime\prime}(y_{jn})}{\pa y_{jn}} 
\frac{\pa y_{jn}}{\pa w_{nm}} \nn
}
are utilized.
Thus, we have
\ali{
\frac{\pa (\bm{\Phi}_{ij})_{12}}{\pa \bm{W}_{:m}}
&= 2 x_{im} \widetilde{\bm{V}}^{\odot 3 \top} \odot
\bm{f}^\prime(\bm{y}_{i}) \odot
\bm{f}^{\prime\prime}(\bm{y}_{i}) \odot
\bm{f}^{\prime\prime}(\bm{y}_{j}) \nn \\
&+ 
x_{jm}  \widetilde{\bm{V}}^{\odot 3 \top} \odot \bm{f}^{\prime\prime\prime}(\bm{y}_{j}) \odot \bm{f}^\prime(\bm{y}_{i})^{\odot 2}\nn \\
&= 2 x_{im} 
\bm{F}^\prime(\bm{y}_{i})
\bm{F}^{\prime\prime}(\bm{y}_{i}) 
\bm{F}^{\prime\prime}(\bm{y}_{j}) 
\widetilde{\bm{V}}^{\odot 3 \top}  \nn \\
&+ x_{jm} 
\bm{F}^{\prime\prime\prime}(\bm{y}_{j})
\bm{F}^\prime(\bm{y}_{i})^{2}
\widetilde{\bm{V}}^{\odot 3 \top} \nn
}
and using Eqs.~\eqref{eqq_b_phi_ij} and \eqref{eqq_c_phi_ij} yields
\ali{
\frac{\pa (\bm{\Phi}_{ij})_{12}}{\pa \bm{w}}
&= \bm{h}(\bm{x}_{i}) \otimes \big( 2
\bm{F}^\prime(\bm{y}_{i})
\bm{F}^{\prime\prime}(\bm{y}_{i}) 
\bm{F}^{\prime\prime}(\bm{y}_{j}) 
\widetilde{\bm{V}}^{\odot 3 \top}
\big)\nn \\
& + 
\bm{h}(\bm{x}_{j}) \otimes \big( 
\bm{F}^{\prime\prime\prime}(\bm{y}_{j})
\bm{F}^\prime(\bm{y}_{i})^{2}
\widetilde{\bm{V}}^{\odot 3 \top}
\big)\nn \\
&= \bm{h}(\bm{x}_{i}) \otimes \bm{\mathfrak{b}}^\Phi_{ij} + \bm{h}(\bm{x}_{j}) \otimes \bm{\mathfrak{c}}^\Phi_{ji}. \nn
}
Also, since $(\bm{\Phi}_{ij})_{21} = (\bm{\Phi}_{ji})_{12}$ from Eqs.~\eqref{eqq_Phi_12} and \eqref{eqq_Phi_21}, 
\ali{
\frac{\pa (\bm{\Phi}_{ij})_{21}}{\pa \bm{w}}
= \frac{\pa (\bm{\Phi}_{ji})_{12}}{\pa \bm{w}}
= \bm{h}(\bm{x}_{i}) \otimes \bm{\mathfrak{c}}^\Phi_{ij} + \bm{h}(\bm{x}_{j}) \otimes \bm{\mathfrak{b}}^\Phi_{ji} \nn
}
is obtained.
\end{proof}

\begin{lemma}
\ali{
\frac{\pa (\bm{\Phi}_{ij})_{22}}{\pa \bm{w}}
&= \bm{h}(\bm{x}_{i}) \otimes \bm{\mathfrak{d}}^\Phi_{ij} + \bm{h}(\bm{x}_{j}) \otimes \bm{\mathfrak{d}}^\Phi_{ji} 
 \in \mathbb{R}^{(M+1)N}.
\label{eqq_Phi_22_w_div}
}
\end{lemma}
\begin{proof}
From $y_n = \sum_{m=1}^M w_{nm}x_m + b_n$, $w_{nm}$ appears only in $y_{n}$.
Also, from Eq.~\eqref{eqq_Phi_22}, since
\ali{
&\frac{\pa (\bm{\Phi}_{ij})_{22}}{\pa w_{nm}}
=  v_n^2 \frac{\pa f^{\prime\prime}(y_{in}) f^{\prime\prime}(y_{jn})}{\pa w_{nm}}\nn \\
&= v_n^2 \bigg( \frac{\pa f^{\prime\prime}(y_{in})}{\pa w_{nm}} f^{\prime\prime}(y_{jn}) + 
\frac{\pa f^{\prime\prime}(y_{jn})}{\pa w_{nm}} f^{\prime\prime}(y_{in})  \bigg) \nn \\
&= v_n^2 \bigg( 
\frac{\pa f^{\prime\prime}(y_{in})}{\pa y_{in}}
\frac{\pa y_{in}}{\pa w_{nm}} f^{\prime\prime}(y_{jn}) + 
\frac{\pa f^{\prime\prime}(y_{jn})}{\pa y_{jn}}
\frac{\pa y_{jn}}{\pa w_{nm}}
f^{\prime\prime}(y_{in}) 
 \bigg)\nn \\
&= v_n^2 \big( 
x_{im}
f^{\prime\prime\prime}(y_{in})
f^{\prime\prime}(y_{jn}) + 
x_{jm}
f^{\prime\prime\prime}(y_{jn})
f^{\prime\prime}(y_{in}) 
 \big), \nn
}
we have
\ali{
\frac{\pa (\bm{\Phi}_{ij})_{22}}{\pa \bm{W}_{:m}}
&=
x_{im}
\bm{F}^{\prime\prime\prime}(\bm{y}_{i})
\bm{F}^{\prime\prime}(\bm{y}_{j})
\widetilde{\bm{V}}^{\odot 2 \top} \nn \\
&+ x_{jm}
\bm{F}^{\prime\prime\prime}(\bm{y}_{j})
\bm{F}^{\prime\prime}(\bm{y}_{i})
\widetilde{\bm{V}}^{\odot 2 \top} \nn
}
and using Eq.~\eqref{eqq_d_phi_ij}, 
\ali{
\frac{\pa (\bm{\Phi}_{ij})_{22}}{\pa \bm{w}}
&= 
\bm{h}(\bm{x}_{i}) \otimes
\big(
\bm{F}^{\prime\prime\prime}(\bm{y}_{i})
\bm{F}^{\prime\prime}(\bm{y}_{j})
\widetilde{\bm{V}}^{\odot 2 \top}
\big) \nn \\
&+ 
\bm{h}(\bm{x}_{j}) \otimes
\big(
\bm{F}^{\prime\prime\prime}(\bm{y}_{j})
\bm{F}^{\prime\prime}(\bm{y}_{i})
\widetilde{\bm{V}}^{\odot 2 \top}
\big)\nn \\
&= \bm{h}(\bm{x}_{i}) \otimes \bm{\mathfrak{d}}^\Phi_{ij} + \bm{h}(\bm{x}_{j}) \otimes \bm{\mathfrak{d}}^\Phi_{ji} \nn
}
is obtained.
\end{proof}

\begin{lemma}
\ali{
&\Bigg[ \frac{\pa \mathrm{tr}(\bm{H}_L(\bm{\theta})^2)}{\pa \bm{w}} \Bigg]_\Phi
= \sum_{i=1}^I \sum_{j=1}^I\nn \\
&\bigg( 
\bm{h}(\bm{x}_i) \otimes (\bm{\mathcal{A}}_{ij}^\Phi + \bm{\mathcal{B}}_{ij}^\Phi)  + \bm{h}(\bm{x}_j) \otimes (\bm{\mathcal{A}}_{ji}^\Phi + \bm{\mathcal{B}}_{ji}^\Phi) \bigg). \label{eqq_w_Phi_fin}
}
\end{lemma}

\begin{proof}
Since the Kronecker product of $\bm{o}_i$ and $\bm{o}_j$ is given by
\ali{
\bm{o}_i \otimes \bm{o}_j = \mat{
s^\prime(z_i) s^\prime(z_j)&
s^\prime(z_i) \delta_j&
\delta_i s^\prime(z_j)&
\delta_i \delta_j}^\top, \label{eqq_oi_oj}
}
we have
\ali{
&\mat{
\bm{\mathfrak{a}}^\Phi_{ji} &
\bm{\mathfrak{c}}^\Phi_{ji} & 
\bm{\mathfrak{b}}^\Phi_{ji} & 
\bm{\mathfrak{d}}^\Phi_{ji}  
}(\bm{o}_i \otimes \bm{o}_j) \nn \\
&=
\mat{
\bm{\mathfrak{a}}^\Phi_{ji} &
\bm{\mathfrak{b}}^\Phi_{ji} & 
\bm{\mathfrak{c}}^\Phi_{ji} & 
\bm{\mathfrak{d}}^\Phi_{ji}  
}(\bm{o}_j \otimes \bm{o}_i) \in \mathbb{R}^{N}. \label{eqq_cb_change}
}

From Eqs.~\eqref{eqq_G_phi_ij}, 
\eqref{eqq_jac_phi_theta}, 
\eqref{eqq_Phi_11_w_div}, 
\eqref{eqq_Phi_12_w_div}, 
\eqref{eqq_Phi_21_w_div}, 
\eqref{eqq_Phi_22_w_div}, 
\eqref{eqq_oi_oj}, and \eqref{eqq_cb_change}, 
the first term of Eq.~\eqref{eqq_phi_theta_dif} becomes
\ali{
&\frac{\pa \bm{\phi}_{ij}}{\pa \bm{w}}(\bm{o}_i \otimes \bm{o}_j) \nn \\
&= 
\mat{
\frac{\pa (\bm{\Phi}_{ij})_{11}}{\pa \bm{w}} &
\frac{\pa (\bm{\Phi}_{ij})_{12}}{\pa \bm{w}} &
\frac{\pa (\bm{\Phi}_{ij})_{21}}{\pa \bm{w}} &
\frac{\pa (\bm{\Phi}_{ij})_{22}}{\pa \bm{w}}}(\bm{o}_i \otimes \bm{o}_j) \nn \\
&=
\bm{h}(\bm{x}_{i}) \otimes
\mat{
\bm{\mathfrak{a}}^\Phi_{ij} &
\bm{\mathfrak{b}}^\Phi_{ij} & 
\bm{\mathfrak{c}}^\Phi_{ij} & 
\bm{\mathfrak{d}}^\Phi_{ij}  
}(\bm{o}_i \otimes \bm{o}_j) \nn \\
&+ 
\bm{h}(\bm{x}_{j})\otimes 
\mat{
\bm{\mathfrak{a}}^\Phi_{ji} &
\bm{\mathfrak{c}}^\Phi_{ji} & 
\bm{\mathfrak{b}}^\Phi_{ji} & 
\bm{\mathfrak{d}}^\Phi_{ji}  
}(\bm{o}_i \otimes \bm{o}_j) \nn \\
&= 
\bm{h}(\bm{x}_{i}) \otimes
\mat{
\bm{\mathfrak{a}}^\Phi_{ij} &
\bm{\mathfrak{b}}^\Phi_{ij} & 
\bm{\mathfrak{c}}^\Phi_{ij} & 
\bm{\mathfrak{d}}^\Phi_{ij}  
}(\bm{o}_i \otimes \bm{o}_j)  \nn \\
&+ 
\bm{h}(\bm{x}_{j})\otimes 
\mat{
\bm{\mathfrak{a}}^\Phi_{ji} &
\bm{\mathfrak{b}}^\Phi_{ji} & 
\bm{\mathfrak{c}}^\Phi_{ji} & 
\bm{\mathfrak{d}}^\Phi_{ji}  
}(\bm{o}_j \otimes \bm{o}_i) \nn\\
&= \bm{h}(\bm{x}_{i}) \otimes \bm{G}^\Phi_{ij} (\bm{o}_i \otimes \bm{o}_j)
+ \bm{h}(\bm{x}_{j}) \otimes \bm{G}^\Phi_{ji} (\bm{o}_j \otimes \bm{o}_i). \label{eqq_phi_change_bc}
}
From Eqs.~\eqref{eqq_kron_1} and \eqref{eqq_o_w_dif}, the second and third terms of Eq.~\eqref{eqq_phi_theta_dif} become
\ali{
\frac{\pa \bm{o}_i}{\pa \bm{w}} \bm{\Phi}_{ij} \bm{o}_j
&= \bm{h}(\bm{x}_i) \otimes (\bm{F}^\prime (\bm{y}_i) \widetilde{\bm{V}}^\top \bm{s}^{\prime\prime/\prime}(z_i)^\top \bm{\Phi}_{ij} \bm{o}_j), \label{eqq_oi_w_div} \\
\frac{\pa \bm{o}_{j} }{\pa \bm{w}} \bm{\Phi}_{ji} \bm{o}_{i}
&= \bm{h}(\bm{x}_j) \otimes (\bm{F}^\prime (\bm{y}_j) \widetilde{\bm{V}}^\top \bm{s}^{\prime\prime/\prime}(z_j)^\top \bm{\Phi}_{ji} \bm{o}_i). \label{eqq_oj_w_div}
}
From Eqs.~\eqref{eqq_Aij_Phi}, \eqref{eqq_Bij_Phi}, \eqref{eqq_kron_1_5}, \eqref{eqq_phi_tr2_div}, \eqref{eqq_phi_theta_dif}, \eqref{eqq_phi_change_bc}, \eqref{eqq_oi_w_div}, and \eqref{eqq_oj_w_div}, 
\ali{
&\Bigg[ \frac{\pa \mathrm{tr}(\bm{H}_L(\bm{\theta})^2)}{\pa \bm{w}} \Bigg]_\Phi
= \sum_{i=1}^I \sum_{j=1}^I 
(1+ \bm{x}_i^\top\bm{x}_j)^2 \frac{\pa \phi_{ij}}{\pa \bm{w}} \nn \\
&= \sum_{i=1}^I \sum_{j=1}^I \bigg( \bm{h}(\bm{x}_i) \otimes (\bm{\mathcal{A}}_{ij}^\Phi + \bm{\mathcal{B}}_{ij}^\Phi) + \bm{h}(\bm{x}_j) \otimes (\bm{\mathcal{A}}_{ji}^\Phi + \bm{\mathcal{B}}_{ji}^\Phi) \bigg)\nn
}
holds.
\end{proof}

\subsubsection{Gradients of the Psi Term}
Here, we clarify
$
[ \pa \mathrm{tr}(\bm{H}_L(\bm{\theta})^2) / \pa \bm{w} ]_\Psi
$
shown in Eq.~\eqref{eqq_psi_tr2_div}.

\begin{lemma}
\ali{
\frac{\pa (\bm{\Psi}_{ij})_{11} }{\pa \bm{w}} 
&= \bm{h}(\bm{x}_{i}) \otimes \bm{\mathfrak{a}}^\Psi_{ij} + \bm{h}(\bm{x}_{j}) \otimes \bm{\mathfrak{a}}^\Psi_{ji}  \in \mathbb{R}^{(M+1)N}. \label{eqq_Psi_11_w_div}
}
\end{lemma}
\begin{proof}
From Eq.~\eqref{eqq_Psi_11}, 
\ali{
&\frac{\pa (\bm{\Psi}_{ij})_{11} }{\pa \bm{\theta}} 
= \frac{\pa (1+\bm{f}(\bm{y}_i)^\top \bm{f}(\bm{y}_j)) \widetilde{\bm{V}} \bm{F}^\prime(\bm{y}_j) \bm{F}^\prime(\bm{y}_i) \widetilde{\bm{V}}^\top }{\pa \bm{\theta}} \nn \\
&= 
\frac{\pa \bm{f}(\bm{y}_i)^\top \bm{f}(\bm{y}_j)}{\pa \bm{\theta}} \widetilde{\bm{V}} \bm{F}^\prime(\bm{y}_j) \bm{F}^\prime(\bm{y}_i) \widetilde{\bm{V}}^\top\nn \\
&+ 
(1+\bm{f}(\bm{y}_i)^\top \bm{f}(\bm{y}_j)) \frac{\pa \widetilde{\bm{V}} \bm{F}^\prime(\bm{y}_j) \bm{F}^\prime(\bm{y}_i) \widetilde{\bm{V}}^\top}{\pa \bm{\theta}} \nn \\
&=\bigg( \frac{\pa \bm{f}(\bm{y}_i)}{\pa \bm{\theta}} \bm{f}(\bm{y}_j) + \frac{\pa \bm{f}(\bm{y}_j)}{\pa \bm{\theta}}\bm{f}(\bm{y}_i) \bigg) \widetilde{\bm{V}} \bm{F}^\prime(\bm{y}_j) \bm{F}^\prime(\bm{y}_i) \widetilde{\bm{V}}^\top\nn \\
&+ 
(1+\bm{f}(\bm{y}_i)^\top \bm{f}(\bm{y}_j)) \nn\\
&\times \bigg( 
\frac{\pa \bm{F}^\prime(\bm{y}_j) \widetilde{\bm{V}}^\top}{\pa \bm{\theta}} \bm{F}^\prime(\bm{y}_i) + \frac{\pa \bm{F}^\prime(\bm{y}_i) \widetilde{\bm{V}}^\top}{\pa \bm{\theta}}  \bm{F}^\prime(\bm{y}_j) \bigg) \widetilde{\bm{V}}^\top \label{eq_psi_11_theta}
}
holds.
Note that $\widetilde{\bm{V}} \bm{F}^\prime(\bm{y}_j) \bm{F}^\prime(\bm{y}_i) \widetilde{\bm{V}}^\top$ is an inner product because $\widetilde{\bm{V}} \bm{F}^\prime(\bm{y}_j)$ is a row vector and $\bm{F}^\prime(\bm{y}_i) \widetilde{\bm{V}}^\top$ is a column vector.
Therefore, Eq.~\eqref{eqq_ripniz} is applied.
From Eqs.~\eqref{eqq_a_psi_ij}, \eqref{eqq_kron_1}, \eqref{eqq_fyv_w}, \eqref{eqq_pfy_w}, and \eqref{eq_psi_11_theta},
\ali{
&\frac{\pa (\bm{\Psi}_{ij})_{11} }{\pa \bm{w}} =\Big( \bm{h}(\bm{x}_i) \otimes (\bm{F}^\prime(\bm{y}_i) \bm{f}(\bm{y}_j) ) \nn \\
&+ \bm{h}(\bm{x}_j) \otimes (\bm{F}^\prime(\bm{y}_j)\bm{f}(\bm{y}_i) )\Big) \widetilde{\bm{V}} \bm{F}^\prime(\bm{y}_j) \bm{F}^\prime(\bm{y}_i) \widetilde{\bm{V}}^\top \nn \\
&+ 
(1+\bm{f}(\bm{y}_i)^\top \bm{f}(\bm{y}_j)) \Big( 
\bm{h}(\bm{x}_j) \otimes (\mathrm{diag}(\widetilde{\bm{V}}^\top) \bm{F}^{\prime\prime}(\bm{y}_j) \bm{F}^\prime(\bm{y}_i) ) \nn \\
&+ 
\bm{h}(\bm{x}_i) \otimes (\mathrm{diag}(\widetilde{\bm{V}}^\top) \bm{F}^{\prime\prime}(\bm{y}_i) \bm{F}^\prime(\bm{y}_j) )
\Big) \widetilde{\bm{V}}^\top \nn \\
&=
\bm{h}(\bm{x}_i) \otimes \Big(\big(\bm{F}^\prime(\bm{y}_i) \bm{f}(\bm{y}_j) \widetilde{\bm{V}} \bm{F}^\prime(\bm{y}_i) \nn \\
&+ \mathrm{diag}(\widetilde{\bm{V}}^\top) \bm{F}^{\prime\prime}(\bm{y}_i)(1+\bm{f}(\bm{y}_i)^\top \bm{f}(\bm{y}_j))\big) \bm{F}^\prime(\bm{y}_j)\widetilde{\bm{V}}^\top\Big) \nn \\
&+\bm{h}(\bm{x}_j) \otimes \Big((
\bm{F}^\prime(\bm{y}_j)\bm{f}(\bm{y}_i)\widetilde{\bm{V}} \bm{F}^\prime(\bm{y}_j)\nn \\
&+
\mathrm{diag}(\widetilde{\bm{V}}^\top) \bm{F}^{\prime\prime}(\bm{y}_j) (1+\bm{f}(\bm{y}_j)^\top \bm{f}(\bm{y}_i)) ) \bm{F}^\prime(\bm{y}_i) \widetilde{\bm{V}}^\top\Big)\nn\\
&= \bm{h}(\bm{x}_{i}) \otimes \bm{\mathfrak{a}}^\Psi_{ij} + \bm{h}(\bm{x}_{j}) \otimes \bm{\mathfrak{a}}^\Psi_{ji}\nn
}
is obtained.
\end{proof}

\begin{lemma}
\ali{
&\frac{\pa (\bm{\Psi}_{ij})_{12}}{\pa \bm{w}} 
= \bm{h}(\bm{x}_{i}) \otimes \bm{\mathfrak{b}}^\Psi_{ij} + \bm{h}(\bm{x}_{j}) \otimes \bm{\mathfrak{c}}^\Psi_{ji} \in \mathbb{R}^{(M+1)N}, \label{eqq_Psi_12_w_div} \\
&\frac{\pa (\bm{\Psi}_{ij})_{21}}{\pa \bm{w}} 
= \bm{h}(\bm{x}_{i}) \otimes \bm{\mathfrak{c}}^\Psi_{ij} + \bm{h}(\bm{x}_{j}) \otimes \bm{\mathfrak{b}}^\Psi_{ji} \in \mathbb{R}^{(M+1)N}. \label{eqq_Psi_21_w_div}
}
\end{lemma}
\begin{proof}
From Eq.~\eqref{eqq_Psi_12}, 
\ali{
&\frac{\pa (\bm{\Psi}_{ij})_{12}}{\pa \bm{\theta}} 
= \frac{\pa \bm{f}(\bm{y}_i)^\top \bm{F}^\prime(\bm{y}_j) \bm{F}^\prime(\bm{y}_i) \widetilde{\bm{V}}^\top}{\pa \bm{\theta}} \nn \\
&=\frac{\pa \bm{F}^\prime(\bm{y}_j) \bm{f}(\bm{y}_i)}{\pa \bm{\theta}}\bm{F}^\prime(\bm{y}_i) \widetilde{\bm{V}}^\top 
+\frac{\pa \bm{F}^\prime(\bm{y}_i)  \widetilde{\bm{V}}^\top}{\pa \bm{\theta}}\bm{F}^\prime(\bm{y}_j) \bm{f}(\bm{y}_i)
\label{eqq_Psi_12_theta_div}
}
holds.
Note that $\bm{f}(\bm{y}_i)^\top \bm{F}^\prime(\bm{y}_j) \bm{F}^\prime(\bm{y}_i) \widetilde{\bm{V}}^\top$ is an inner product because $\bm{f}(\bm{y}_i)^\top \bm{F}^\prime(\bm{y}_j)$ is a row vector and $\bm{F}^\prime(\bm{y}_i) \widetilde{\bm{V}}^\top$ is a column vector.
Therefore, Eq.~\eqref{eqq_ripniz} is applied.
Since $\bm{F}^\prime(\bm{y}_j) \bm{f}(\bm{y}_i) = \mat{f^\prime(y_{j1}) f(y_{i1}) \cdots f^\prime(y_{jN}) f(y_{iN})}^\top$, deriving the partial derivative of $f^\prime(y_{jn}) f(y_{in})$ with respect to $w_{nm}$ yields
\ali{
&\frac{\pa f^\prime(y_{jn}) f(y_{in})}{\pa w_{nm}} 
= \frac{\pa f^\prime(y_{jn})}{\pa w_{nm}} f(y_{in})
+\frac{\pa f(y_{in})}{\pa w_{nm}} f^\prime(y_{jn})\nn\\
&= \frac{\pa f^\prime(y_{jn})}{\pa y_{jn}} \frac{\pa y_{jn}}{\pa w_{nm}} f(y_{in})
+\frac{\pa f(y_{in})}{\pa y_{in}} \frac{\pa y_{in}}{\pa w_{nm}} f^\prime(y_{jn}) \nn \\
&= x_{jm} f^{\prime\prime}(y_{jn}) f(y_{in})
+ x_{im} f^{\prime}(y_{in}) f^\prime(y_{jn}). \nn
}
Therefore, 
\ali{
&\frac{\pa \bm{F}^\prime(\bm{y}_j) \bm{f}(\bm{y}_i)}{\pa w_{nm}}\nn \\
&= \mat{\frac{\pa f^\prime(y_{j1}) f(y_{i1}) }{\pa w_{nm}}  \cdots  \frac{\pa f^\prime(y_{jn}) f(y_{in}) }{\pa w_{nm}}  \cdots   \frac{\pa f^\prime(y_{jN}) f(y_{iN}) }{\pa w_{nm}}} \nn \\
&= \mat{ 0  \cdots  x_{jm} f^{\prime\prime}(y_{jn}) f(y_{in})
+ x_{im} f^{\prime}(y_{in}) f^\prime(y_{jn})  \cdots  0}, \nn
}
and 
\ali{
\frac{\pa \bm{F}^\prime(\bm{y}_j) \bm{f}(\bm{y}_i)}{\pa \bm{W}_{:m}}
&= 
x_{im} \bm{F}^{\prime}(\bm{y}_j) \bm{F}^\prime(\bm{y}_i)
+
x_{jm} \bm{F}^{\prime\prime}(\bm{y}_j) \bm{F}(\bm{y}_i),
\nn \\
\frac{\pa \bm{F}^\prime(\bm{y}_j) \bm{f}(\bm{y}_i)}{\pa \bm{w}} 
&= \bm{h}(\bm{x}_{i}) \otimes (\bm{F}^{\prime}(\bm{y}_j) \bm{F}^\prime(\bm{y}_i))\nn \\
&+ 
\bm{h}(\bm{x}_{j}) \otimes (\bm{F}^{\prime\prime}(\bm{y}_j) \bm{F}(\bm{y}_i))
\label{eqq_fyjyi_w_div}
}
are obtained.
Finally, from Eqs.~\eqref{eqq_b_psi_ij}, \eqref{eqq_c_psi_ij}, \eqref{eqq_kron_1}, \eqref{eqq_kron_1_5}, \eqref{eqq_fyv_w}, \eqref{eqq_Psi_12_theta_div}, and \eqref{eqq_fyjyi_w_div}, we obtain
\ali{
\frac{\pa (\bm{\Psi}_{ij})_{12}}{\pa \bm{w}} 
&=\Big( \bm{h}(\bm{x}_{i}) \otimes (\bm{F}^{\prime}(\bm{y}_j) \bm{F}^\prime(\bm{y}_i)) \nn \\
&+ \bm{h}(\bm{x}_{j}) \otimes (\bm{F}^{\prime\prime}(\bm{y}_j) \bm{F}(\bm{y}_i)) \Big)
\bm{F}^\prime(\bm{y}_i) \widetilde{\bm{V}}^\top \nn \\
&+\bm{h}(\bm{x}_i) \otimes (\mathrm{diag}(\widetilde{\bm{V}}^\top) \bm{F}^{\prime\prime}(\bm{y}_i) \bm{F}^\prime(\bm{y}_j) \bm{f}(\bm{y}_i) ) \nn \\
&=\bm{h}(\bm{x}_{i}) \otimes ( (
\bm{F}^\prime(\bm{y}_i)^2 + \bm{F}^{\prime\prime}(\bm{y}_i) \bm{F}(\bm{y}_i) ) \bm{F}^{\prime}(\bm{y}_j) \widetilde{\bm{V}}^\top)\nn \\
&+
\bm{h}(\bm{x}_{j}) \otimes (\bm{F}^{\prime\prime}(\bm{y}_j) \bm{F}(\bm{y}_i) \bm{F}^\prime(\bm{y}_i) \widetilde{\bm{V}}^\top) \nn \\
&= \bm{h}(\bm{x}_{i}) \otimes \bm{\mathfrak{b}}^\Psi_{ij} + \bm{h}(\bm{x}_{j}) \otimes \bm{\mathfrak{c}}^\Psi_{ji}, 
}
where 
\ali{
\mathrm{diag}(\widetilde{\bm{V}}^\top) \bm{F}^{\prime\prime}(\bm{y}_i) \bm{F}^\prime(\bm{y}_j) \bm{f}(\bm{y}_i) = \bm{F}^{\prime\prime}(\bm{y}_i) \bm{F}(\bm{y}_i) \bm{F}^{\prime}(\bm{y}_j) \widetilde{\bm{V}}^\top\nn
} is used.
Also, from Eqs.~\eqref{eqq_Psi_12} and \eqref{eqq_Psi_21}, 
\ali{
(\bm{\Psi}_{ij})_{21} &= \widetilde{\bm{V}} \bm{F}^\prime(\bm{y}_j) \bm{F}^\prime(\bm{y}_i) \bm{f}(\bm{y}_j)\nn \\
&= \bm{f}(\bm{y}_j)^\top \bm{F}^\prime(\bm{y}_i) \bm{F}^\prime(\bm{y}_j) \widetilde{\bm{V}}^\top 
= (\bm{\Psi}_{ji})_{12} \label{eqq_Psi_12_21_change}
}
holds.
Therefore, from Eq.~\eqref{eqq_Psi_12_w_div},
\ali{
\frac{\pa (\bm{\Psi}_{ij})_{21}}{\pa \bm{w}} = \frac{\pa (\bm{\Psi}_{ji})_{12}}{\pa \bm{w}}
= \bm{h}(\bm{x}_{i}) \otimes \bm{\mathfrak{c}}^\Psi_{ij} + \bm{h}(\bm{x}_{j}) \otimes \bm{\mathfrak{b}}^\Psi_{ji}\nn
}
is obtained.
\end{proof}

\begin{lemma}
\ali{
\frac{\pa (\bm{\Psi}_{ij})_{22} }{\pa \bm{w}} 
&= \bm{h}(\bm{x}_{i}) \otimes \bm{\mathfrak{d}}^\Psi_{ij} + \bm{h}(\bm{x}_{j}) \otimes \bm{\mathfrak{d}}^\Psi_{ji}
\in \mathbb{R}^{(M+1)N}.
\label{eqq_Psi_22_w_div}
}
\end{lemma}
\begin{proof}
From Eqs.~\eqref{eqq_d_psi_ij}, \eqref{eqq_ripniz}, and \eqref{eqq_fpyy_and_yw_div},
\ali{
&\frac{\pa (\bm{\Psi}_{ij})_{22} }{\pa \bm{w}} 
= \frac{\pa \bm{f}^\prime(\bm{y}_i)^\top \bm{f}^\prime(\bm{y}_j) }{\pa \bm{w}}\nn \\
&= \frac{\pa \bm{f}^\prime(\bm{y}_i) }{\pa \bm{w}} \bm{f}^\prime(\bm{y}_j)
+ \frac{\pa \bm{f}^\prime(\bm{y}_j) }{\pa \bm{w}} \bm{f}^\prime(\bm{y}_i) \nn \\
&= \bm{h}(\bm{x}_i)\otimes (\bm{F}^{\prime\prime}(\bm{y}_i) \bm{f}^\prime(\bm{y}_j) )
+ \bm{h}(\bm{x}_j)\otimes ( \bm{F}^{\prime\prime}(\bm{y}_j) \bm{f}^\prime(\bm{y}_i) ) \nn \\
&= \bm{h}(\bm{x}_{i}) \otimes \bm{\mathfrak{d}}^\Psi_{ij} + \bm{h}(\bm{x}_{j}) \otimes \bm{\mathfrak{d}}^\Psi_{ji}\nn
}
is obtained.
\end{proof}

\begin{lemma}
\ali{
&\Bigg[ \frac{\pa \mathrm{tr}(\bm{H}_L(\bm{\theta})^2)}{\pa \bm{w}} \Bigg]_\Psi
= \sum_{i=1}^I \sum_{j=1}^I \nn \\
&\bigg(
\bm{h}(\bm{x}_i) \otimes (\bm{\mathcal{A}}_{ij}^\Psi + \bm{\mathcal{B}}_{ij}^\Psi) + \bm{h}(\bm{x}_j) \otimes (\bm{\mathcal{A}}_{ji}^\Psi + \bm{\mathcal{B}}_{ji}^\Psi) \bigg).  \label{eqq_w_Psi_fin}
}
\end{lemma}

\begin{proof}
From Eqs.~\eqref{eqq_G_psi_ij}, \eqref{eqq_jac_psi_theta}, \eqref{eqq_cb_change}, \eqref{eqq_Psi_11_w_div}, \eqref{eqq_Psi_12_w_div}, \eqref{eqq_Psi_21_w_div}, and \eqref{eqq_Psi_22_w_div}, the first term of Eq.~\eqref{eqq_psi_theta_dif} becomes
\ali{
&\frac{\pa \bm{\psi}_{ij}}{\pa \bm{w}}(\bm{o}_i \otimes \bm{o}_j) \nn \\
&= 
\mat{
\frac{\pa (\bm{\Psi}_{ij})_{11}}{\pa \bm{w}} &
\frac{\pa (\bm{\Psi}_{ij})_{12}}{\pa \bm{w}} &
\frac{\pa (\bm{\Psi}_{ij})_{21}}{\pa \bm{w}} &
\frac{\pa (\bm{\Psi}_{ij})_{22}}{\pa \bm{w}}}(\bm{o}_i \otimes \bm{o}_j) \nn \\
&= \bm{h}(\bm{x}_{i}) \otimes
\mat{
\bm{\mathfrak{a}}^\Psi_{ij} &
\bm{\mathfrak{b}}^\Psi_{ij} & 
\bm{\mathfrak{c}}^\Psi_{ij} & 
\bm{\mathfrak{d}}^\Psi_{ij}  
}(\bm{o}_i \otimes \bm{o}_j)\nn \\
& + 
\bm{h}(\bm{x}_{j})\otimes 
\mat{
\bm{\mathfrak{a}}^\Psi_{ji} &
\bm{\mathfrak{c}}^\Psi_{ji} & 
\bm{\mathfrak{b}}^\Psi_{ji} & 
\bm{\mathfrak{d}}^\Psi_{ji}  
}(\bm{o}_i \otimes \bm{o}_j) \nn \\
&= 
\bm{h}(\bm{x}_{i}) \otimes
\mat{
\bm{\mathfrak{a}}^\Psi_{ij} &
\bm{\mathfrak{b}}^\Psi_{ij} & 
\bm{\mathfrak{c}}^\Psi_{ij} & 
\bm{\mathfrak{d}}^\Psi_{ij}  
}(\bm{o}_i \otimes \bm{o}_j) \nn \\
&+ 
\bm{h}(\bm{x}_{j})\otimes 
\mat{
\bm{\mathfrak{a}}^\Psi_{ji} &
\bm{\mathfrak{b}}^\Psi_{ji} & 
\bm{\mathfrak{c}}^\Psi_{ji} & 
\bm{\mathfrak{d}}^\Psi_{ji}  
}(\bm{o}_j \otimes \bm{o}_i) \nn \\
&= \bm{h}(\bm{x}_{i}) \otimes \bm{G}^\Psi_{ij} (\bm{o}_i \otimes \bm{o}_j)
+ \bm{h}(\bm{x}_{j}) \otimes \bm{G}^\Psi_{ji} (\bm{o}_j \otimes \bm{o}_i). \label{eqq_psi_w_oi_oj}
}
From Eqs.~\eqref{eqq_kron_1} and \eqref{eqq_o_w_dif}, the second and third terms of Eq.~\eqref{eqq_psi_theta_dif} become
\ali{
\frac{\pa \bm{o}_i}{\pa \bm{w}} \bm{\Psi}_{ij} \bm{o}_j
&= \bm{h}(\bm{x}_i) \otimes (\bm{F}^\prime (\bm{y}_i) \widetilde{\bm{V}}^\top \bm{s}^{\prime\prime/\prime}(z_i)^\top \bm{\Psi}_{ij} \bm{o}_j),\label{eqq_oi_w_psi_oj} \\
\frac{\pa \bm{o}_{j} }{\pa \bm{w}} \bm{\Psi}_{ji} \bm{o}_{i}
&= \bm{h}(\bm{x}_j) \otimes (\bm{F}^\prime (\bm{y}_j) \widetilde{\bm{V}}^\top \bm{s}^{\prime\prime/\prime}(z_j)^\top \bm{\Psi}_{ji} \bm{o}_i). \label{eqq_oj_w_psi_oi}
}
From Eqs.~\eqref{eqq_Aij_Psi}, \eqref{eqq_Bij_Psi}, \eqref{eqq_kron_1_5}, \eqref{eqq_psi_tr2_div}, \eqref{eqq_psi_theta_dif}, \eqref{eqq_psi_w_oi_oj}, \eqref{eqq_oi_w_psi_oj}, and \eqref{eqq_oj_w_psi_oi}, 
\ali{
&\Bigg[ \frac{\pa \mathrm{tr}(\bm{H}_L(\bm{\theta})^2)}{\pa \bm{w}} \Bigg]_\Psi
= 2\sum_{i=1}^I \sum_{j=1}^I 
(1+ \bm{x}_i^\top\bm{x}_j) \frac{\pa \psi_{ij}}{\pa \bm{w}} \nn \\
&= \sum_{i=1}^I \sum_{j=1}^I \bigg( \bm{h}(\bm{x}_i) \otimes (\bm{\mathcal{A}}_{ij}^\Psi + \bm{\mathcal{B}}_{ij}^\Psi) + \bm{h}(\bm{x}_j) \otimes (\bm{\mathcal{A}}_{ji}^\Psi + \bm{\mathcal{B}}_{ji}^\Psi) \bigg)\nn
}
holds.
\end{proof}

\subsubsection{Gradient of the Omega Term}
With respect to $\omega_{ij}$, the following holds.
\begin{lemma}
\ali{
\omega_{ij} = \omega_{ji}. \label{eqq_omega_change_ij}
}
\end{lemma}
\begin{proof}

From Eq.~\eqref{eq_ome_ij}, $\bm{\Omega}_{ij} = \bm{\Omega}_{ji}$ holds. Therefore, we obtain
\ali{
\omega_{ij} = \bm{o}_{i}^\top \bm{\Omega}_{ij} \bm{o}_{j} 
= \bm{o}_{j}^\top \bm{\Omega}_{ji} \bm{o}_{i} = \omega_{ji}. \nn
}
Note that $\omega_{ij} = s^\prime(z_i) s^\prime(z_j)$ from Eq.~\eqref{eqq_o_def}.
\end{proof}

\begin{lemma}
\ali{
&\Bigg[ \frac{\pa \mathrm{tr}(\bm{H}_L(\bm{\theta})^2)}{\pa \bm{w}} \Bigg]_\Omega
= \sum_{i=1}^I \sum_{j=1}^I\nn \\
&\bigg( 
\bm{h}(\bm{x}_i) \otimes (\bm{\mathcal{A}}_{ij}^\Omega + \bm{\mathcal{B}}_{ij}^\Omega) + \bm{h}(\bm{x}_j) \otimes (\bm{\mathcal{A}}_{ji}^\Omega + \bm{\mathcal{B}}_{ji}^\Omega) \bigg). \label{eqq_w_Omega_fin}
}
\end{lemma}
\begin{proof}
From Eqs.~\eqref{eq_ome_ij} and \eqref{eqq_spz_w}, 
\ali{
\frac{\pa \omega_{ij}}{\pa \bm{w}}
&= \frac{\pa \bm{o}_{i}^\top \bm{\Omega}_{ij} \bm{o}_{j} }{\pa \bm{w}}
= \frac{\pa s^\prime(z_i) s^\prime(z_j) }{\pa \bm{w}}\nn \\
&= \frac{\pa s^\prime(z_i) }{\pa \bm{w}} s^\prime(z_j) + \frac{\pa s^\prime(z_j) }{\pa \bm{w}} s^\prime(z_i) \nn \\
&= s^{\prime\prime}(z_i) s^\prime(z_j) \bm{h}(\bm{x}_i) 
\otimes (\bm{F}^{\prime}(\bm{y}_i) \widetilde{\bm{V}}^\top) \nn \\
&+ s^{\prime\prime}(z_j)s^\prime(z_i) \bm{h}(\bm{x}_j) 
 \otimes (\bm{F}^{\prime}(\bm{y}_j) \widetilde{\bm{V}}^\top) \nn
}
holds.
From Eq.~\eqref{eqq_Aij_Ome} we obtain
\ali{
&(1+ \bm{f}(\bm{y}_i)^\top\bm{f}(\bm{y}_j))^2 \frac{\pa \omega_{ij}}{\pa \bm{w}}\nn \\
&= \bm{h}(\bm{x}_i) 
\otimes ((1+ \bm{f}(\bm{y}_i)^\top\bm{f}(\bm{y}_j))^2 s^{\prime\prime}(z_i) s^\prime(z_j) \bm{F}^{\prime}(\bm{y}_i) \widetilde{\bm{V}}^\top) \nn \\
&+  \bm{h}(\bm{x}_j) 
 \otimes ((1+ \bm{f}(\bm{y}_j)^\top\bm{f}(\bm{y}_i))^2 s^{\prime\prime}(z_j)s^\prime(z_i) \bm{F}^{\prime}(\bm{y}_j) \widetilde{\bm{V}}^\top)\nn \\ 
&=
\bm{h}(\bm{x}_i) \otimes \bm{\mathcal{A}}_{ij}^\Omega + \bm{h}(\bm{x}_j) \otimes \bm{\mathcal{A}}_{ji}^\Omega. 
\nn
}
From Eqs.~\eqref{eqq_Bij_Ome}, \eqref{eqq_ripniz}, and \eqref{eqq_pfy_w},
\ali{
&\omega_{ij} \frac{\pa (1+ \bm{f}(\bm{y}_i)^\top\bm{f}(\bm{y}_j))^2 }{\pa \bm{w}}\nn \\
&= \omega_{ij} \frac{\pa (1+ \bm{f}(\bm{y}_i)^\top\bm{f}(\bm{y}_j))^2 }{\pa (1+ \bm{f}(\bm{y}_i)^\top\bm{f}(\bm{y}_j))} 
\frac{\pa \bm{f}(\bm{y}_i)^\top\bm{f}(\bm{y}_j) }{\pa \bm{w}}\nn \\
&= 2 \omega_{ij} (1+ \bm{f}(\bm{y}_i)^\top\bm{f}(\bm{y}_j))
\bigg( \frac{\pa \bm{f}(\bm{y}_i) }{\pa \bm{w}} \bm{f}(\bm{y}_j) + \frac{\pa \bm{f}(\bm{y}_j) }{\pa \bm{w}} \bm{f}(\bm{y}_i) \bigg) \nn \\
&=  
2 \omega_{ij} (1+ \bm{f}(\bm{y}_i)^\top \bm{f}(\bm{y}_j))\nn \\
&\times  (\bm{h}(\bm{x}_i) \otimes \bm{F}^\prime(\bm{y}_i) \bm{f}(\bm{y}_j) + 
 \bm{h}(\bm{x}_j) \otimes \bm{F}^\prime(\bm{y}_j) \bm{f}(\bm{y}_i) ) \nn\\
&= \bm{h}(\bm{x}_i) \otimes \bm{\mathcal{B}}_{ij}^\Omega + \bm{h}(\bm{x}_j) \otimes \bm{\mathcal{B}}_{ji}^\Omega
\nn
}
holds.
Therefore, from Eqs.~\eqref{eqq_kron_1_5} and \eqref{eqq_omega_tr2_div}, 
\ali{
&\Bigg[ \frac{\pa \mathrm{tr}(\bm{H}_L(\bm{\theta})^2)}{\pa \bm{w}} \Bigg]_\Omega
= \sum_{i=1}^I \sum_{j=1}^I 
(1+ \bm{f}(\bm{y}_i)^\top\bm{f}(\bm{y}_j))^2 \frac{\pa \omega_{ij}}{\pa \bm{w}} \nn \\
&+ 
\sum_{i=1}^I \sum_{j=1}^I 
\omega_{ij} \frac{\pa (1+ \bm{f}(\bm{y}_i)^\top\bm{f}(\bm{y}_j))^2 }{\pa \bm{w}} \nn \\
&= 
 \sum_{i=1}^I \sum_{j=1}^I \bigg(
\bm{h}(\bm{x}_i) \otimes (\bm{\mathcal{A}}_{ij}^\Omega + \bm{\mathcal{B}}_{ij}^\Omega)+ \bm{h}(\bm{x}_j) \otimes (\bm{\mathcal{A}}_{ji}^\Omega + \bm{\mathcal{B}}_{ji}^\Omega) \bigg)\nn
}
is obtained.
\end{proof}

\subsubsection{Completion of the Proof}
From 
Eqs.~\eqref{eqq_H2_theta_div}, 
\eqref{eqq_phi_tr2_div}, 
\eqref{eqq_psi_tr2_div}, 
\eqref{eqq_omega_tr2_div}, 
\eqref{eqq_w_Phi_fin}, 
\eqref{eqq_w_Psi_fin}, and
\eqref{eqq_w_Omega_fin}, we have
\ali{
&\frac{\pa \mathrm{tr}(\bm{H}_L(\bm{\theta})^2)}{\pa \bm{w}} = \sum_{a \in \mathfrak{B}} \Bigg[ \frac{\pa \mathrm{tr}(\bm{H}_L(\bm{\theta})^2)}{\pa \bm{w}} \Bigg]_a \nn \\
&=  \sum_{i=1}^I \sum_{j=1}^I \bm{h}(\bm{x}_i) \otimes (\bm{\mathcal{A}}_{ij}^\Phi + \bm{\mathcal{A}}_{ij}^\Psi
+ \bm{\mathcal{A}}_{ij}^\Omega + \bm{\mathcal{B}}_{ij}^\Phi + \bm{\mathcal{B}}_{ij}^\Psi + \bm{\mathcal{B}}_{ij}^\Omega) \nn\\
&+  \sum_{i=1}^I \sum_{j=1}^I \bm{h}(\bm{x}_j) \otimes (\bm{\mathcal{A}}_{ji}^\Phi + \bm{\mathcal{A}}_{ji}^\Psi
+ \bm{\mathcal{A}}_{ji}^\Omega + \bm{\mathcal{B}}_{ji}^\Phi + \bm{\mathcal{B}}_{ji}^\Psi + \bm{\mathcal{B}}_{ji}^\Omega) \nn \\
&= 2 \sum_{i=1}^I \sum_{j=1}^I \bm{h}(\bm{x}_i) \otimes (\bm{\mathcal{W}}_{ij}^\Phi + \bm{\mathcal{W}}_{ij}^\Psi
+ \bm{\mathcal{W}}_{ij}^\Omega). \nn
}
Therefore, Eq.~\eqref{eqq_H2_w_div_fin} holds.

\subsection{Proof for Eq. \eqref{eqq_H2_v_div_fin}} \label{secc_H2_grad_v}
\subsubsection{Gradient with Respect to Affine Parameters from Hidden to Output Layer}
$\pa \phi_{ij}/\pa \bm{\theta}$ contains $\pa \bm{o} / \pa \bm{\theta}$. Therefore, we present the Jacobian of $\bm{o}$ with respect to $\bm{v}$ here.

\begin{lemma}
\ali{
\frac{\pa \bm{o}}{\pa \bm{v}} = \bm{h}(\bm{f}(\bm{y})) \bm{s}^{\prime\prime/\prime}(z)^\top \in \mathbb{R}^{(N+1) \times 2}. \label{eqq_o_v_div}
}
\end{lemma}
\begin{proof}
From Eqs.~\eqref{eqq_spp_sp}, \eqref{eqq_szp_v}, and \eqref{eqq_d_v_div}, 
\ali{
\frac{\pa \bm{o}}{\pa \bm{v}}
&= \mat{\frac{\pa s^\prime(z)}{\pa \bm{v}} & \frac{\pa \delta}{\pa \bm{v}}} = \bm{h}(\bm{f}(\bm{y})) \mat{s^{\prime\prime}(z) & s^{\prime}(z)} \nn \\
&= \bm{h}(\bm{f}(\bm{y})) \bm{s}^{\prime\prime/\prime}(z)^\top \nn
}
holds, which is the Jacobian of $\bm{o}$ with respect to $\bm{v}$.
\end{proof}

\subsubsection{Gradient of the Phi Term}
\begin{lemma}
\ali{
\frac{\pa (\bm{\Phi}_{ij})_{11}}{\pa \bm{v}}
&= \bm{\mathfrak{e}}^{\Phi}_{ij} + \bm{\mathfrak{e}}^{\Phi}_{ji} 
\in \mathbb{R}^{N+1}.
\label{eqq_Phi_11_v_div}
}
\end{lemma}
\begin{proof}
From Eq.~\eqref{eqq_Phi_11}, we have
\ali{
(\bm{\Phi}_{ij})_{11}^{1/2} = \widetilde{\bm{V}} \bm{F}^\prime(\bm{y}_j) \bm{F}^\prime(\bm{y}_i) \widetilde{\bm{V}}^\top
= \sum_{n=1}^N v_n^2 f^\prime(y_{jn}) f^\prime(y_{in}). \nn
}
Therefore, 
\ali{
\frac{\pa (\bm{\Phi}_{ij})_{11}^{1/2}}{\pa v_n}
= \begin{cases}
0, & n = 0 \\
2 v_n f^\prime(y_{jn}) f^\prime(y_{in}), & n \in \mathbb{N}_{\le N}
\end{cases}\nn
}
holds.
Using Eq.~\eqref{eqq_e_phi_ij}, we obtain
\ali{
\frac{\pa (\bm{\Phi}_{ij})_{11}}{\pa \bm{v}}
&= 
\frac{\pa (\bm{\Phi}_{ij})_{11}}{\pa (\bm{\Phi}_{ij})_{11}^{1/2}}\frac{\pa (\bm{\Phi}_{ij})_{11}^{1/2}}{\pa \bm{v}}\nn \\
&= 
4 (\bm{\Phi}_{ij})_{11}^{1/2} \bm{V}^\top \odot\bm{f}^{\prime}_0(\bm{y}_j)
\odot 
\bm{f}^{\prime}_0(\bm{y}_i) \nn \\
&= 
2 (\bm{\Phi}_{ij})_{11}^{1/2} \bm{V}^\top \odot\bm{f}^{\prime}_0(\bm{y}_j)
\odot 
\bm{f}^{\prime}_0(\bm{y}_i)\nn \\
&+
2 (\bm{\Phi}_{ji})_{11}^{1/2} \bm{V}^\top \odot\bm{f}^{\prime}_0(\bm{y}_i)
\odot 
\bm{f}^{\prime}_0(\bm{y}_j) \nn \\
&= \bm{\mathfrak{e}}^{\Phi}_{ij} + \bm{\mathfrak{e}}^{\Phi}_{ji}, \nn
}
where $(\bm{\Phi}_{ij})_{11} = (\bm{\Phi}_{ji})_{11}$ from Eq.~\eqref{eqq_Phi_11} is utilized.
\end{proof}

\begin{lemma}
\ali{
\frac{\pa (\bm{\Phi}_{ij})_{12}}{\pa \bm{v}} &= \bm{\mathfrak{f}}^{\Phi}_{ij}\in \mathbb{R}^{N+1}, \label{eqq_Phi_12_v_div} \\ 
\frac{\pa (\bm{\Phi}_{ij})_{21}}{\pa \bm{v}} &= \bm{\mathfrak{f}}^{\Phi}_{ji}\in \mathbb{R}^{N+1}. \label{eqq_Phi_21_v_div}
}
\end{lemma}
\begin{proof}
From Eq.~\eqref{eqq_Phi_12}, since
\ali{
(\bm{\Phi}_{ij})_{12} &= \widetilde{\bm{V}}  \bm{F}^\prime(\bm{y}_i) \mathrm{diag}(\widetilde{\bm{V}}^\top)\bm{F}^{\prime\prime}(\bm{y}_j) \bm{F}^\prime(\bm{y}_i) \widetilde{\bm{V}}^\top \nn \\
&= \sum_{n=1}^N v_n^3 f^\prime(y_{in})^2 f^{\prime\prime}(y_{jn}), \nn
}
we obtain
\ali{
\frac{\pa (\bm{\Phi}_{ij})_{12}}{\pa v_n} =
\begin{cases}
0, & n = 0 \\
3v_n^2 f^\prime(y_{in})^2 f^{\prime\prime}(y_{jn}), & n \in \mathbb{N}_{\le N}
\end{cases}. \nn
}
Therefore, from Eq.~\eqref{eqq_f_phi_ij} we have
\ali{
\frac{\pa (\bm{\Phi}_{ij})_{12}}{\pa \bm{v}} =
3 \bm{V}^{\odot 2\top} \odot \bm{f}^{\prime}_0(\bm{y}_i)^{\odot 2} \odot \bm{f}^{\prime\prime}_0(\bm{y}_j)
= \bm{\mathfrak{f}}^{\Phi}_{ij}. \nn
}
Also, from Eqs.~\eqref{eqq_Phi_12} and \eqref{eqq_Phi_21}, 
\ali{
\frac{\pa (\bm{\Phi}_{ij})_{21}}{\pa \bm{v}} 
= \frac{\pa (\bm{\Phi}_{ji})_{12}}{\pa \bm{v}} 
= \bm{\mathfrak{f}}^{\Phi}_{ji} \nn
}
holds.
\end{proof}

\begin{lemma}
\ali{
\frac{\pa (\bm{\Phi}_{ij})_{22}}{\pa \bm{v}}
&= \bm{\mathfrak{g}}^{\Phi}_{ij} + \bm{\mathfrak{g}}^{\Phi}_{ji} \in \mathbb{R}^{N+1}. \label{eqq_Phi_22_v_div} 
}
\end{lemma}
\begin{proof}
From Eq.~\eqref{eqq_Phi_22} we have
\ali{
\frac{\pa (\bm{\Phi}_{ij})_{22}}{\pa v_n} 
&= \frac{\pa \sum_{n=1}^N v_n^2 f^{\prime\prime}(y_{in}) f^{\prime\prime}(y_{jn})}{\pa v_n} \nn \\
&= 
\begin{cases}
0, & n = 0 \\
2v_n f^{\prime\prime}(y_{in}) f^{\prime\prime}(y_{jn}), & n \in \mathbb{N}_{\le N}
\end{cases}. \nn
}
Then, from this and Eq.~\eqref{eqq_g_phi_ij}, 
\ali{
&\frac{\pa (\bm{\Phi}_{ij})_{22}}{\pa \bm{v}} 
= 2 \bm{V}^{\top} \odot \bm{f}^{\prime\prime}_0(\bm{y}_{i}) \odot \bm{f}^{\prime\prime}_0(\bm{y}_{j}) \nn \\
&= \bm{V}^{\top} \odot \bm{f}^{\prime\prime}_0(\bm{y}_{i}) \odot \bm{f}^{\prime\prime}_0(\bm{y}_{j})
+ \bm{V}^{\top} \odot \bm{f}^{\prime\prime}_0(\bm{y}_{j}) \odot \bm{f}^{\prime\prime}_0(\bm{y}_{i}) \nn \\
&= \bm{\mathfrak{g}}^{\Phi}_{ij} + \bm{\mathfrak{g}}^{\Phi}_{ji} \nn
}
holds.
\end{proof}

\begin{lemma}
\ali{
\Bigg[ \frac{\pa \mathrm{tr}(\bm{H}_L(\bm{\theta})^2)}{\pa \bm{v}} \Bigg]_\Phi
&= \sum_{i=1}^I \sum_{j=1}^I
\Big( 
\bm{\mathcal{C}}_{ij}^\Phi + \bm{\mathcal{D}}_{ij}^\Phi + \bm{\mathcal{C}}_{ji}^\Phi + \bm{\mathcal{D}}_{ji}^\Phi \Big) \label{eqq_v_Phi_fin}
}
\end{lemma}
\begin{proof}
From Eqs.~\eqref{eqq_K_Phi_ij}, 
\eqref{eqq_jac_phi_theta}, 
\eqref{eqq_oi_oj}, 
\eqref{eqq_Phi_11_v_div}, 
\eqref{eqq_Phi_12_v_div}, 
\eqref{eqq_Phi_21_v_div}, and 
\eqref{eqq_Phi_22_v_div}, the first term of Eq.~\eqref{eqq_phi_theta_dif} becomes
\ali{
&\frac{\pa \bm{\phi}_{ij}}{\pa \bm{v}}(\bm{o}_i \otimes \bm{o}_j) \nn \\
&= \mat{
\frac{\pa (\bm{\Phi}_{ij})_{11}}{\pa \bm{v}} &
\frac{\pa (\bm{\Phi}_{ij})_{12}}{\pa \bm{v}} &
\frac{\pa (\bm{\Phi}_{ij})_{21}}{\pa \bm{v}} &
\frac{\pa (\bm{\Phi}_{ij})_{22}}{\pa \bm{v}}}(\bm{o}_i \otimes \bm{o}_j) \nn \\
&= 
s^\prime(z_i) s^\prime(z_j) ( \bm{\mathfrak{e}}^{\Phi}_{ij} + \bm{\mathfrak{e}}^{\Phi}_{ji}) +
s^\prime(z_i) \delta_j \bm{\mathfrak{f}}^{\Phi}_{ij}  \nn \\
&+
\delta_i s^\prime(z_j) \bm{\mathfrak{f}}^{\Phi}_{ji} +
\delta_i \delta_j (\bm{\mathfrak{g}}^{\Phi}_{ij} + \bm{\mathfrak{g}}^{\Phi}_{ji})\nn \\
&= 
\big( s^\prime(z_i) s^\prime(z_j) \bm{\mathfrak{e}}^{\Phi}_{ij} +
s^\prime(z_i) \delta_j \bm{\mathfrak{f}}^{\Phi}_{ij} +
\delta_i \delta_j \bm{\mathfrak{g}}^{\Phi}_{ij} \big)\nn \\
&+ 
\big( s^\prime(z_j) s^\prime(z_i) \bm{\mathfrak{e}}^{\Phi}_{ji} +
s^\prime(z_j) \delta_i \bm{\mathfrak{f}}^{\Phi}_{ji} +
\delta_j \delta_i \bm{\mathfrak{g}}^{\Phi}_{ji} \big)\nn \\
&= \mat{\bm{\mathfrak{e}}^{\Phi}_{ij} & \bm{\mathfrak{f}}^{\Phi}_{ij} & \bm{0}_{N+1} &\bm{\mathfrak{g}}^{\Phi}_{ij}}(\bm{o}_i \otimes \bm{o}_j) \nn \\
&+ \mat{\bm{\mathfrak{e}}^{\Phi}_{ji} & \bm{\mathfrak{f}}^{\Phi}_{ji} & \bm{0}_{N+1} &\bm{\mathfrak{g}}^{\Phi}_{ji}}(\bm{o}_j \otimes \bm{o}_i) \nn \\ 
&= \bm{K}_{ij}^\Phi (\bm{o}_i \otimes \bm{o}_j) + \bm{K}_{ji}^\Phi (\bm{o}_j \otimes \bm{o}_i). \nn
} 
Also, from Eq.~\eqref{eqq_o_v_div}, the second and third terms of Eq.~\eqref{eqq_phi_theta_dif} become
\ali{
&\frac{\pa \bm{o}_i}{\pa \bm{v}} \bm{\Phi}_{ij} \bm{o}_j
+ \frac{\pa \bm{o}_{j} }{\pa \bm{v}} \bm{\Phi}_{ji} \bm{o}_{i}\nn \\
&= \bm{h}(\bm{f}(\bm{y}_i)) \bm{s}^{\prime\prime, \prime}(z_i)^\top \bm{\Phi}_{ij} \bm{o}_j
+ \bm{h}(\bm{f}(\bm{y}_j)) \bm{s}^{\prime\prime, \prime}(z_j)^\top \bm{\Phi}_{ji} \bm{o}_i. \nn
}
Finally, from Eqs.~\eqref{eqq_Cij_Phi}, \eqref{eqq_Dij_Phi}, \eqref{eqq_phi_tr2_div}, and \eqref{eqq_phi_theta_dif}, 
\ali{
\Bigg[ \frac{\pa \mathrm{tr}(\bm{H}_L(\bm{\theta})^2)}{\pa \bm{v}} \Bigg]_\Phi
&= \sum_{i=1}^I \sum_{j=1}^I 
(1+ \bm{x}_i^\top\bm{x}_j)^2 \frac{\pa \phi_{ij}}{\pa \bm{v}} \nn \\
&= \sum_{i=1}^I \sum_{j=1}^I \Big( 
\bm{\mathcal{C}}_{ij}^\Phi + \bm{\mathcal{D}}_{ij}^\Phi + 
\bm{\mathcal{C}}_{ji}^\Phi + \bm{\mathcal{D}}_{ji}^\Phi \Big)\nn
}
is obtained.
\end{proof}

\subsubsection{Gradient of the Psi Term}
\begin{lemma}
\ali{
\frac{\pa (\bm{\Psi}_{ij})_{11} }{\pa \bm{v}} 
&= \bm{\mathfrak{e}}^{\Psi}_{ij} + \bm{\mathfrak{e}}^{\Psi}_{ji} \in \mathbb{R}^{N+1}.
\label{eqq_Psi_11_v_div}
}
\end{lemma}
\begin{proof}
Substituting Eqs.~\eqref{eqq_ypv_v} and \eqref{eqq_pfy_v} into Eq.~\eqref{eq_psi_11_theta} yields
\ali{
\frac{\pa (\bm{\Psi}_{ij})_{11} }{\pa \bm{v}} 
&=(1+\bm{f}(\bm{y}_i)^\top \bm{f}(\bm{y}_j)) \nn \\
&\times \big( 
\bm{F}^\prime_0(\bm{y}_j) \bm{F}^\prime(\bm{y}_i)  + 
\bm{F}^\prime_0(\bm{y}_i) \bm{F}^\prime(\bm{y}_j)  
\big) \widetilde{\bm{V}}^\top \nn \\
&=(1+\bm{f}(\bm{y}_i)^\top \bm{f}(\bm{y}_j)) 
\bm{F}^\prime_0(\bm{y}_i) \bm{F}^\prime(\bm{y}_j) 
\widetilde{\bm{V}}^\top  \nn \\
&+ (1+\bm{f}(\bm{y}_j)^\top \bm{f}(\bm{y}_i)) 
\bm{F}^\prime_0(\bm{y}_j) \bm{F}^\prime(\bm{y}_i) 
\widetilde{\bm{V}}^\top \nn\\
&= \bm{\mathfrak{e}}^{\Psi}_{ij} + \bm{\mathfrak{e}}^{\Psi}_{ji}, \nn
}
where Eqs.~\eqref{eqq_e_psi_ij} are utilized.
\end{proof}

\begin{lemma}
\ali{
\frac{\pa (\bm{\Psi}_{ij})_{12}}{\pa \bm{v}} 
&= \bm{\mathfrak{f}}^{\Psi}_{ij}   \in \mathbb{R}^{N+1}, \label{eqq_Psi_12_v_div} \\
\frac{\pa (\bm{\Psi}_{ij})_{21}}{\pa \bm{v}} 
&= \bm{\mathfrak{f}}^{\Psi}_{ji}  \in \mathbb{R}^{N+1}. \label{eqq_Psi_21_v_div}
}
\end{lemma}
\begin{proof}
Since $\bm{F}^\prime(\bm{y}_j) \bm{f}(\bm{y}_i)$ does not depend on $\bm{v}$, its Jacobian becomes
\ali{
\frac{\pa \bm{F}^\prime(\bm{y}_j) \bm{f}(\bm{y}_i)}{\pa \bm{v}} = \bm{0}_{(N+1) \times N}. \nn
}
From Eqs.~\eqref{eqq_f_psi_ij}, \eqref{eqq_ypv_v}, and \eqref{eqq_Psi_12_theta_div}, we have
\ali{
\frac{\pa (\bm{\Psi}_{ij})_{12}}{\pa \bm{v}} 
& 
=\bm{0}_{(N+1) \times N}\bm{F}^\prime(\bm{y}_i) \widetilde{\bm{V}}^\top 
+\bm{F}^\prime_0(\bm{y}_i) \bm{F}^\prime(\bm{y}_j) \bm{f}(\bm{y}_i)\nn \\
&=\bm{F}^\prime_0(\bm{y}_i) \bm{F}^\prime(\bm{y}_j) \bm{f}(\bm{y}_i)
= \bm{\mathfrak{f}}^{\Psi}_{ij}. \nn
}
Using Eq.~\eqref{eqq_Psi_12_21_change}, 
\ali{
\frac{\pa (\bm{\Psi}_{ij})_{21}}{\pa \bm{v}} 
= \frac{\pa (\bm{\Psi}_{ji})_{12}}{\pa \bm{v}} 
= \bm{\mathfrak{f}}^{\Psi}_{ji}\nn
}
is obtained.
\end{proof}

\begin{lemma}
\ali{
\frac{\pa (\bm{\Psi}_{ij})_{22} }{\pa \bm{v}} &= \bm{0}_{N+1}. \label{eqq_Psi_22_v_div}
}
\end{lemma}
\begin{proof}
As it can be seen from Eq.~\eqref{eqq_Psi_22}, $(\bm{\Psi}_{ij})_{22}$ does not depend on $\bm{v}$.
\end{proof}

\begin{lemma}
\ali{
\Bigg[ \frac{\pa \mathrm{tr}(\bm{H}_L(\bm{\theta})^2)}{\pa \bm{v}} \Bigg]_\Psi
&= \sum_{i=1}^I \sum_{j=1}^I
\Big( 
\bm{\mathcal{C}}_{ij}^\Psi + \bm{\mathcal{D}}_{ij}^\Psi + \bm{\mathcal{C}}_{ji}^\Psi + \bm{\mathcal{D}}_{ji}^\Psi \Big). \label{eqq_v_Psi_fin}
}
\end{lemma}
\begin{proof}
From Eqs.~\eqref{eqq_K_Psi_ij}, \eqref{eqq_jac_psi_theta}, \eqref{eqq_oi_oj}, \eqref{eqq_Psi_11_v_div}, \eqref{eqq_Psi_12_v_div}, \eqref{eqq_Psi_21_v_div}, and \eqref{eqq_Psi_22_v_div}, the first term of Eq.~\eqref{eqq_psi_theta_dif} becomes
\ali{
&\frac{\pa \bm{\psi}_{ij}}{\pa \bm{v}}(\bm{o}_i \otimes \bm{o}_j) \nn \\
&= \mat{
\frac{\pa (\bm{\Psi}_{ij})_{11}}{\pa \bm{v}} &
\frac{\pa (\bm{\Psi}_{ij})_{12}}{\pa \bm{v}} &
\frac{\pa (\bm{\Psi}_{ij})_{21}}{\pa \bm{v}} &
\frac{\pa (\bm{\Psi}_{ij})_{22}}{\pa \bm{v}}}(\bm{o}_i \otimes \bm{o}_j) \nn \\
&=
s^\prime(z_i) s^\prime(z_j) ( \bm{\mathfrak{e}}^{\Psi}_{ij} + \bm{\mathfrak{e}}^{\Psi}_{ji}) +
s^\prime(z_i) \delta_j \bm{\mathfrak{f}}^{\Psi}_{ij} +
\delta_i s^\prime(z_j) \bm{\mathfrak{f}}^{\Psi}_{ji} )\nn \\
&= 
\big( s^\prime(z_i) s^\prime(z_j) \bm{\mathfrak{e}}^{\Psi}_{ij} +
s^\prime(z_i) \delta_j \bm{\mathfrak{f}}^{\Psi}_{ij}\big)\nn \\
&+ 
\big( s^\prime(z_j) s^\prime(z_i) \bm{\mathfrak{e}}^{\Psi}_{ji} +
s^\prime(z_j) \delta_i \bm{\mathfrak{f}}^{\Psi}_{ji} \big) \nn\\
&= \mat{\bm{\mathfrak{e}}^{\Psi}_{ij} & \bm{\mathfrak{f}}^{\Psi}_{ij} & \bm{0}_{N+1} &\bm{0}_{N+1}}(\bm{o}_i \otimes \bm{o}_j)  \nn\\
&+ \mat{\bm{\mathfrak{e}}^{\Psi}_{ji} & \bm{\mathfrak{f}}^{\Psi}_{ji} & \bm{0}_{N+1} &\bm{0}_{N+1}}(\bm{o}_j \otimes \bm{o}_i)\nn \\
&= \bm{K}_{ij}^\Psi (\bm{o}_i \otimes \bm{o}_j) + \bm{K}_{ji}^\Psi (\bm{o}_j \otimes \bm{o}_i). \nn
}
Also, from Eq.~\eqref{eqq_o_v_div}, the second and third terms of Eq.~\eqref{eqq_psi_theta_dif} become
\ali{
&\frac{\pa \bm{o}_i}{\pa \bm{v}} \bm{\Psi}_{ij} \bm{o}_j
+ \frac{\pa \bm{o}_{j} }{\pa \bm{v}} \bm{\Psi}_{ji} \bm{o}_{i}\nn \\
&= \bm{h}(\bm{f}(\bm{y}_i)) \bm{s}^{\prime\prime, \prime}(z_i)^\top \bm{\Psi}_{ij} \bm{o}_j
+ \bm{h}(\bm{f}(\bm{y}_j)) \bm{s}^{\prime\prime, \prime}(z_j)^\top \bm{\Psi}_{ji} \bm{o}_i. \nn
}
From Eqs.~\eqref{eqq_Cij_Psi}, \eqref{eqq_Dij_Psi}, \eqref{eqq_psi_tr2_div}, and \eqref{eqq_psi_theta_dif}, 
\ali{
\Bigg[ \frac{\pa \mathrm{tr}(\bm{H}_L(\bm{\theta})^2)}{\pa \bm{v}} \Bigg]_\Psi
&= 2\sum_{i=1}^I \sum_{j=1}^I 
(1+ \bm{x}_i^\top\bm{x}_j) \frac{\pa \psi_{ij}}{\pa \bm{v}} \nn \\
&= \sum_{i=1}^I \sum_{j=1}^I \Big( 
\bm{\mathcal{C}}_{ij}^\Psi + \bm{\mathcal{D}}_{ij}^\Psi + 
\bm{\mathcal{C}}_{ji}^\Psi + \bm{\mathcal{D}}_{ji}^\Psi \Big) \nn
}
is obtained.
\end{proof}

\subsubsection{Gradient of the Omega Term}
\begin{lemma}
\ali{
\Bigg[ \frac{\pa \mathrm{tr}(\bm{H}_L(\bm{\theta})^2)}{\pa \bm{v}} \Bigg]_\Omega
&= \sum_{i=1}^I \sum_{j=1}^I
\Big( \bm{\mathcal{C}}_{ij}^\Omega +\bm{\mathcal{D}}_{ij}^\Omega
+ \bm{\mathcal{C}}_{ji}^\Omega + \bm{\mathcal{D}}_{ji}^\Omega \Big). \label{eqq_v_Omega_fin}
}
\end{lemma}
\begin{proof}
From Eqs.~\eqref{eq_ome_ij} and \eqref{eqq_szp_v}, we have
\ali{
\frac{\pa \omega_{ij} }{\pa \bm{v}} 
&= \frac{\pa \bm{o}_{i}^\top \bm{\Omega}_{ij} \bm{o}_{j} }{\pa \bm{v}}
= \frac{\pa s^\prime(z_i) s^\prime(z_j) }{\pa \bm{v}}\nn \\
&= \frac{\pa s^\prime(z_i) }{\pa \bm{v}} s^\prime(z_j)
+\frac{\pa s^\prime(z_j) }{\pa \bm{v}} s^\prime(z_i)\nn \\
&=s^{\prime\prime}(z_i) s^\prime(z_j) \bm{h}(\bm{f}(\bm{y}_i))
+ s^{\prime\prime}(z_j) s^\prime(z_i) \bm{h}(\bm{f}(\bm{y}_j)). \nn
}
From this and Eq.~\eqref{eqq_Cij_Omega}, we obtain
\ali{
&(1+ \bm{f}(\bm{y}_i)^\top\bm{f}(\bm{y}_j))^2 \frac{\pa \omega_{ij} }{\pa \bm{v}}\nn \\
&= s^{\prime\prime}(z_i)s^\prime(z_j) (1+ \bm{f}(\bm{y}_i)^\top\bm{f}(\bm{y}_j))^2 \bm{h}(\bm{f}(\bm{y}_i)) \nn \\
&+ s^{\prime\prime}(z_j)s^\prime(z_i) (1+ \bm{f}(\bm{y}_j)^\top\bm{f}(\bm{y}_i))^2 \bm{h}(\bm{f}(\bm{y}_j))\nn \\
&=\bm{\mathcal{C}}_{ij}^\Omega + \bm{\mathcal{C}}_{ji}^\Omega. \nn
}
Using the fact that $\bm{f}(\bm{y})$ does not depend on $\bm{v}$ and Eq.~\eqref{eqq_Dij_Omega},
\ali{
\omega_{ij} \frac{\pa (1+ \bm{f}(\bm{y}_i)^\top\bm{f}(\bm{y}_j))^2 }{\pa \bm{v}} 
&= \omega_{ij} \bm{0}_{N+1} + \omega_{ji} \bm{0}_{N+1} \nn \\
&= \bm{\mathcal{D}}_{ij}^\Omega + \bm{\mathcal{D}}_{ji}^\Omega\nn
}
holds, where Eq.~\eqref{eqq_omega_change_ij} is utilized.
From these and Eq.~\eqref{eqq_omega_tr2_div}, 
\ali{
&\Bigg[ \frac{\pa \mathrm{tr}(\bm{H}_L(\bm{\theta})^2)}{\pa \bm{v}} \Bigg]_\Omega
= \sum_{i=1}^I \sum_{j=1}^I \nn \\
&\bigg(
(1+ \bm{f}(\bm{y}_i)^\top\bm{f}(\bm{y}_j))^2 \frac{\pa \omega_{ij}}{\pa \bm{v}} + 
\omega_{ij} \frac{\pa (1+ \bm{f}(\bm{y}_i)^\top\bm{f}(\bm{y}_j))^2 }{\pa \bm{v}} 
\bigg) \nn \\
&= 
\sum_{i=1}^I \sum_{j=1}^I 
\Big( \bm{\mathcal{C}}_{ij}^\Omega +\bm{\mathcal{D}}_{ij}^\Omega
+ \bm{\mathcal{C}}_{ji}^\Omega + \bm{\mathcal{D}}_{ji}^\Omega \Big) \nn
}
is obtained.
\end{proof}

\subsubsection{Completion of the Proof}
From Eqs.~\eqref{eqq_H2_theta_div}, 
\eqref{eqq_phi_tr2_div}, 
\eqref{eqq_psi_tr2_div}, 
\eqref{eqq_omega_tr2_div}, 
\eqref{eqq_v_Phi_fin}, 
\eqref{eqq_v_Psi_fin}, 
\eqref{eqq_v_Omega_fin}, we have
\ali{
&\frac{\pa \mathrm{tr}(\bm{H}_L(\bm{\theta})^2)}{\pa \bm{v}}
= \sum_{a \in \mathfrak{B}} \Bigg[ \frac{\pa \mathrm{tr}(\bm{H}_L(\bm{\theta})^2)}{\pa \bm{v}} \Bigg]_a \nn \\
&=  \sum_{i=1}^I \sum_{j=1}^I (\bm{\mathcal{C}}_{ij}^\Phi + \bm{\mathcal{C}}_{ij}^\Psi
+ \bm{\mathcal{C}}_{ij}^\Omega + \bm{\mathcal{D}}_{ij}^\Phi + \bm{\mathcal{D}}_{ij}^\Psi + \bm{\mathcal{D}}_{ij}^\Omega) \nn \\
&+  \sum_{i=1}^I \sum_{j=1}^I (\bm{\mathcal{C}}_{ji}^\Phi + \bm{\mathcal{C}}_{ji}^\Psi
+ \bm{\mathcal{C}}_{ji}^\Omega + \bm{\mathcal{D}}_{ji}^\Phi + \bm{\mathcal{D}}_{ji}^\Psi + \bm{\mathcal{D}}_{ji}^\Omega) \nn \\
&= 2 \sum_{i=1}^I \sum_{j=1}^I (\bm{\mathcal{V}}_{ij}^\Phi + \bm{\mathcal{V}}_{ij}^\Psi
+ \bm{\mathcal{V}}_{ij}^\Omega). \nn
}
Therefore, Eq.~\eqref{eqq_H2_v_div_fin} holds.

\subsection{Proof for Eq. \eqref{eqq_data_i_conv} }\label{secc_data_i_conv}
Each term in the gradient of $\mathrm{tr}(\bm{H}_L(\bm{\theta}))$ becomes
\ali{
&p_i \rightarrow q_i  \nn \\
&\Rightarrow (s(z_i) \rightarrow 0 \lor 1) \land (\delta_i \rightarrow 0) \Rightarrow s^{\prime}(z_i) \rightarrow 0 \land s^{\prime\prime}(z_i) \rightarrow 0 \nn \\
&\Rightarrow 
\bigg( \bm{h}(\bm{x}_i) \otimes \sum_{a \in \mathfrak{A}} \bm{\mathcal{W}}_{i}^a  \rightarrow \bm{0}_{(M+1)N} \bigg)
\land
\bigg( \sum_{a \in \mathfrak{A}} \bm{\mathcal{V}}_{i}^a \rightarrow \bm{0}_{N+1} \bigg). \nn
}
Similarly, each term in the gradient of $\mathrm{tr}(\bm{H}_L(\bm{\theta})^2)$ becomes
\ali{
p_i \rightarrow q_i 
&\Rightarrow (s(z_i) \rightarrow 0 \lor 1) \land (\delta_i \rightarrow 0) \nn \\
&\Rightarrow (s^{\prime}(z_i) \rightarrow 0) \land (s^{\prime\prime}(z_i) \rightarrow 0 ) \nn \\
&\Rightarrow (\bm{o}_i \otimes \bm{o}_j \rightarrow \bm{0}_4) \land (\bm{s}^{\prime\prime/\prime}(z) \rightarrow \bm{0}_2) \nn \\
&\Rightarrow (\bm{\mathcal{A}}_{ij}^a \rightarrow \bm{0}_N) 
\land (\bm{\mathcal{B}}_{ij}^a \rightarrow \bm{0}_N)
\land (\bm{\mathcal{C}}_{ij}^a \rightarrow \bm{0}_{N+1}) \nn \\ 
&\land (\bm{\mathcal{D}}_{ij}^b \rightarrow \bm{0}_{N+1}), \forall a \in \mathfrak{B}, \forall b \in \mathfrak{B}\setminus\{\Omega\} \nn \\
&\Rightarrow \bigg( \bm{h}(\bm{x}_i) \otimes  \sum_{a \in \mathfrak{B}}
\bm{\mathcal{W}}_{ij}^a
 \rightarrow \bm{0}_{(M+1)N} \bigg) \nn \\ 
& \land
\bigg( \sum_{a \in \mathfrak{B}}
\bm{\mathcal{V}}_{ij}^a
 \rightarrow \bm{0}_{N+1} \bigg). \nn
}
Therefore, Eq.~\eqref{eqq_data_i_conv} holds.

\subsection{Proof for Eq.~\eqref{eqq_inf_sup_rel}}\label{secc_pr_sup_inf}
First, we focus on the upper bound. 
From Eq.~\eqref{eq_main_theorem}, since
\ali{
&\Big( \sigma(\bm{\theta}^{(t+1)}) < \sigma(\bm{\theta}^{(t)}) \Big) \land 
\Big( \mu(\bm{\theta}^{(t+1)}) = \mu(\bm{\theta}^{(t)}) \Big)
\Rightarrow \nn \\
&\lambda_{\mathrm{sup}}(\bm{\theta}^{(t+1)}) 
= \mu(\bm{\theta}^{(t+1)}) + \sqrt{D-1} \sigma(\bm{\theta}^{(t+1)}) \nn \\
&< \mu(\bm{\theta}^{(t+1)}) + \sqrt{D-1} \sigma(\bm{\theta}^{(t)}) \nn \\
&= \mu(\bm{\theta}^{(t)}) + \sqrt{D-1} \sigma(\bm{\theta}^{(t)})
= \lambda_{\mathrm{sup}}(\bm{\theta}^{(t)}) \nn
}
holds, we obtain $\lambda_{\mathrm{sup}}(\bm{\theta}^{(t+1)}) < \lambda_{\mathrm{sup}}(\bm{\theta}^{(t)})$.
Next, we focus on the lower bound. 
From Eq.~\eqref{eqq_lambda_inf}, since
\ali{
&\Big( \sigma(\bm{\theta}^{(t+1)}) < \sigma(\bm{\theta}^{(t)}) \Big) \land 
\Big( \mu(\bm{\theta}^{(t+1)}) = \mu(\bm{\theta}^{(t)}) \Big)
\Rightarrow \nn \\
&\lambda_{\mathrm{inf}}(\bm{\theta}^{(t+1)}) = \mu(\bm{\theta}^{(t+1)}) - \sqrt{D-1} \sigma(\bm{\theta}^{(t+1)}) \nn \\
&> \mu(\bm{\theta}^{(t+1)}) - \sqrt{D-1} \sigma(\bm{\theta}^{(t)})\nn \\
&= \mu(\bm{\theta}^{(t)}) - \sqrt{D-1} \sigma(\bm{\theta}^{(t)})
= \lambda_{\mathrm{inf}}(\bm{\theta}^{(t)}) \nn
}
holds, we obtain $\lambda_{\mathrm{inf}}(\bm{\theta}^{(t+1)}) > \lambda_{\mathrm{inf}}(\bm{\theta}^{(t)})$.


\bibliographystyle{elsarticle-num}
\bibliography{refs}


\end{document}